\documentclass[letterpaper]{article} % DO NOT CHANGE THIS
\usepackage{aaai24}  % DO NOT CHANGE THIS
\usepackage{times}  % DO NOT CHANGE THIS
\usepackage{helvet}  % DO NOT CHANGE THIS
\usepackage{courier}  % DO NOT CHANGE THIS
\usepackage[hyphens]{url}  % DO NOT CHANGE THIS
\usepackage{graphicx} % DO NOT CHANGE THIS
\usepackage{natbib}  % DO NOT CHANGE THIS AND DO NOT ADD ANY OPTIONS TO IT
\usepackage{caption} % DO NOT CHANGE THIS AND DO NOT ADD ANY OPTIONS TO IT
\newtheorem{thm}{Theorem}

\newtheorem{lma}[thm]{Lemma}

\newtheorem{assump}[thm]{Assumption}

\newtheorem{proof}{Proof}

\usepackage{amsmath,amssymb}
\usepackage{subfig}
\usepackage{booktabs}
\usepackage{bibentry}
\usepackage{multirow}
\usepackage{xcolor}
\usepackage[ruled]{algorithm2e}
\usepackage{newfloat}
\usepackage{listings}
\DeclareCaptionStyle{ruled}{labelfont=normalfont,labelsep=colon,strut=off} % DO NOT CHANGE THIS
\title{Multi-level Cross-modal Alignment for Image Clustering}
\author {
    Liping Qiu\equalcontrib,
    Qin Zhang\equalcontrib,
    Xiaojun Chen\footnote{Corresponding author.},
    Shaotian Cai
}
\affiliations {
    Shenzhen University, Shenzhen, China\\
    qiuliping2021@email.szu.edu.cn, \{qinzhang, xjchen\}@szu.edu.cn, cai.st@foxmail.com
}

\usepackage{bibentry}
\begin{document}
\maketitle
\begin{abstract}

Recently, the cross-modal pretraining model has been employed to produce meaningful pseudo-labels to supervise the training of an image clustering model. However, numerous erroneous alignments in a cross-modal pretraining model could produce poor-quality pseudo labels and degrade clustering performance. To solve the aforementioned issue, we propose a novel \textbf{Multi-level Cross-modal Alignment} method to improve the alignments in a cross-modal pretraining model for downstream tasks, by building a smaller but better semantic space and aligning the images and texts in three levels, i.e., instance-level, prototype-level, and semantic-level. Theoretical results show that our proposed method converges, and suggests effective means to reduce the expected clustering risk of our method. Experimental results on five benchmark datasets clearly show the superiority of our new method.

\end{abstract}
\section{Introduction}
\label{sec:intro}

Image clustering which groups images into different clusters without labels is an essential task in unsupervised learning. Many methods are proposed to utilize the large-scale pre-training models such as Resnet~\cite{he2016deep} or ViT~\cite{dosovitskiy2020image} to extract high-quality representations for image clustering~\cite{ji2019invariant,li2021contrastive,zhong2021graph,wu2019deep,gansbeke2020scan,dang2021nearest}. Then, to unsupervised classification models, multiple indirect loss functions are used  (\textit{e.g.} sample relations~\cite{chang2017deep}, invariant information ~\cite{ji2019invariant,li2021contrastive}, mutual information~\cite{wu2019deep} and entropy~\cite{huang2020deep,gansbeke2020scan,li2021contrastive}. However, as pointed out in~\cite{cai2022semantic}, the aforementioned techniques have difficulties in handling examples that are semantically different but visually comparable by focusing only on images.

 \begin{figure}[!htb]
     \centering
     \includegraphics[width=0.9\columnwidth]{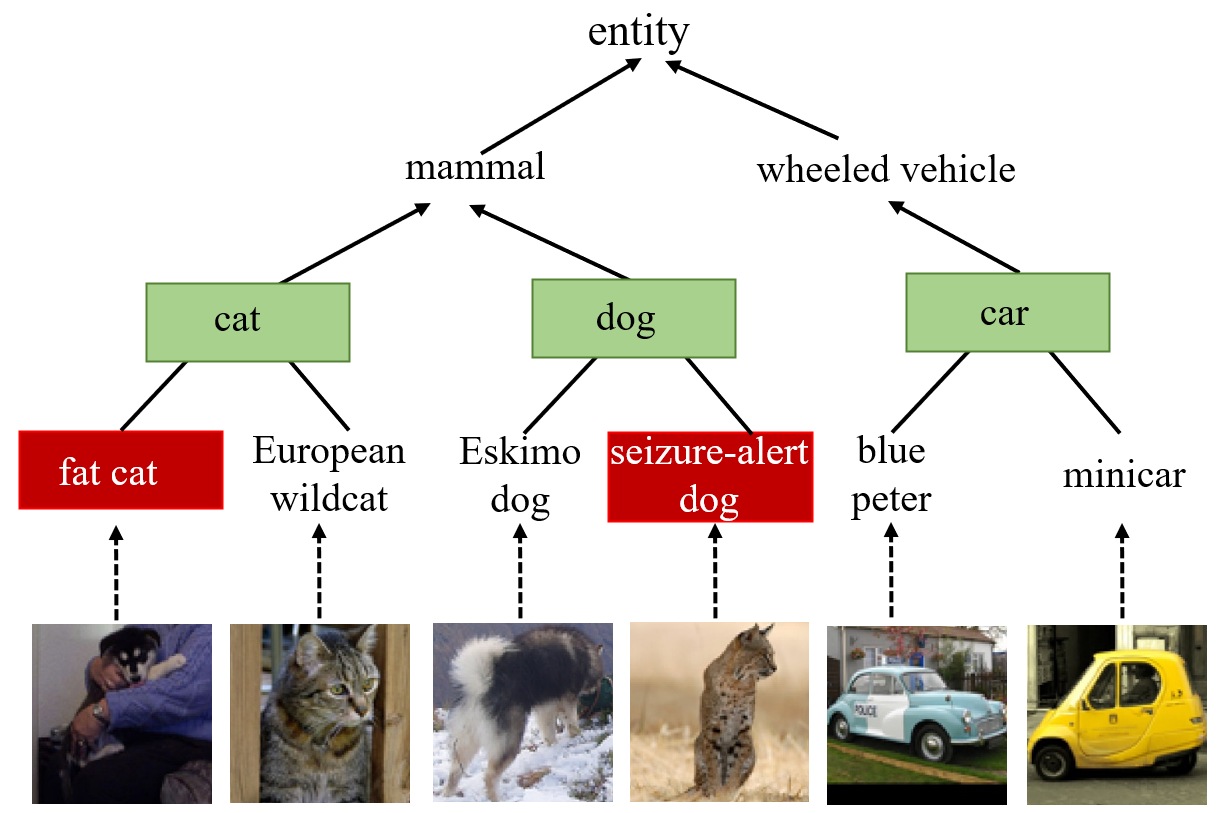}
     \caption{The nearest noun (selected from WorldNet~\cite{miller1995wordnet}) for the images in STL10, where the image and text embeddings are obtained via CLIP~\cite{zhou2021learning}. The green words correspond to the correct alignments, while the red words indicate incorrect alignments. }    
     \label{fig:hierarchical_structure}
 \end{figure}

Recently, many vision-language pre-training (VLP) models have been developed to align images and texts into a unified semantic space~\cite{li2019visualbert,chen2020uniter,ramesh2021zero,li2020unimo,radford2021learning,jia2021scaling}. To utilize the VLP models for image clustering, Cai et al.~\cite{cai2022semantic} proposed to use CLIP~\cite{radford2021learning} to produce meaningful pseudo-labels and achieved significant improvements on a wide range of datasets in comparison with conventional image clustering methods. Li et al.~\cite{li2022masked} also used CLIP for a zero-shot image classification task. The success of these methods suggests a promising direction for image clustering. However, as shown in Figure~\ref{fig:hierarchical_structure}, there are instances where the alignments between images and texts in CLIP may be incorrect for downstream tasks, resulting in substandard pseudo-labels and poor clustering performance. SIC~\cite{cai2022semantic} simply uses CLIP to obtain the embeddings of images and texts and cannot deal with incorrect alignments. Although MUST~\cite{li2022masked} strives to optimize the image encoder in CLIP, its efficiency is hampered by using the pretraining task to update the image encoder, resulting in a slow process.

To address the above problem, we propose a novel method, namely \textbf{Multi-level Cross-modal Alignment (MCA)}, an efficient way to improve the alignments between images and texts in CLIP for clustering tasks. In general, our main contributions are as follows:
\begin{itemize}
    \item We propose to use the hierarchical structure in WordNet (see Figure~\ref{fig:hierarchical_structure}) to filter irrelevant words and construct a smaller but better semantic space, thus reducing the affection of unrelated nouns for clustering. Our experimental findings demonstrate that it can reduce the number of words by up to 60\% and significantly enhance clustering performance compared to SIC.
    \item  We propose to optimize both image and text embeddings for downstream tasks, by aligning the images and texts at three levels, i.e., instance-level, prototype-level, and semantic-level. Our proposed method can better fix the incorrect alignments in CLIP for downstream tasks compared to SIC and MUST.
    \item Theoretical findings demonstrate that our proposed method converges at a sublinear rate and offers effective strategies for lowering the expected clustering risk of our method. These findings will provide valuable guidance for the design of new image clustering methods.
    \item  Experimental results on five benchmark datasets clearly show the superiority of our new method, especially when dealing with complex clusters. 
\end{itemize}

\section{Related Work}
\label{sec:related}

Early deep clustering methods simply combine representation learning and shallow clustering~\cite{xie2016unsupervised,Yang2017Towards,Tian2017Deep,shaham2018spectral}. With the rapid development of the pre-training paradigm, many methods employ large-scale pre-training models such as Resnet~\cite{he2016deep} or ViT~\cite{dosovitskiy2020image} to extract high-quality representations and train a classification model, by maximizing the consistency between each image and its augmentations/neighbors~\cite{ji2019invariant,li2021contrastive,zhong2021graph,wu2019deep,gansbeke2020scan,zhong2021graph,dang2021nearest}, or generating pseudo-labels~\cite {wu2019deep,gansbeke2020scan}. However, as pointed out in~\cite{cai2022semantic}, it is challenging for the aforementioned techniques to handle examples that are semantically different but visually comparable by only accessing visual information in images.

Cross-modal clustering has made significant progress in recent years, which usually learns a shared subspace such that the mutual agreement between multiple modalities is maximized, by Canonical Correlation
Analysis (CCA)~\cite{gao2020cross} or mutual information optimization~\cite{mao2021deep}. However, these methods require image-text pairs as input, which may be cost-intensive to collect in real applications. 

Recently, vision-language pre-training (VLP) models that align multi-modal data in common feature space by different pre-training tasks have been proposed. For example, VisualBert~\cite{li2019visualbert}, UNITER~\cite{chen2020uniter} and DALL-E~\cite{ramesh2021zero} use language-based training strategies, including mask LM (Language Modeling) such as Masked Language/Region Modeling, or autoregressive LM such as image caption and text-grounded image generation. CLIP~\cite{radford2021learning} and ALIGN~\cite{jia2021scaling} utilize cross-modal contrastive learning to align the visual and textual information into a unified semantic space.

Since VLP captures the relationships among images and texts (low-level semantics), it is natural to utilize VLP models to compensate for the semantic information for better image clustering. Cai et al.~\cite{cai2022semantic} proposed to use CLIP~\cite{radford2021learning} to generate meaningful pseudo-labels for image clustering. Li et al.~\cite{li2022masked} also proposed to use CLIP for zero-shot image classification tasks.

\section{Notation and Problem Definition}
\label{sec:notation}

Suppose we have an image dataset $\mathcal{X}= \{x_1, x_2, \dots, x_n\}$ with $n$ instances sampled i.i.d. from input space $\mathcal{D} $, we can obtain the embeddings of these images as $\mathcal{U} = \{\mathbf{u}_1, \mathbf{u}_2, \dots, \mathbf{u}_n\}$ where $\mathbf{u}_i=e_I(x_i)\in\mathbb{R}^{d\times 1}$ is obtained via the image encoder $e_I(.)$ of CLIP, where $d$ is the embedding dimension. To capture the semantic meaning of these images, we first introduce a noun vocabulary $\mathcal{T} = \{t_1, t_2, \dots, t_{m}\}$ that includes $m$ noun phrases sampled from WordNet~\cite{miller1995wordnet}. Then we can obtain the embeddings of these $m$ words as $\mathcal{V} = \{\mathbf{v}_1, \mathbf{v}_2, \dots, \mathbf{v}_{m}\}$ where $\mathbf{v}_i=e_{T}(s_i)\in\mathbb{R}^{d\times 1}$, 
$s_i$ is a sentence like ``$\texttt{A photo of a } \{t_i\}$'' and $e_{T}(.)$ is the text encoder of CLIP. Let $c$ be the number of categories, our goal is to group the images in $\mathcal{X}$ into $c$ clusters with the help of CLIP. Let $f_{I}(e_I(\mathcal{X});\phi)$ denotes the image classification network with parameters $\phi$ that maps an image $x_i$ into a soft cluster assignment probability vector $\mathbf{q}_{i}\in\mathbb{R}^{c\times 1}$, and $f_{S}(e_{T}(\mathcal{T});\theta)$ denotes the text classification network with parameters $\theta$ that maps a word $t_i$ into a soft cluster assignment probability vector $\mathbf{p}_{i}\in\mathbb{R}^{c\times 1}$. Notably, $e_{I}$ and $e_{T}$ in CLIP are kept frozen during the training process. 

\section{Method}

\begin{figure*}[t]
        \centering
        \includegraphics[width=0.8\linewidth]{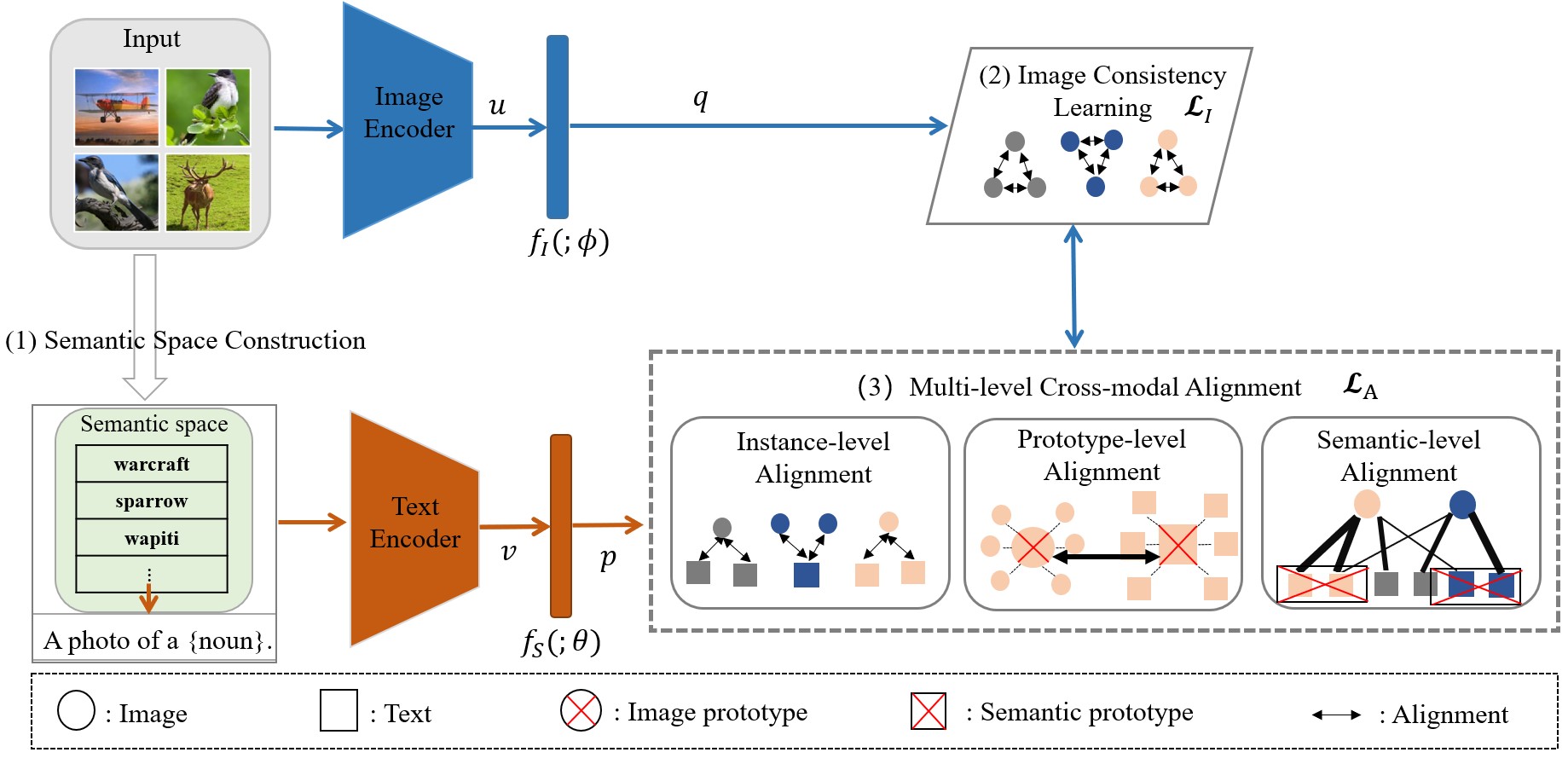}
        \caption{The framework of MCA consists of three parts: (1) Semantic space construction. (2) Image consistency learning (3) Multi-level cross-modal alignment. The thickness of lines in adaptive instance-level alignment reflects the magnitude of attention scores.}
        \label{fig:framework}
\end{figure*}

In this paper, we propose our new method which is shown in Figure~\ref{fig:framework}. This new method mainly consists of three components: 1) \textbf{Semantic space construction} builds a proper semantic space $\mathcal{T}$, 2) \textbf{Image consistency learning} performs the consistency learning in image space, and 3) \textbf{Multi-level cross-modal alignment} aligns images and texts at three distinct levels: instance-level, prototype-level, and semantic-level. This comprehensive alignment approach significantly improves the clustering performance on downstream tasks.

\subsection{Semantic Space Construction}

Constructing a proper semantic space $\mathcal{T}$ from WordNet~\cite{miller1995wordnet} such that the images can be well represented by the words in $\mathcal{T}$ is very important for image clustering, because too small $\mathcal{T}$ may lose important relevant words but too large $\mathcal{T}$ may contain too many irrelevant noisy words. In this step, we first build a candidate semantic space $\mathcal{W}$ as 82,000 nouns in the WordNet dataset~\cite{miller1995wordnet}. Considering an image dataset usually makes up a small part of semantics, we propose a two-step filtering strategy to construct a proper semantic space for an image dataset: 1) \textbf{Uniqueness-based filtering} selects $\gamma_r$ nearest words for each of $c$ image cluster centers obtained by $k$-means of the most unique nouns 
whose uniqueness scores~\cite{cai2022semantic} are greater than a given hyperparameter $\rho_u$. 2) \textbf{Hierarchy-based filtering} employs the hierarchical structure in WordNet to further filter $\mathcal{W}_c$ to form the final semantic space $\mathcal{T}$. Let $\mathcal{T}=\emptyset$. Given an image $x_i$, we find its nearest noun $w_i\in\mathcal{W}_c$ and search its hierarchical structure from WordNet~\cite{miller1995wordnet} to form a hierarchical semantic tree. In general, the words in the lower layers provide more fine-grained information for distinguishing the images, while the words in the higher layers may be useless for clustering. Figure~\ref{fig:hierarchical_structure} shows, for example, ``mammal'' is the common parent of ``dog'' and ``cat'' and cannot distinguish the images in ``dog'' or ``cat''. Therefore, we propose a hierarchy-based filtering strategy that filters out the top $\gamma_{h}$ levels (excluding root node) and add each of the remaining words into $\mathcal{T}$ if it is also in $\mathcal{W}_c$.

\subsection{Image Consistency Learning}
Intuitively, an image and its nearest images may have similar soft cluster assignments. Therefore, we propose the following loss function for image consistency learning:

\begin{equation}
\label{loss_ic}
\begin{aligned}
\mathcal{L}_{I}(f_{I}(e_{I}(\mathcal{X});\phi)) =& -\frac{1}{n} \sum_{i=1}^{n} \sum_{j=rn(\mathcal{N}_{k_{I}}^{I}(x_i))} \log \mathbf{q}_{i}^{T} \mathbf{q}_{j}\\
&-\eta\sum_{l=1}^{c}\bar{q}_{l}\log \bar{q}_{l}
\end{aligned}
\end{equation}
where $\mathcal{N}_{k_{I}}^{I}(x_i)$ contains $k_{I}$ nearest images of $x_i$ and $rn(\mathcal{N}_{k_{I}}^{I}(x_i))$ randomly selects a sample from $\mathcal{N}_{k_{I}}^{I}(x_i)$. The second item is the popular negative entropy loss for preventing trivial solutions that most samples belong to a small proportion of clusters, where
$\bar{q}_{l}=\frac{\sum_{i=1}^{n} q_{il}}{n}$ is the average cluster assignment. $\eta$ is a trade-off parameter.

% \subsection{Semantic Pseudo Label Generation}
\subsection{Multi-level Cross-modal Alignment}
\label{sec:align}

When using a cross-modal pretraining model for image clustering, the main challenge is to rectify incorrect alignments between images and words in image data. In this paper, we propose a novel \textbf{Multi-level Cross-modal Alignment} method for this task, which is shown in Figure~\ref{fig:framework}. Specifically, our method employs a three-level alignment approach. Firstly, at the \textbf{Instance-level Alignment}, each image is aligned with its neighboring texts. Secondly, the \textbf{Prototype-level Alignment} aligns each image prototype with its nearest text prototype. Lastly, at the \textbf{Semantic-level Alignment}, each image is aligned with its neighboring texts in the semantic space. The detailed descriptions of these three alignment processes are provided below.

\noindent\textbf{Instance-level alignment}: Given an image $x_i$ and its neighboring texts $\mathcal{N}_{k_{S}}^{S}(x_i)$, we propose the following contrastive loss function to facilitate the alignment:
\begin{equation}
\begin{split}
\footnotesize
\label{loss_align_si}
    \mathcal{L}_{ia}=-\frac{1}{n} \sum_{i=1}^{n} \sum_{j=rn(\mathcal{N}_{k_{S}}^{S}(x_i))}    \log\frac{exp({\mathbf{q}_{i}}^T\mathbf{p}_{j}/\tau_{ia})}{\sum_{l=1,l\neq j}^{m}  exp({\mathbf{q}_{i}}^{T}\mathbf{p}_{l}/\tau_{ia})}
\end{split}
\end{equation}
where $\tau_{ia}$ is a tempreature parameter.

\noindent\textbf{Prototype-level alignment}: Instancel-level alignment may be affected by noisy neighborhood relationships, so we further propose to align images and texts at prototype level which is more robust to noisy texts. We first compute an image prototype set $\mathcal{H}^{I}$, where $\mathbf{h}_{l}^{I}\in \mathcal{H}^{I}$ is computed as $\mathbf{h}_{l}^{I}=\frac{1}{\left\|\mathbf{q}_{l}\right\|_{1}} \sum_{i=1}^{n} q_{il} \mathbf{u}_{i}$. Then, for each image prototype $\mathbf{h}_{l}^{I}\in \mathcal{H}^{I}$, we can identify the word in $\mathcal{T}$ that is closest to $\mathbf{h}_{l}^{I}$ and finally construct a prototype set $\mathcal{H}^{S}$. To further improve the prototypes in $\mathcal{H}^{S}$, we find $k_{p}$ nearest neighborhoods for each $\mathbf{h}_{l}^{S}\in\mathcal{H}^{S}$ to compute the prototype of these neighborhoods and replace $\mathbf{h}^{S}$ to update $\mathcal{H}^{S}$.

Finally, to align images and texts at the prototype level, we propose the following loss function:
\begin{equation}
\begin{split}
\footnotesize
\label{loss_align2}
    \mathcal{L}_{pa}=-\frac{1}{c}\sum_{j=1}^{c}\log\frac{exp(f_{I}(\mathbf{h}_{j}^{I},\phi)^T f_{S}(\mathbf{h}_{j}^{S},\theta)/\tau_{pa})}{\sum_{l=1,l\neq j}^{c}exp(f_{I}(\mathbf{h}_{j}^{I},\phi)^T f_{S}(\mathbf{h}_{l}^{S},\theta)/\tau_{pa})}
\end{split}
\end{equation}
where $\tau_{pa}$ is a tempreature parameter.

\noindent\textbf{Semantic-level alignment}: Given an image $x_i$ and its neighbouring texts $\mathcal{N}_{k_{S}}^{S}(x_i)$, we let $\mathbf{v}_j=e_{T}(s_j)$ be the the embeddings of $t_j\in\mathcal{N}_{k_{S}}^{S}(x_i)$ and $s_j$ is a sentence like ``$\texttt{A photo of a } \{t_j\}$'', and $\mathbf{p}_j=f_{S}(\mathbf{v}_j,\theta)$. It makes sense that the neighboring texts of an image can help determine the cluster assignment of this image. Note that the alignment relationships between images and texts may vary in different downstream tasks, we propose to use the attention mechanism~\cite{vaswani2017attention} to quantify the correlations between an image and its neighboring texts. Specifically, we compute $\mathbf{p}^\prime_{i}$ as the weighted combination of the neighboring texts' assignments as:

\begin{equation}
\label{def_attention}
\begin{aligned}
\mathbf{p}^\prime_{i}=&f_{A}(\mathbf{u}_i,\mathbf{V}^{S},\mathbf{P}^{S};\mathbf{W}^{I},\mathbf{W}^{S})\\
&=\sum_{j\in \mathcal{N}_{k_{S}}^{S}(x_i)}softmax((\mathbf{W}^{I}\mathbf{u}_i)^{T}\mathbf{W}^{S}\mathbf{v}_{j})\mathbf{p}_{j}
\end{aligned}
\end{equation}
where $f_{A}(\cdots;\mathbf{W}^{I},\mathbf{W}^{S})$ is the attention network to quantify the correlations between an image and its neighboring texts and $\mathbf{W}^{I},\mathbf{W}^{S}\in\mathbb{R}^{d\times d}$ are two parameter matrices. 
Then we use the argmax operation to generate one-hot pseudo-label for $x_i$ as:
\begin{equation}\label{def_pseudo_labels}
    \mathbf{q}^\prime_{i} = \operatorname{one-hot}\left(c,\operatorname{argmax}_{l} p^\prime_{il} \right)
\text{}\end{equation}
where $\operatorname{one-hot}(c,l)$ generates a $c$-bit one-hot vector with the $l$-th element as $1$.

Here, $\mathbf{q}^\prime_{i}$ can be considered as the semantic cluster assignment of $x_i$. Therefore, we perform alignment for each image $x_i$ by aligning the semantic cluster assignment $\mathbf{q}^\prime_{i}$ to the image cluster assignment $\mathbf{q}_i$ with the following loss function:
\begin{equation}
\label{loss_alignment_ai}
\mathcal{L}_{sa} = \frac{1}{n}\sum_{i=1}^{n} CE(\mathbf{q}_i,\mathbf{q}_i^{\prime})    
\end{equation}
where $CE(.)$ is the cross entropy function.

\noindent\textbf{The overall alignment loss function} is: 
\begin{equation}
\label{loss_align}
\begin{aligned}
& \mathcal{L}_{A}(f_{I}(e_I(\mathcal{X});\phi), f_{S}(e_{T}(\mathcal{T});\theta), f_{A}(\cdots;\mathbf{W}^{I},\mathbf{W}^{S})) \\
& =\mathcal{L}_{ia}+\lambda_{pa}\mathcal{L}_{pa}+\lambda_{sa}\mathcal{L}_{sa}
\end{aligned}
\end{equation}
where $\lambda_{pa}$ and $\lambda_{sa}$ are two trade-off parameters.

\subsection{The Overall Objective}

We denote $\mathcal{L}_{A}(f_{I}(e_I(\mathcal{X});\phi), f_{S}(e_{T}(\mathcal{T});\theta), f_{A}(\cdots;\mathbf{W}^{I},\mathbf{W}^{S}))$ as $\mathcal{L}(g(\mathcal{S};\varphi))$ for simplicity, where $\mathcal{S}=(\mathcal{X},\mathcal{T})$, $\varphi=(\phi,\theta,\mathbf{W}^{I},\mathbf{W}^{S})$ and $g$ consists of $f_{I}$, $f_{S}$ and $f_{A}$.
Finally, the overall objective can be formulated as
\begin{equation}
\label{loss_overall}
\small
\begin{split}
    \mathcal{L}(g(\mathcal{S};\varphi))    =& \mathcal{L}_{I}(g(\mathcal{S};\varphi))
    +\lambda_{a}\mathcal{L}_{A}(g(\mathcal{S};\varphi))
\end{split}
\end{equation}
where $\lambda_{a}$ is a trade-off parameter.

\subsection{Theoretical Analysis}
In this part, we first analyze the convergence of our proposed method and then analyze its expected clustering risk. We first introduce the following assumptions:
\begin{assump}\label{as:knn_I}
\textbf{Image Neighborhood Consistency Bound:}
    $\forall x_i \in \mathcal{X}$, $x_j \in \mathcal{N}_{k_{I}}^{I}(x_i)$, $\mathbf{q}_i^{T}\mathbf{q}_j \in[\mu_{I},1].$ 
\end{assump}
\begin{assump}\label{as:cross_knn}
\textbf{Cross-modal Neighborhood Consistency Bound:}
    $\forall x_i \in \mathcal{X}$, $t_j \in \mathcal{N}_{k_S}^{S}(x_i)$, $\mathbf{q}_i^{T}\mathbf{p}_j \in[\mu_{C},1].$ 
\end{assump}
\begin{assump}\label{as:pc}
\textbf{Image Prediction Confidence Bound:}
    $\forall x_i \in \mathcal{X}$, $\|\mathbf{q}_i\|_{\infty} \leq \mu_p.$
\end{assump}
\begin{assump}\label{as:num_neighbor}
\textbf{Image Neighborhood Imbalance Bound:}
    $\forall x_i \in \mathcal{X}$, $x_i$ is in at most $k_{I}'$ samples' (in $\mathcal{X}$) nearest neighborhoods.
\end{assump}

We first give the following theorem demonstrating that the optimization algorithm theoretically converges to the local optima in a sublinear speed.

\begin{thm}
\label{thm:conv}
   Suppose that $g(\mathcal{S};\varphi)$ is twice differential with bounded gradients and Hessians, and $\mathcal{L}(g(\mathcal{S};\varphi))$ has $L$-Lipschitz continuous gradient. Suppose that the learning rate $\eta_{\varphi}$ satisfies $\eta_{\varphi}=\min \{1, \frac{C}{\sqrt{T}}\}$ for some $C>0$, such that $\frac{C}{\sqrt{T}} \geq L$. Then our proposed method can achieve $\min _{0 \leq t \leq T} \mathbb{E}\left[\left\|\nabla \mathcal{L}(g(\mathcal{S});\varphi^{(t)})\right\|_{2}^{2}\right] \leq \epsilon$ in $\mathcal{O}\left(1 / \epsilon^{2}\right)$ steps, where $\epsilon$ is a very small positive real number.
\end{thm}

Next, we analyze the ability of our method to achieve cluster performance on unseen data. Let $\hat{\mathcal{L}}_{n}(g)$ be the empirical clustering risk of MCA and its expectation can be denoted as $\mathcal{L}(g)$. The family of $g$ is defined as $\mathcal{G}$. Then we can obtain the following theorem by analyzing the generalization bound of our proposed method.

\begin{thm}
\label{thm:risk}
Suppose $f_{I}(.;\phi)$ is Lipschitz smooth with constant $L_{I}$ and $\|\mathbf{u}\|_{\infty}\leq M_{u}$. Suppose $\beta(\frac{1}{n}\sum_{i=1}^{n}q_{il}\mathbf{u}_{i})=f_{I}(\mathbf{h}^{I},\phi)^Tf_{S}(\mathbf{h}^{S},\theta)$ is $L_{IS}$-Lipschitz continious, where $\mathbf{h}_{l}^{I}$ and $\mathbf{h}^{S}$ are computed according to the method in Section~\ref{sec:align}. For any $0 < \delta < 1$, we can guarantee that with a probability of at least $1-\delta$ for any $g\in\mathcal{G}$, the following inequality holds.

\begin{equation*}
    \mathcal{L}(g) \leq \widehat{\mathcal{L}}_{n}(g) + \frac{\tilde{c}_1}{\sqrt{n}} + \tilde{c}_2\sqrt{\frac{1}{2n}\log \delta^{-1}}+\frac{2dL_{IS}M_{u}}{n\tau_{pa}}.
\end{equation*}
\textit{where $\tilde{c}_1=2\mu_{I}^{-1}+2\eta C+2\lambda_{a}m/\tau_{ia}+
2\lambda_{a}\lambda_{pa}dL_{IS}M_{u}/\tau_{pa}+2\lambda_{a}\lambda_{sa}c\log\mu_p^{-1}$ and $\tilde{c}_2= (2+2k_{I}')\log\mu_{I}^{-1} + \eta C +\frac{2\lambda_{a}(1-\mu_{C})}{\tau_{ia}}+\lambda_{a}\lambda_{pa}\frac{dcL_{I}M_{u}^2}{\tau_{pa}}+2\lambda_{a}\lambda_{sa}c\log\mu_p^{-1}$ are constants dependent on $\{n, m, \mu_{I}, \mu_{C}, \mu_{p}, k_{I}', c, L_{IS}, L_{I}, M_{u}, d, C\}$. $C$ is a constant.} 
\end{thm}

Theorem~\ref{thm:risk} shows that our proposed method, with high probability $1-\delta$, is with a bounded expected clustering risk on the unseen data. To summarize, the expected clustering risk of MCA is theoretically guaranteed in clustering tasks. Note that the margin $\mathcal{L}(g) - \widehat{\mathcal{L}}_{n}(g)$ is inversely proportional to $\mu_{I}$ and $\mu_{C}$ which reflect the neighborhood consistency in both image domain and cross-domain, and $\mu_{p}$ which reflects the prediction confidence, indicating that improving the neighborhood consistency in both image domain and cross-domain and prediction confidence reduces the expected risk of MCA. Meanwhile, the margin $\mathcal{L}(g) - \widehat{\mathcal{L}}_{n}(g)$ is proportional to $k_{I}'$ which reflects the neighborhood overlapping in the image domain, indicating that reducing the neighborhood imbalance (e.g., by setting a smaller number of neighbors $k_{I}$ or filtering neighborhoods to reduce neighborhood imbalance) also reduces the expected risk of MCA.

\begin{table*}[h]
\centering

\resizebox{\linewidth}{!}{
\begin{tabular}{@{}|l|ccc|ccc|ccc|ccc|ccc|ccc|@{}}
\toprule
Dataset &
  \multicolumn{3}{c|}{\textbf{STL10}} &
  \multicolumn{3}{c|}{\textbf{Cifar10}} & 
  \multicolumn{3}{c|}{\textbf{Cifar100-20}} & 
  \multicolumn{3}{c|}{\textbf{ImageNet-Dogs}} & 
  \multicolumn{3}{c|}{\textbf{Tiny-ImageNet}} \\ 
\midrule
Metrics &
  \multicolumn{1}{c}{ACC} &
  \multicolumn{1}{c}{NMI} &
  \multicolumn{1}{c|}{ARI} &
  \multicolumn{1}{c}{ACC} &
  \multicolumn{1}{c}{NMI} &
  \multicolumn{1}{c|}{ARI} &
  \multicolumn{1}{c}{ACC} &
  \multicolumn{1}{c}{NMI} &
  \multicolumn{1}{c|}{ARI} &
  \multicolumn{1}{c}{ACC} &
  \multicolumn{1}{c}{NMI} &
  \multicolumn{1}{c|}{ARI} &
  \multicolumn{1}{c}{ACC} &
  \multicolumn{1}{c}{NMI} &
  \multicolumn{1}{c|}{ARI} \\ 
\midrule

$k$-means~\cite{MacQueen-1967}      & 19.2 & 12.5 & 6.1  & 22.9 & 8.7   & 4.9  & 13.0  & 8.4  & 2.8  & 10.5 & 5.5 & 2.0 & 2.5 & 6.5 & 0.5\\
SC~\cite{zelnik2005Self}           & 15.9 & 9.8  & 4.8  & 24.7 & 10.3  & 8.5  & 13.6  & 9.0  & 2.2  & 11.1 & 3.8   & 1.3  & 2.2   & 6.3  & 0.4 \\
NMF~\cite{cai2009Locality}          & 18.0 & 9.6  & 4.6  & 19.0 & 8.1   & 3.4  & 11.8  & 7.9  & 2.6  & 11.8 & 4.4   & 1.6  & 2.9   & 7.2  & 0.5\\
               
% AC           & 33.2 & 23.9 & 14.0 & 22.8 & 10.5  & 6.5  & 13.8  & 9.8  & 3.4  & 13.9 & 3.7   & 2.1  & 2.7   & 6.9  & 0.5 \\
               
JULE~\cite{yang2016Joint}         & 27.7 & 18.2 & 16.4 & 27.2 & 19.2  & 13.8 & 13.7  & 10.3 & 3.3  & 13.8 & 5.4   & 2.8  & 3.3   & 10.2 & 0.6 \\
              
SAE~\cite{ng2011sparse}          & 32.0 & 25.2 & 16.1 & 29.7 & 24.7  & 15.6 & 15.7  & 10.9 & 4.4  & --   & --    & --   & --    & --   & --   \\
               
DAE~\cite{vincent2010stacked}          & 30.2 & 22.4 & 15.2 & 29.7 & 25.1  & 16.3 & 15.1  & 11.1 & 4.6  & 19.0 & 10.4  & 7.8  & 3.9   & 12.7 & 0.7 \\
             
% SWWAE        & 27.0 & 19.6 & 13.6 & 28.4 & 23.3  & 16.4 & 14.7  & 10.3 & 3.9 & --   & --    & --   & --    & --   & --\\
             
AE~\cite{bengio2007greedy}           & 30.3 & 25.0 & 16.1 & 31.4 & 23.4  & 16.9 & 16.5  & 10.0 & 4.7 & 18.5 & 10.4  & 7.3  & 4.1   & 13.1 & 0.7 \\
             
% DCGAN~\cite{radford2016unsupervised}        & 29.8 & 21.0 & 13.9 & 31.5 & 26.5  & 17.6 & 15.1  & 12.0 & 4.5 & 17.4 & 12.1  & 7.8  & 4.1   & 13.5 & 0.7  \\
              
% DeCNN        & 29.9 & 22.7 & 16.2 & 28.2 & 24.0  & 17.4 & 13.3  & 9.2  & 3.8  & 17.5 & 9.8   & 7.3  & 3.5   & 11.1 & 0.6  \\
              
VAE~\cite{kingma2014autoencoding}          & 28.2 & 20.0 & 14.6 & 29.1 & 24.5  & 16.7 & 15.2  & 10.8 & 4.0 & 17.9 & 10.7  & 7.9  & 3.6   & 11.3 & 0.6  \\
              
DEC~\cite{xie2016unsupervised}          & 35.9 & 27.6 & 18.6 & 30.1 & 25.7  & 16.1 & 18.5  & 13.6 & 5.0 & 19.5 & 12.2  & 7.9  & 3.7   & 11.5 & 0.7 \\
              
ADC~\cite{haeusser2018associative}          & 53.0 & --   & --  & 32.5 & --    & --   & 16.0  & --   & --  & --   & --    & --   & --    & --   & --   \\
             
DeepCluster~\cite{caron2018Deep}  & 33.4 & --   & -- & 37.4 & --    & --   & 18.9  & --   & -- & --   & --    & --   & --    & --   & --   \\
             
DAC~\cite{chang2017deep}          & 47.0 & 36.6 & 25.6 & 52.2 & 40.0  & 30.1 & 23.8  & 18.5 & 8.8 & 27.5 & 21.9  & 11.1 & 6.6   & 19.0 & 1.7 \\
             
DDC~\cite{chang2019deep}          & 48.9 & 37.1 & 26.7 & 52.4 & 42.4  & 32.9 & -- & --    & --  & --   & --    & --   & --     & --    & --   \\  
              
DCCM~\cite{wu2019deep}         & 48.2 & 37.6 & 26.2 & 62.3 & 49.6  & 40.8 & 32.7  & 28.5 & 17.3 & 38.3 & 32.1  & 18.2 & 10.8  & 22.4 & 3.8\\
              
IIC~\cite{ji2019invariant}          & 59.6 & 49.6 & 39.7 & 61.7 & 51.1  & 41.1 & 25.7  & 22.5 & 11.7 & --   & --    & --   & --    & --   & --\\
             
PICA~\cite{huang2020deep}         & 71.3 & 61.1 & 53.1 & 69.6 & 59.1  & 51.2 & 33.7  & 31.0 & 17.1 & 35.2 & 35.2  & 20.1 & 9.8   & 27.7 & 4.0\\

GCC~\cite{zhong2021graph}      & 78.8 & 68.4 & 63.1 & 85.6 & 76.4 & 72.8 & 47.2 & 47.2 & 30.5 & 52.6 & 49.0 & 36.2 & 13.8 & 34.7 & 7.5\\

CC~\cite{li2021contrastive}           & 85.0 & 76.4 & 72.6 & 79.0 & 70.5  & 63.7 & 42.9  & 43.1 & 26.6 & 42.9 & 44.5  & 27.4 & 14.0  & 34.0 & 7.1\\ 

TCL~\cite{li2022twin}           & 86.8 & 79.9 & 75.7 & 88.7 & 81.9  & 78.0 & 53.1  & 52.9 & 35.7 & 64.4 & 62.3  & 51.6 & --  & -- & -- \\ 

\midrule
SCAN$^*$ (Avg{$\pm$Std})  & 75.5{$\pm$2.0} & 65.4{$\pm$1.2} & 59.0{$\pm$1.6} & 81.8{$\pm$0.3} & 71.2{$\pm$0.4} & 66.5{$\pm$0.4} & 42.2{$\pm$3.0} & 44.1{$\pm$1.0} & 26.7{$\pm$1.3} & 55.6{$\pm$1.5} & 58.7{$\pm$1.3} & 42.8{$\pm$1.3} & 41.1{$\pm$0.5} & 69.4{$\pm$0.3} & 32.7{$\pm$0.4}\\

SCAN$^\dagger$ (Avg{$\pm$Std}) & 76.7{$\pm$1.9} & 68.0{$\pm$1.2} & 61.6{$\pm$1.8}   & 87.6{$\pm$0.4} & 78.7{$\pm$0.5} & 75.8{$\pm$0.7} & 45.9{$\pm$2.7} & 46.8{$\pm$1.3} & 30.1{$\pm$2.1} & 59.2{$\pm$0.2} & 60.8{$\pm$0.4} & 45.3{$\pm$0.4} & \textbf{--}    & \textbf{--}   & \textbf{--}\\
SCAN$^\dagger$ (Best)~\cite{gansbeke2020scan} & 80.9 & 69.8 & 64.6  & 88.3 &79.7 & 77.2 & 50.7 & 48.6 & 33.3 & 59.3 & 61.2 & 45.7 & 42.0 & 69.8 & 33.2\\

NNM~\cite{dang2021nearest}          & 76.8{$\pm$1.2} & 66.3{$\pm$1.3} & 59.6{$\pm$1.5} & 83.7{$\pm$0.3} & 73.7{$\pm$0.5} & 69.4{$\pm$0.6} & 45.9{$\pm$0.2} & 48.0{$\pm$0.4} & 30.2{$\pm$0.4} & 58.6{$\pm$1.5} & 60.4{$\pm$0.5} & 44.9{$\pm$0.2} & 37.8{$\pm$0.1} & 66.3{$\pm$0.1} & 27.1{$\pm$0.1} \\

SIC (direct) (Avg{$\pm$Std}) & 95.5{$\pm$0.1} & 92.7{$\pm$0.2} & 91.1{$\pm$0.2} & 78.3{$\pm$0.1} & 74.3{$\pm$0.1} &  66.9{$\pm$0.1} & 51.3{$\pm$0.1} & 53.9{$\pm$0.1} & 36.8{$\pm$0.1} & 59.0{$\pm$0.2} &  57.7{$\pm$1.8} & 41.1{$\pm$3.2}  & 55.7{$\pm$0.8} &  77.4{$\pm$0.1} & 44.9{$\pm$0.6}  \\

SIC (center-based) (Avg{$\pm$Std}) & 96.7{$\pm$0.1} & 93.7{$\pm$0.1} & 93.2{$\pm$0.1} & 91.8{$\pm$0.1} & 83.4{$\pm$0.1} &  83.1{$\pm$0.1} & 54.0{$\pm$0.1} & 54.4{$\pm$0.4} & 38.6{$\pm$0.4} & 61.8{$\pm$1.1} & 63.9{$\pm$1.9} & 49.8{$\pm$1.4} & 61.0{$\pm$0.2} & 80.4{$\pm$0.1} & 51.2{$\pm$0.2} \\

SIC (adjusted center-based) (Avg{$\pm$Std}) & 98.1{$\pm$0.1} & 95.3{$\pm$0.1} & 95.9{$\pm$0.1} & 92.6{$\pm$0.1} & 84.7{$\pm$0.1} &  84.4{$\pm$0.1} & 58.3{$\pm$0.1} & 59.03{$\pm$0.1} & 43.9{$\pm$0.1} & 69.7{$\pm$1.1} & 69.0{$\pm$1.6} & 55.8{$\pm$1.5} & 60.2{$\pm$0.3} & 79.4{$\pm$0.1} & 49.4{$\pm$0.2} \\

\textbf{SIC (Best)}~\cite{cai2022semantic} & 
98.1 & 95.4 & 95.9 & 92.67 & 84.8 & 84.6 & 58.4 & 59.3 & 44.0 & 71.3 & 71.8 & 58.6 & 61.2 & 80.5 & 51.4 \\

Our method (Avg{$\pm$Std}) & 98.1{$\pm$0.1} & 95.5{$\pm$0.1} & 96.0{$\pm$0.1} & 92.7{$\pm$0.2} & 84.9{$\pm$0.2} &  84.6{$\pm$0.2} & 59.7{$\pm$0.9} & 59.8{$\pm$0.5} & 44.0{$\pm$0.9} & {74.9$\pm$2.5} & {73.3$\pm$1.5} & {61.6$\pm$2.5} & {61.2$\pm$0.5} & {79.7$\pm$0.7} & {51.9$\pm$0.8} \\

\textbf{Our method (Best)} & 
\textbf{98.2} & \textbf{95.5} & \textbf{96.0} & \textbf{92.8} & \textbf{85.0} & \textbf{84.9} & \textbf{61.2} & \textbf{60.6} & \textbf{45.5} & \textbf{77.9} & \textbf{75.1} & \textbf{64.3} & \textbf{61.9} & \textbf{81.1} & \textbf{52.3} \\

\bottomrule
\end{tabular}}
\caption{Clustring results on five benchmark datasets. The best results are highlighted in bold.}
\label{tab_result}
\end{table*}

\section{Experiments and Analysis}
\label{sec:experiments}
In this section, experiments are conducted on five image benchmark datasets to validate the effectiveness of our proposed method.

\subsection{Experimental Setup}
\noindent\textbf{Benchmarks and implementation details.} We used the followifive benchmark datasets in our experiment: STL10~\cite{adam2011an}, Cifar10~\cite{krizhevsky2009learning}, Cifar100-20~\cite{krizhevsky2009learning}, ImageNet-Dogs~\cite{chang2017dog} and Tiny-ImageNet~\cite{ya2015tiny}. 

\noindent\textbf{Evaluation Metrics.}
We used three evaluation metrics to evaluate clustering results, including clustering Accuracy (ACC), Normalized Mutual Information (NMI)~\cite{mcdaid2011normalized}, and Adjusted Rand Index (ARI)~\cite{hubert1985comparing}. For these metrics, a higher value means better performance.

\subsection{Comparisons with State-of-the-arts}
\noindent\textbf{Setup.} We took the entire list of nouns in the WordNet dataset~\cite{miller1995wordnet} to form an initial semantic dataset for filtering which contains more than 82, 000 nouns. To evaluate the effectiveness of our proposed method, we compare it with 27 state-of-the-art clustering methods on the five datasets. For SC~\cite{zelnik2005Self}, NMF~\cite{cai2009Locality}, AE~\cite{bengio2007greedy}, DAE~\cite{vincent2010stacked}, DCGAN~\cite{radford2016unsupervised} and VAE~\cite{kingma2014autoencoding}, the clustering results were obtained by $k$-means. In addition, to better represent the generalization properties of our methods for novel unseen examples, we used the same train and validation splits as those in SCAN~\cite{gansbeke2020scan} and SIC~\cite{cai2022semantic}. We repeated the training five times independently on each dataset and reported their mean and standard deviation values. The nearest neighbors were searched through Faiss Library ~\cite{jeff2021billion}.
% The best hyper-parameters on five benchmark datasets are shown in Appendix.

\noindent\textbf{Results.}
The comparison results with the state-of-the-art methods in terms of ACC, NMI, and ARI are presented in Tabel~\ref{tab_result}. From this table, we can observe that our method outperforms all other methods on five benchmark datasets. Especially, MCA improves ACC, NMI, and ARI by 2.8\%, 1.3\%, and 1.5\% on Cifar100-20, and 6.6\%, 3.3\%, and 5.7\% on ImageNet-Dogs, demonstrating that MCA better amend the incorrect alignments and thus achieve significant performance improvement. Especially for fine-grained images such as Imagenet-dogs whose categories are difficult to distinguish, our method can provide good clustering assistance information for them through the alignment strategy. Our method can give useful clustering support information for images through the three alignment strategy, especially for fine-grained images like ImageNet-Dogs whose categories are difficult to distinguish.

\subsection{Ablation Studies}

\noindent\textbf{Semantic space construction.}

\begin{table}[!h]
  \begin{center}
   {
  \begin{tabular}{c|cc}\toprule
   Steps & Cifar10 & ImageNet-Dogs \\
  \hline
  \textbf{UF} & 91.4 & 65.3 \\
  \textbf{UF}+\textbf{HF} & 91.9 & 71.8
  \\
  \bottomrule
  \end{tabular}}
  \end{center}
  \caption{Ablation studies of semantic space construction (we did not report the clustering results without filtering because the number of initial words is too large. \textbf{UF: Uniqueness-based filtering}, \textbf{HF: Hierarchy-based filtering}.}
  \label{tab:filter_strategy_ablation_loss}
\end{table}

We first conduct an experiment on Cifar10 and ImageNet-Dogs to verify the effectiveness of our semantic space construction method and show the results in Table ~\ref{tab:filter_strategy_ablation_loss}. The results demonstrate that \textbf{hierarchy-based filtering} can significantly enhance the semantic space compared to \textbf{uniqueness-based filtering}, especially when dealing with complex clusters that are challenging to distinguish, such as those in the ImageNet-Dogs dataset.

\begin{table}[!h]
	\centering
        \small
		\begin{tabular}{cccc|lll}
		\toprule[1pt]
		\multicolumn{4}{c|}{\textbf{Loss Components}} & \multicolumn{3}{c}{\textbf{Result}} \\
		 $\mathcal{L}_{I}$  & $\mathcal{L}_{ia}$ & $\mathcal{L}_{pa}$ &   $\mathcal{L}_{sa}$ & ACC  & NMI  & ARI \\ 	\toprule[1pt]
		$\checkmark$ &       &         &         &        46.7{$\pm$0.5}     & 48.9 {$\pm$0.7}     &     32.6  {$\pm$0.5} \\          
         $\checkmark$   &               &  $\checkmark$ &   $\checkmark$  &      66.0  {$\pm$3.4} &          67.7 {$\pm$2.5}			&   50.7 {$\pm$4.1}   \\

          $\checkmark$   &      $\checkmark$           &  &   $\checkmark$ &       63.5 {$\pm$2.5}  &          69.7 {$\pm$2.1}   &   52.7 {$\pm$3.5}   \\

         $\checkmark$   &      $\checkmark$           & $\checkmark$  &    &       53.1 {$\pm$3.8}  &        55.4 {$\pm$2.8}   &    37.4{$\pm$2.4}   \\
         
		$\checkmark$    &      $\checkmark$       & $\checkmark$&     $\checkmark$   &    \textbf{ 74.8 {$\pm$2.7}}  &       \textbf{  72.7	{$\pm$2.0}}	&  \textbf{ 61.9  {$\pm$2.5}}    \\ 	\toprule[1pt]
	\end{tabular}
	\centering
	\caption{Ablation studies on ImageNet-Dogs.}
	\label{experiment:loss_ablation}
\end{table}

\noindent\textbf{Loss components effectiveness.}
 We perform an ablation analysis on ImageNet-Dogs to measure the importance of four loss components in our model, i.e., image consistency loss $\mathcal{L}_{I}$, instance-level alignment loss $\mathcal{L}_{ia}$, prototype-level alignment loss $\mathcal{L}_{pa}$ and semantic-level alignment $\mathcal{L}_{sa}$. The results are shown in Tabel~\ref{experiment:loss_ablation}, indicating that each of the four components plays an important role. The integration of three cross-modal alignment strategies significantly enhances the clustering performance, even obtaining 28.1\%, 23.8\%, and 29.3\% performance gains when all three strategies are simultaneously used. Among the three cross-modal alignment strategies, introducing semantic-level alignment yields the most significant improvement in clustering performance, indicating that operating at the semantic level can effectively address incorrect alignments. These results confirm the effectiveness of our proposed cross-modal alignment methods.

\begin{figure}[!htb]
% \begin{figure}[h]
\centering
    \subfloat[Cifar100-20.]{
           \centering 
           \label{fig:Cifar100-20_pseudo}  
           \includegraphics[width = .4\columnwidth]{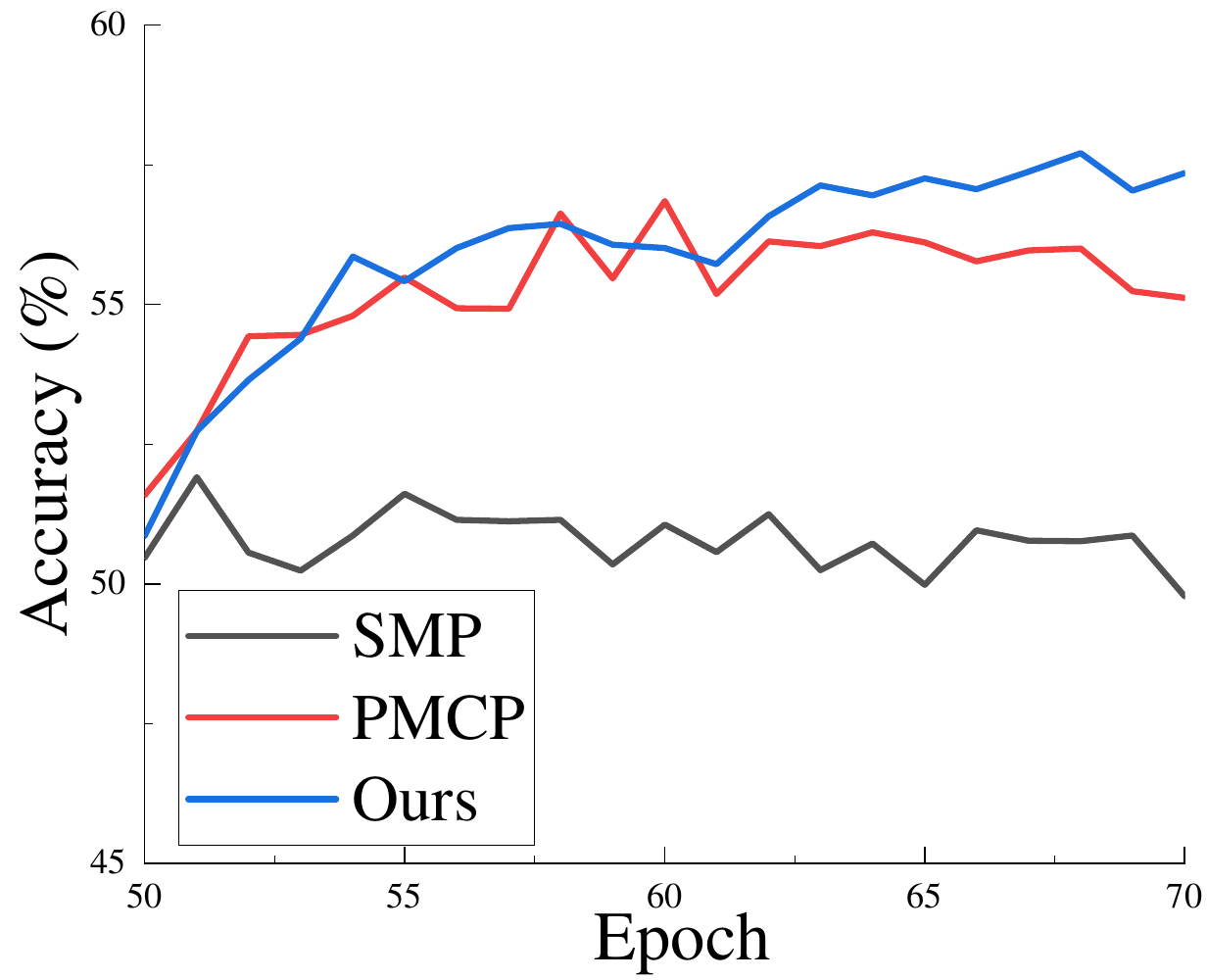}}\hspace{-3mm}
    \subfloat[ImageNet-Dogs.]{
            \centering 
            \label{fig:imagenet_dog_alignment}
            \includegraphics[width = .4\columnwidth]{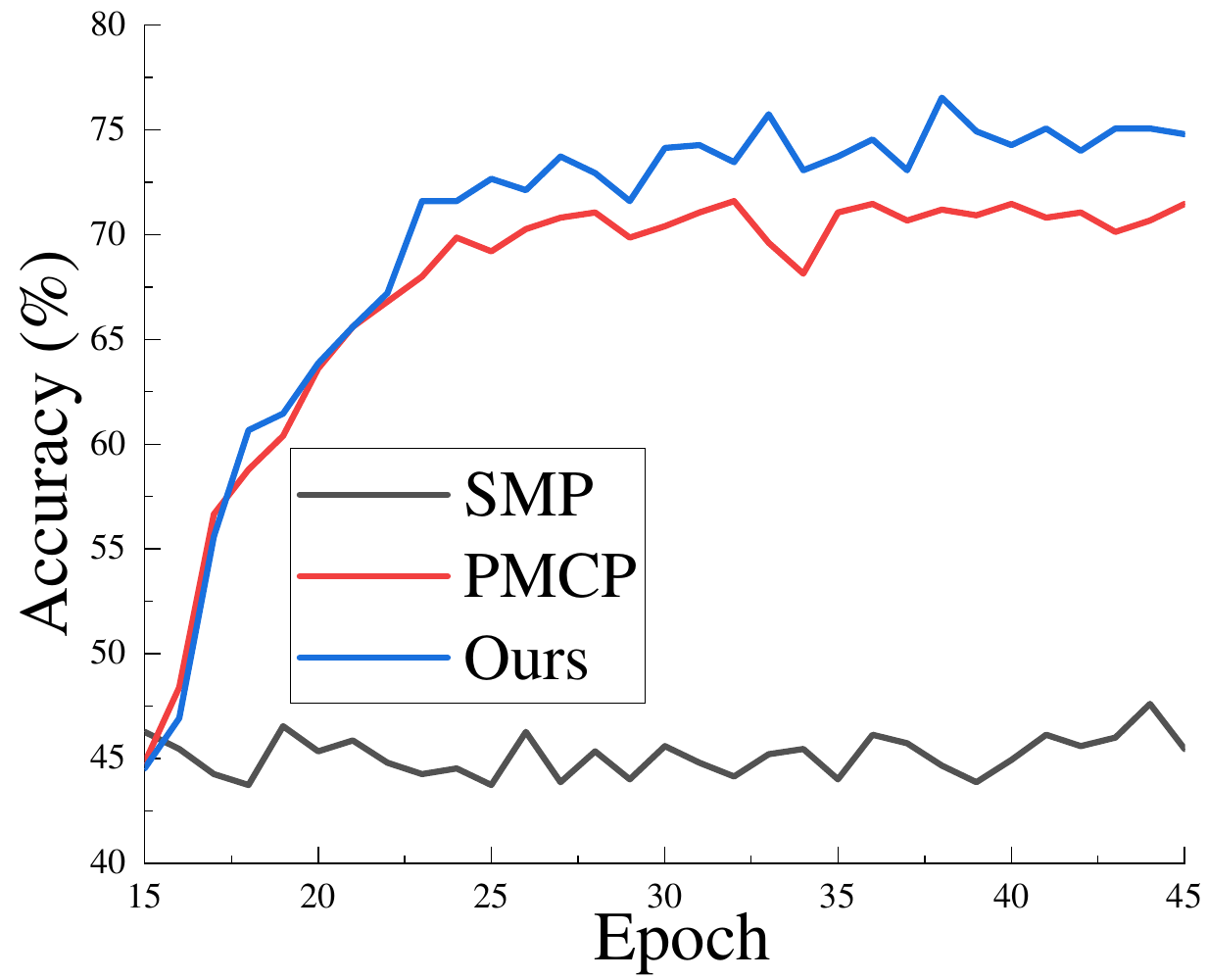}}
\caption{The average accuracy of 10 runs of pseudo-labels with epochs on Cifar100-20 and ImageNet-Dogs evolves.}
\label{fig:effciency_pesudo_label}
\end{figure}

\noindent\textbf{Comparison of three pseudo-label generation methods.} In our method, semantic-level cross-modal alignment can be considered self-training with cross-modal pseudo-labels. To verify the effectiveness of our method, we compared three pseudo-label generation methods implemented in our framework: 1) Single Modal Pseudo-labeling (\textbf{SMP}) that directly generates one-hot pseudo-labels via the argmax operation on the soft cluster assignments only from images, 2) Prototype Mapping based Cross-modal Pseudo-labeling (\textbf{PMCP}, is the adjusted center-based method in~\cite{cai2022semantic}) that generates pseudo-labels from the prototype level alignments in the original CLIP, and 3) \textbf{Ours} that generates pseudo-labels by simultaneously learning relationships among images and neighboring texts, while also updating the image and text embeddings. The comparison results on Cifar100-20 and ImageNet-Dogs are shown in Figure~\ref{fig:effciency_pesudo_label}, demonstrating that our method significantly outperforms the other two methods on both datasets. These results indicate that learning adaptive relationships at the semantic level substantially improves the quality of pseudo-labels.

\begin{figure}[!htb]
     \centering
     \includegraphics[width=0.8\columnwidth]{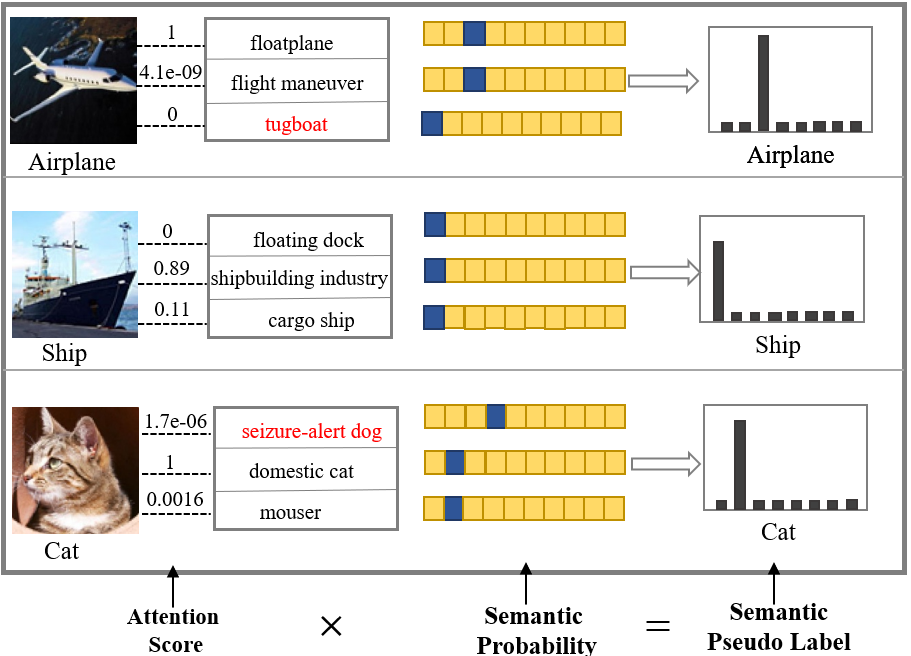}
     \caption{Example of pseudo-label generation in MCA. The words below (on the right side of) the images are ground-truth/ neighboring labels and the red color indicates irrelevant texts. The blue block in the semantic probability indicates the class the left word is assigned to (with the largest probability).}
     \label{fig:attention_pseudo_label}
 \end{figure}

Figure~\ref{fig:attention_pseudo_label} shows an example of our proposed pseudo-label generation method. In this example, we select an example from each of the three classes and three neighboring words for each image. Although these exist irrelevant neighboring words for an image, our method can identify irrelevant words and eliminate the affection of incorrect alignments in CLIP. 

\subsection{Sensitivity Analysis}

\begin{figure*}[t]
\centering
    \subfloat[\footnotesize $k_S$ on Cifar10.]{
           \centering 
           \label{fig:cifar10_k_t}  
           \includegraphics[width = .25\linewidth]{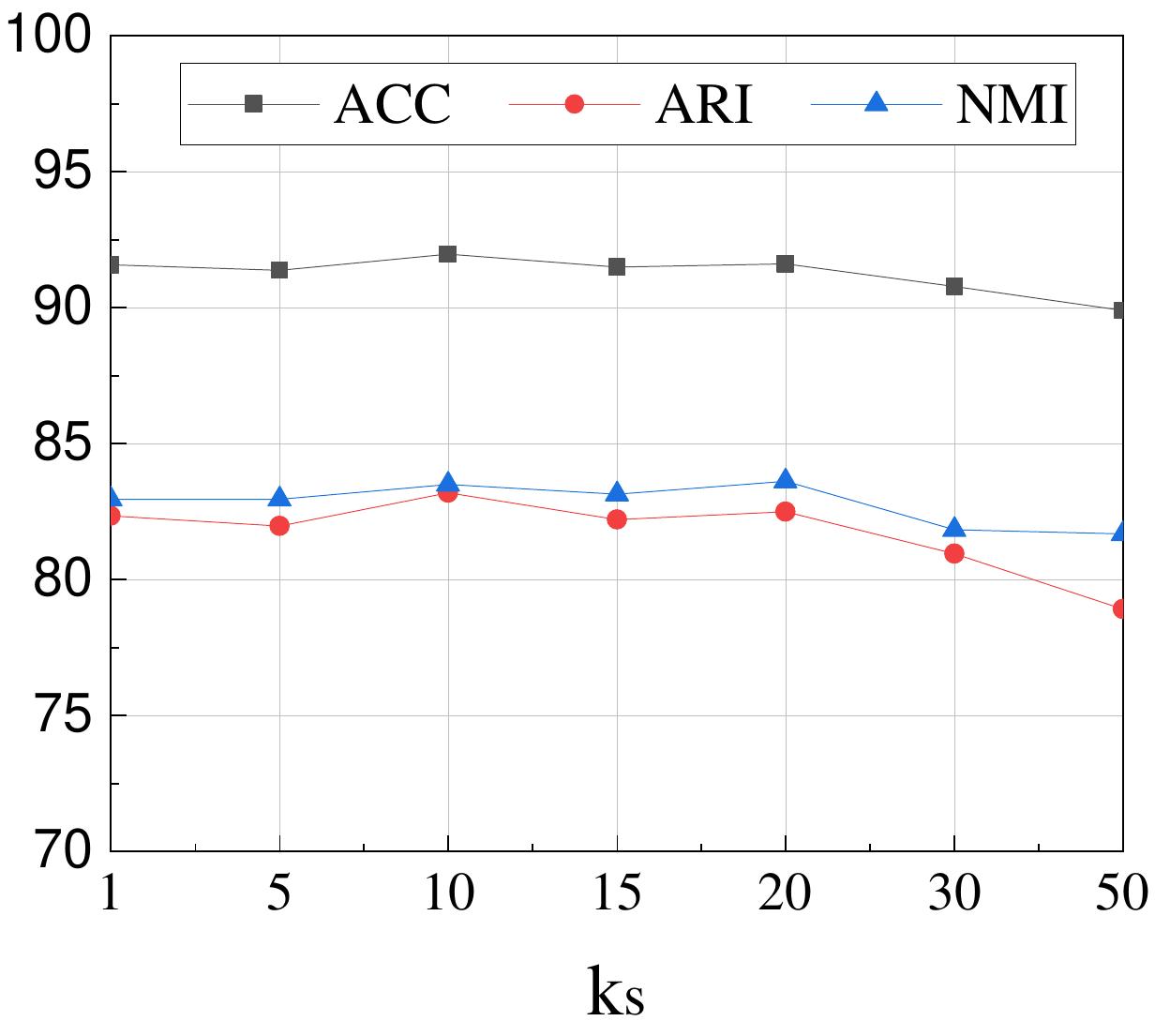}}
    \subfloat[\footnotesize $k_{S}$ on ImageNet-Dogs.]{
            \centering 
            \label{fig:imagenet_dog_k_t}
            \includegraphics[width = .24\linewidth]{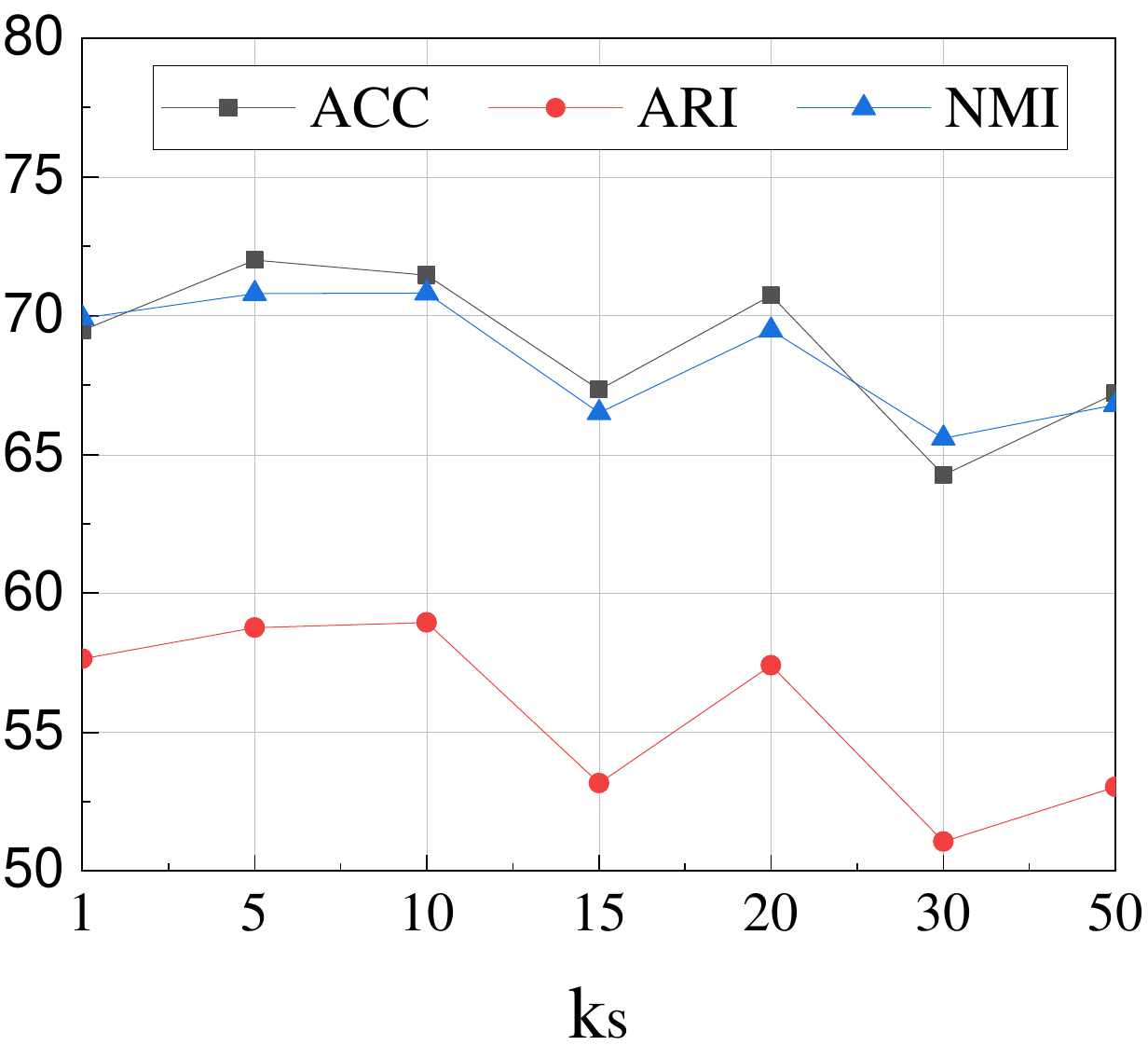}}
    \subfloat[\footnotesize $k_p$ on Cifar10.]{
           \centering 
           \label{fig:cifar10_ck}  
           \includegraphics[width = .25\linewidth]{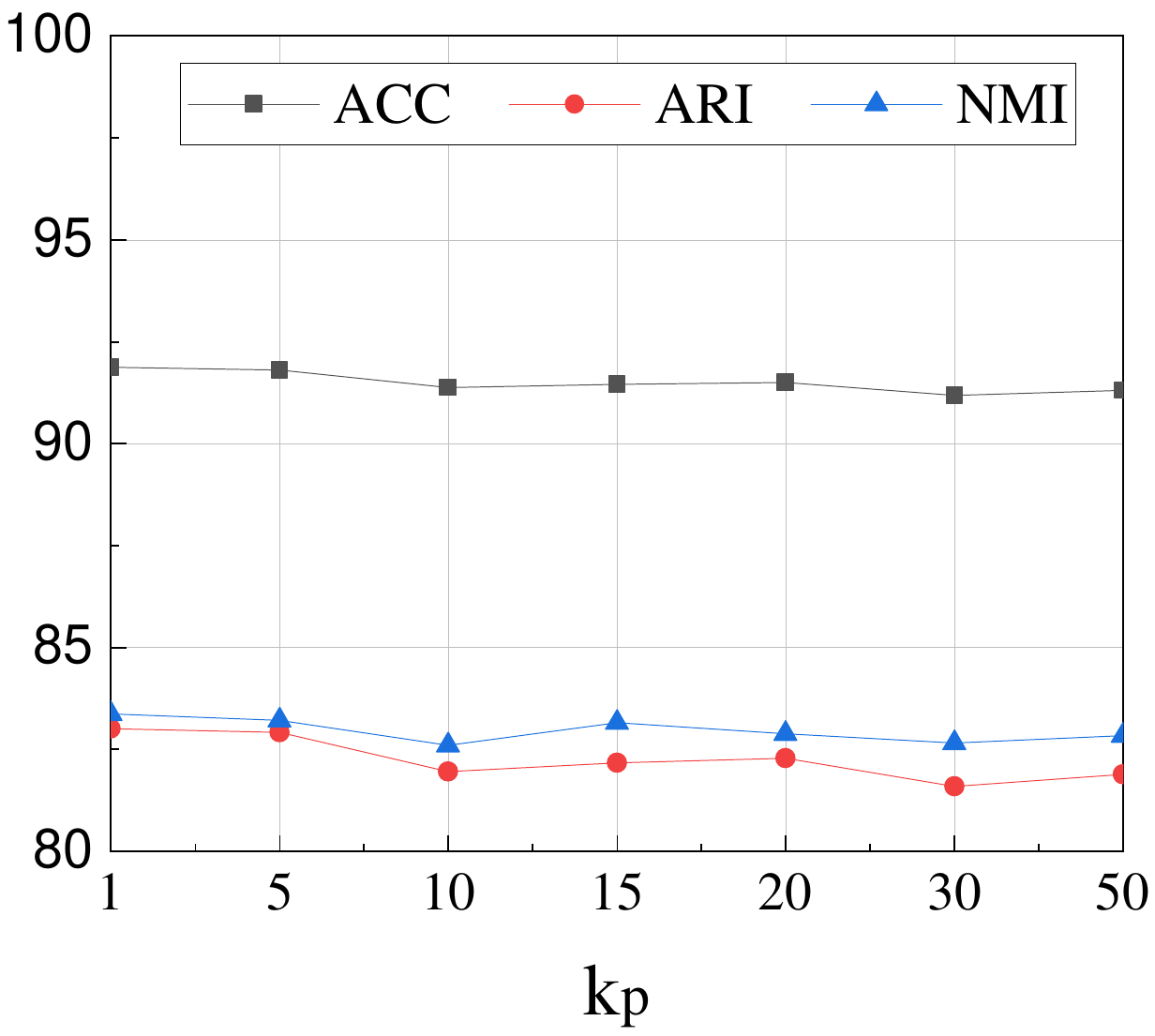}}
    \subfloat[\footnotesize $k_{p}$ on ImageNet-Dogs.]{
            \centering 
            \label{fig:imagenet_dog_k_p}
            \includegraphics[width = .24\linewidth]{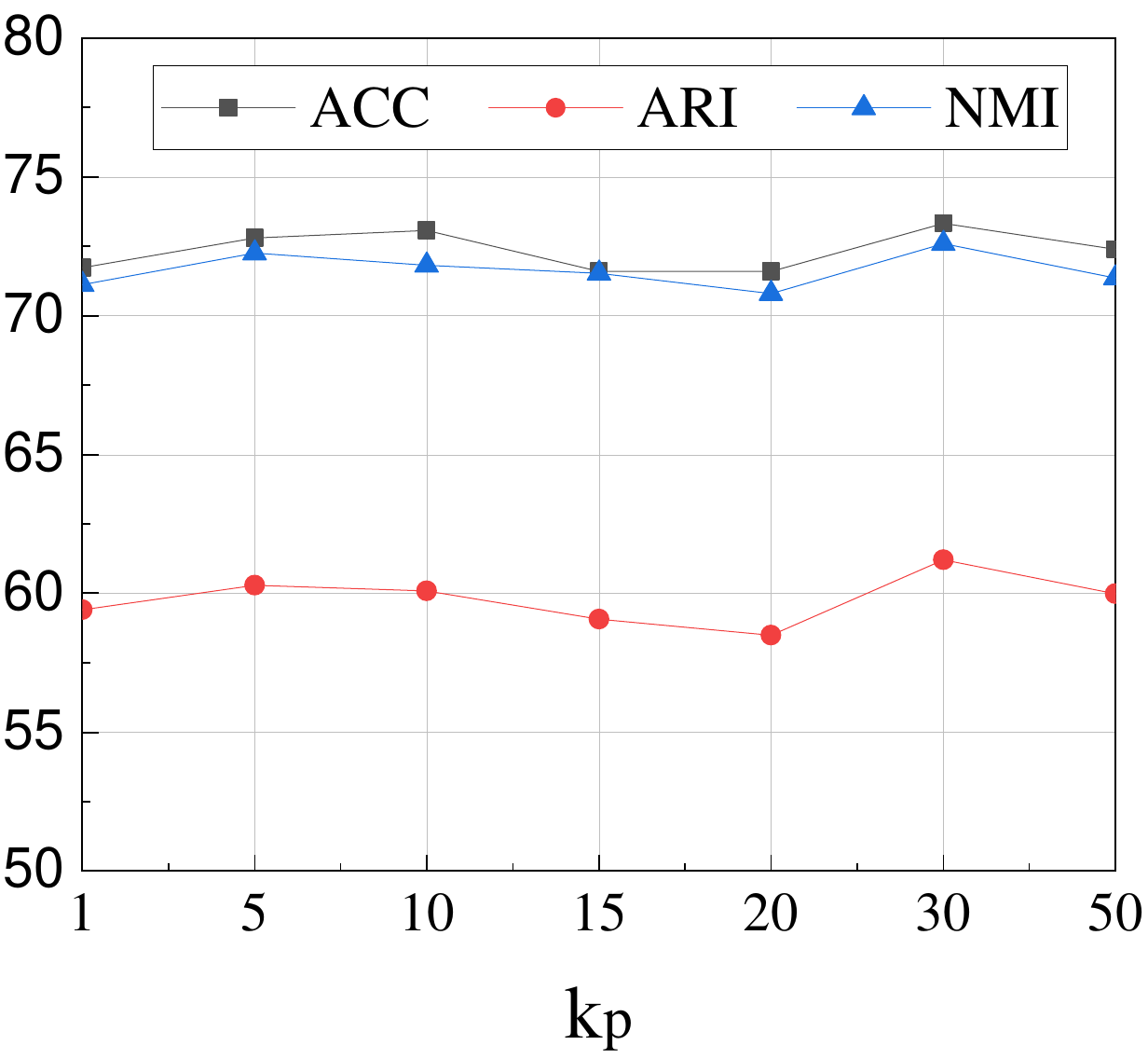}}             
\caption{Sensitivity analysis of $k_S$ and $k_p$. }
\label{fig:sens_k_t_p}
\end{figure*}

\noindent \textbf{Sensitivity on neighborhood parameters $k_{S}$ and $k_{p}$ in cross-modal alignment.} 
In our alignment method, $k_{S}$ controls the number of neighboring texts in the instance-level and semantic-level alignments, and $k_p$ controls the number of neighboring texts to recompute the semantic prototype in the prototype-level alignment. Figures~\ref{fig:cifar10_k_t} and~\ref{fig:imagenet_dog_k_t} show that too large $k_{S}$ causes performance degeneration due to the introduction of irrelevant texts. Figures~\ref{fig:cifar10_ck} and~\ref{fig:imagenet_dog_k_p} show that
$k_p$ does not change the performance too much.

\begin{figure*}[!htb]
\centering
    
\subfloat[\footnotesize ($ \eta$,$\lambda_a$) on STL10.]{
           \centering 
           \label{fig:stl10_eta_lambda}  
           \includegraphics[width = .23\linewidth]{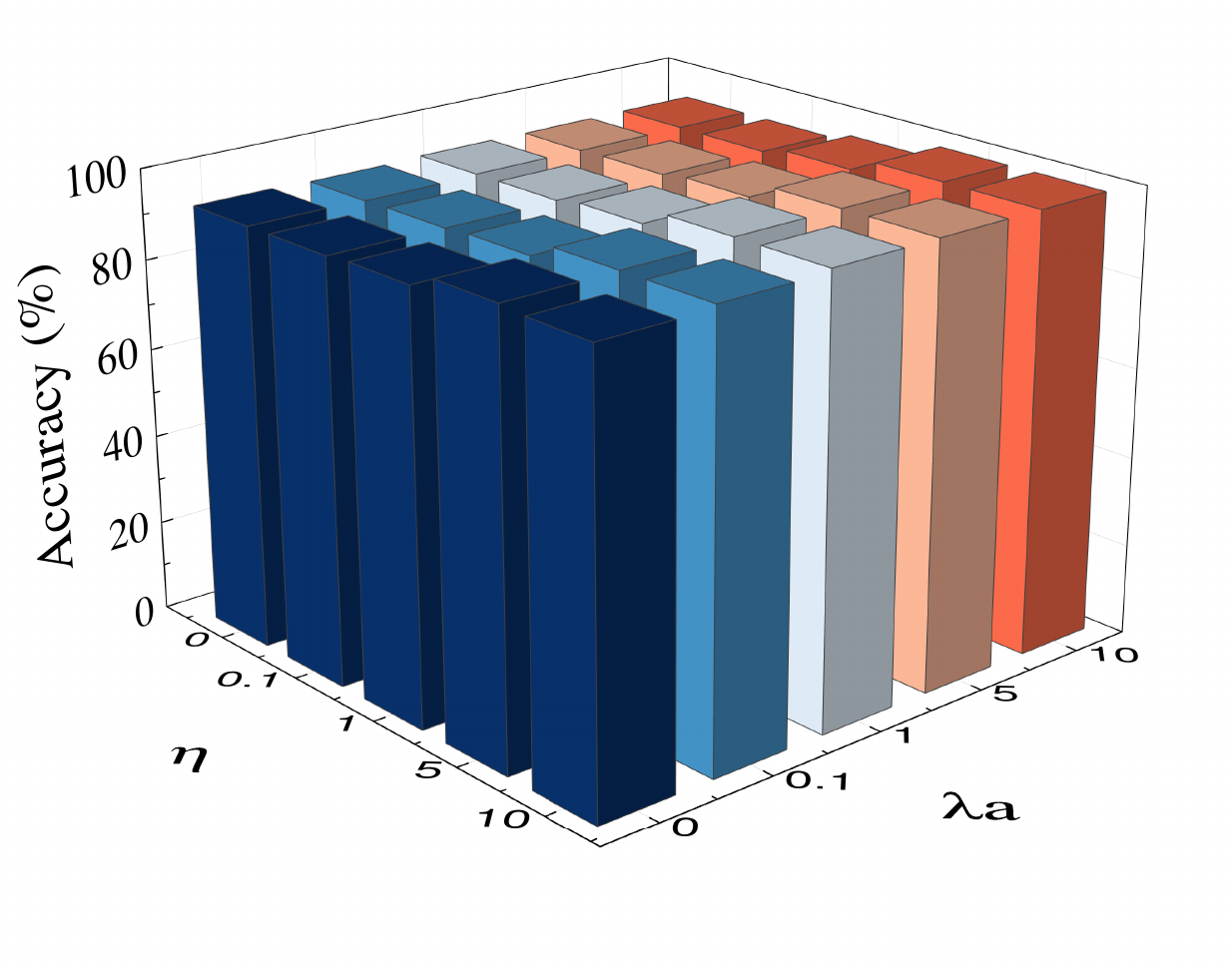}}\hspace{-3mm}
    \subfloat[\footnotesize ($ \eta$,$\lambda_a$) on ImageNet-Dogs.]{
            \centering 
            \label{fig:imagenet_dog_eta_lambda}
            \includegraphics[width = .24\linewidth]{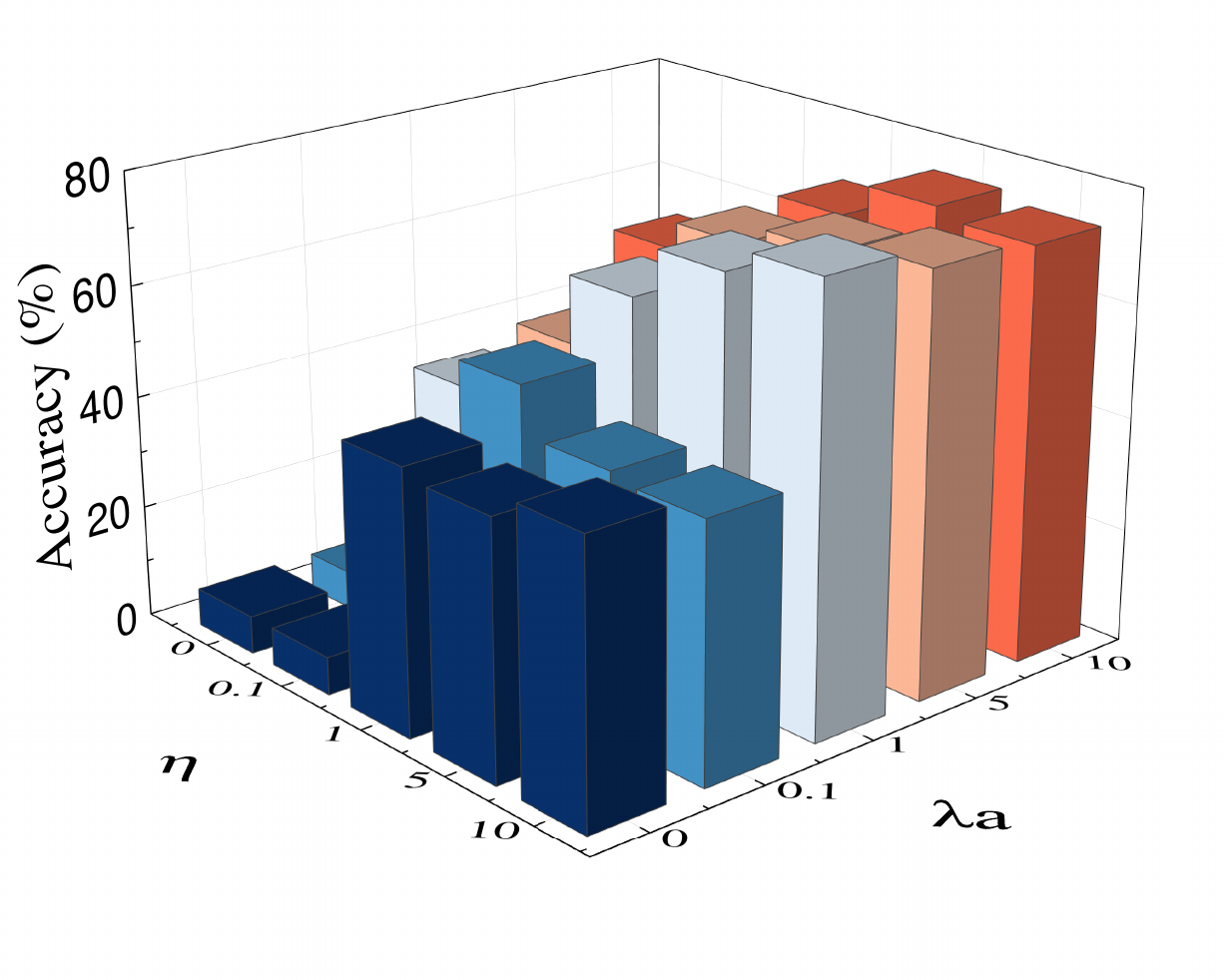}}
    \subfloat[\footnotesize ($\eta$,$\lambda_{pa}$) on STL10.]{
           \centering 
           \label{fig:stl10_eta_pa}  
           \includegraphics[width = .23\linewidth]{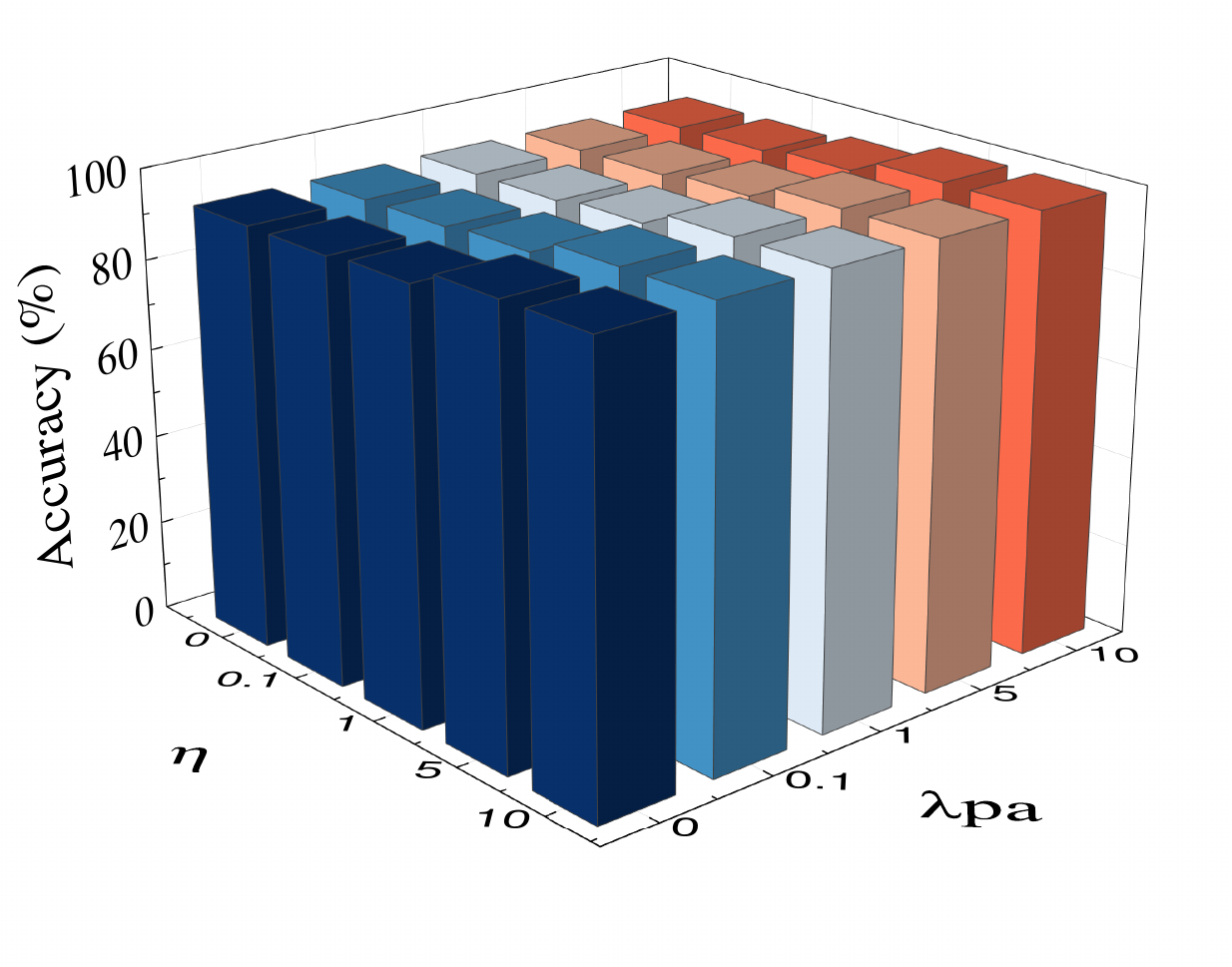}}\hspace{-3mm}
    \subfloat[\footnotesize ($\eta$,$\lambda_{pa}$) on ImageNet-Dogs.]{
            \centering 
            \label{fig:imagenet_dog_eta_pa}
            \includegraphics[width = .24\linewidth]{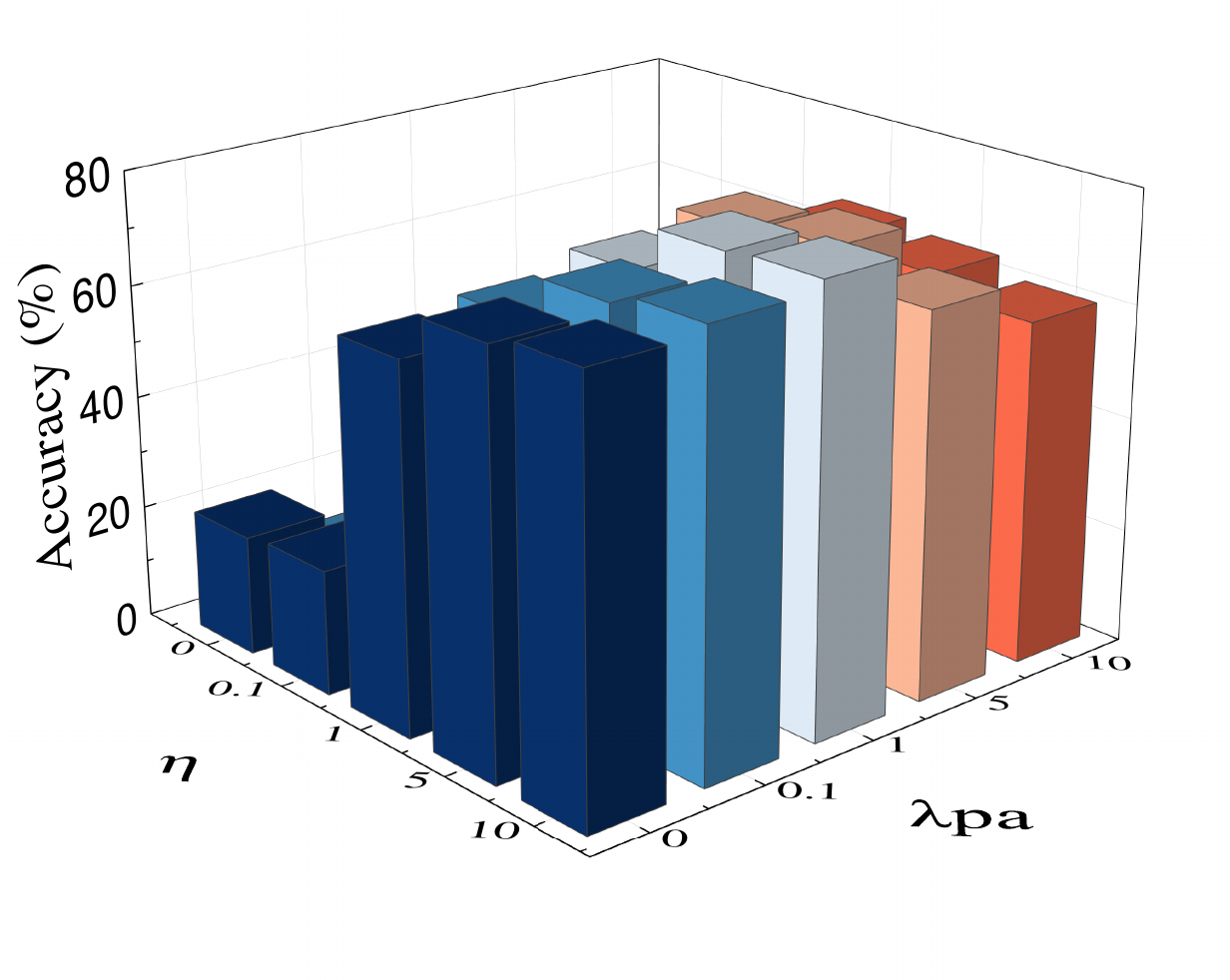}}
    \\
    \subfloat[\footnotesize ($\eta$,$\lambda_{sa}$) on STL10.]{
           \centering 
           \label{fig:stl10_eta_ai}  
           \includegraphics[width = .23\linewidth]{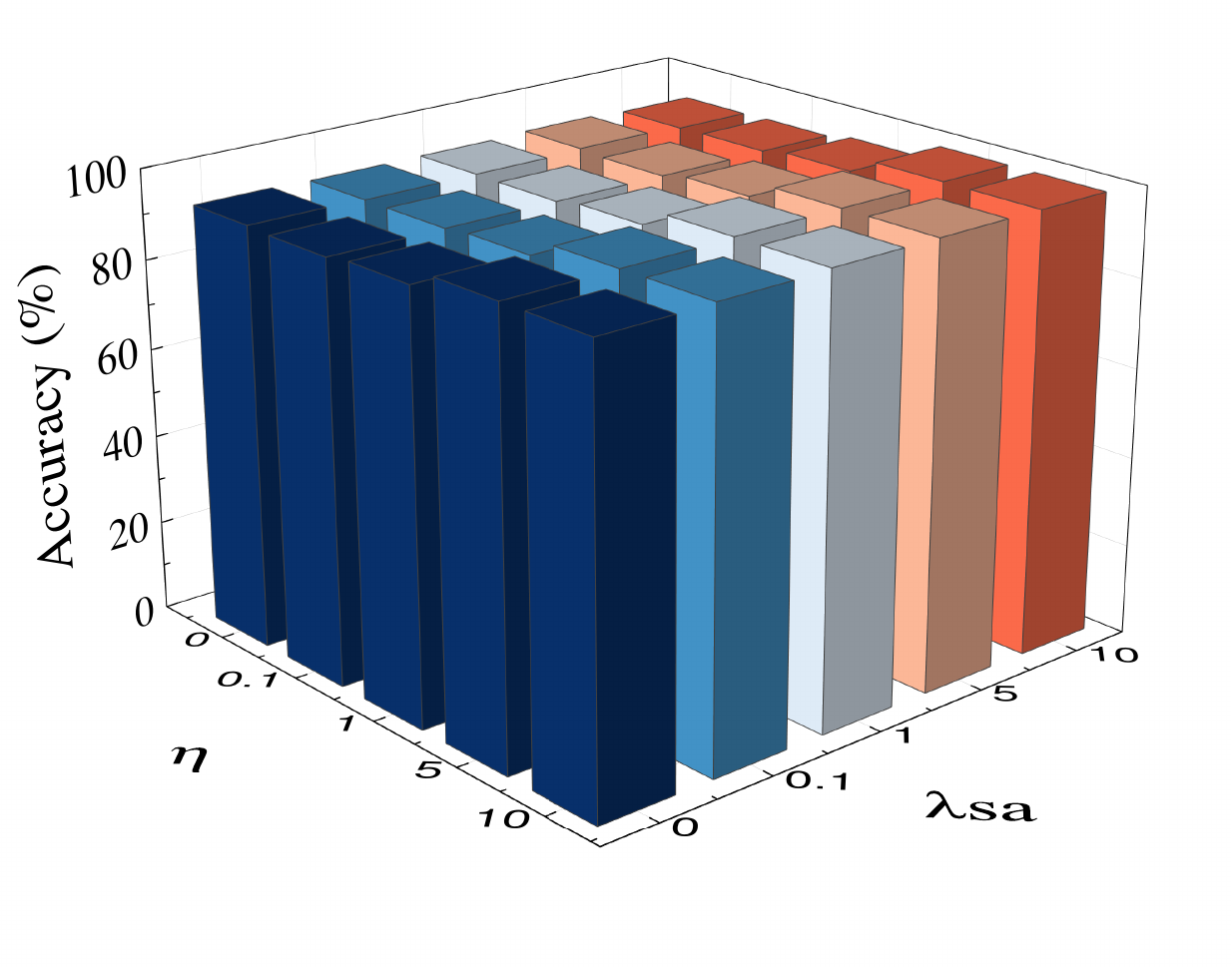}}\hspace{-3mm}
    \subfloat[\footnotesize ($\eta$,$\lambda_{sa}$) on ImageNet-Dogs.]{
            \centering 
            \label{fig:imagenet_dog_eta_ai}
            \includegraphics[width = .24\linewidth]{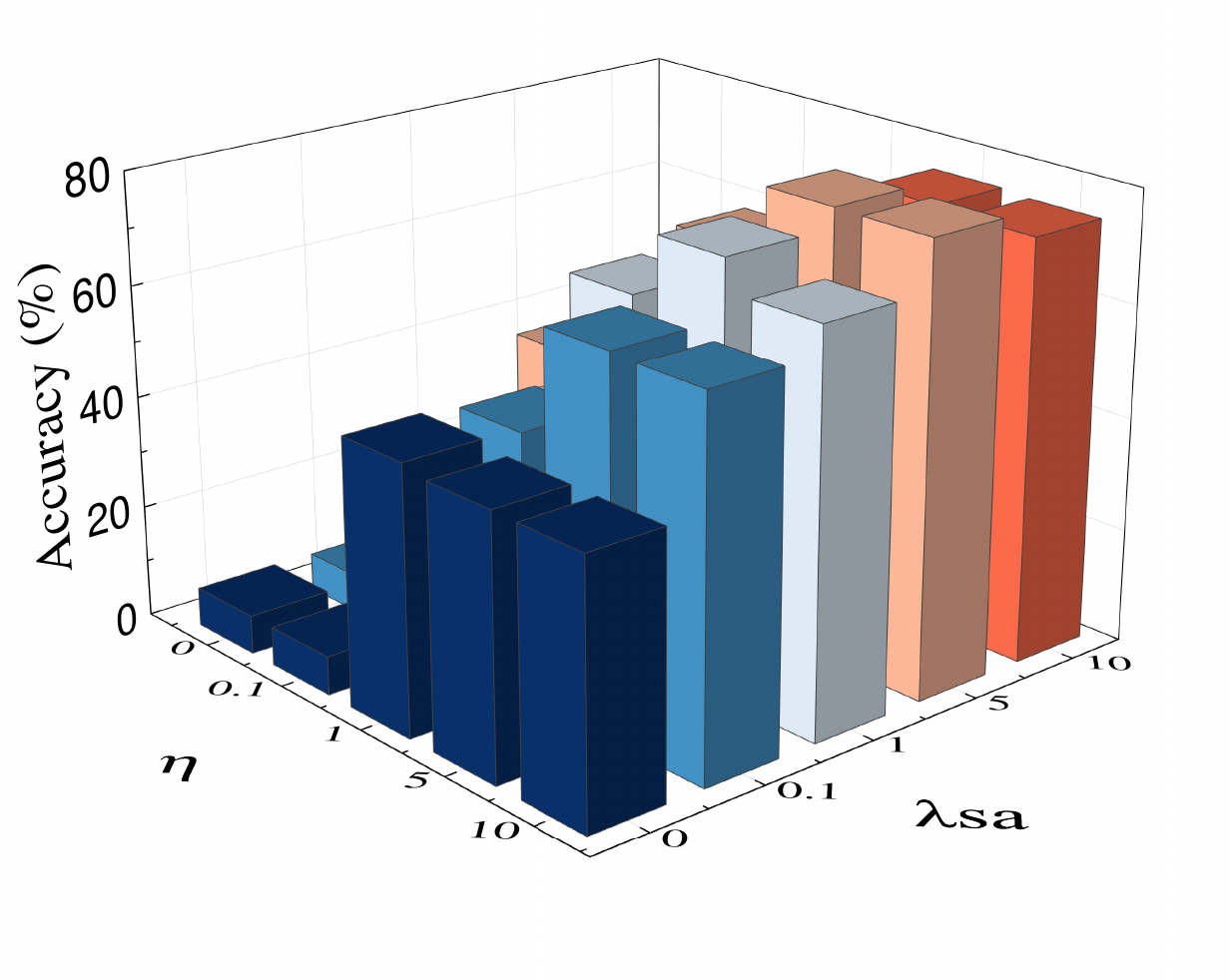}}
    \subfloat[\footnotesize ($ \lambda_{pa}$,$\lambda_{sa}$) on STL10.]{
           \centering 
           \label{fig:stl10_sp_ai}  
           \includegraphics[width = .23\linewidth]{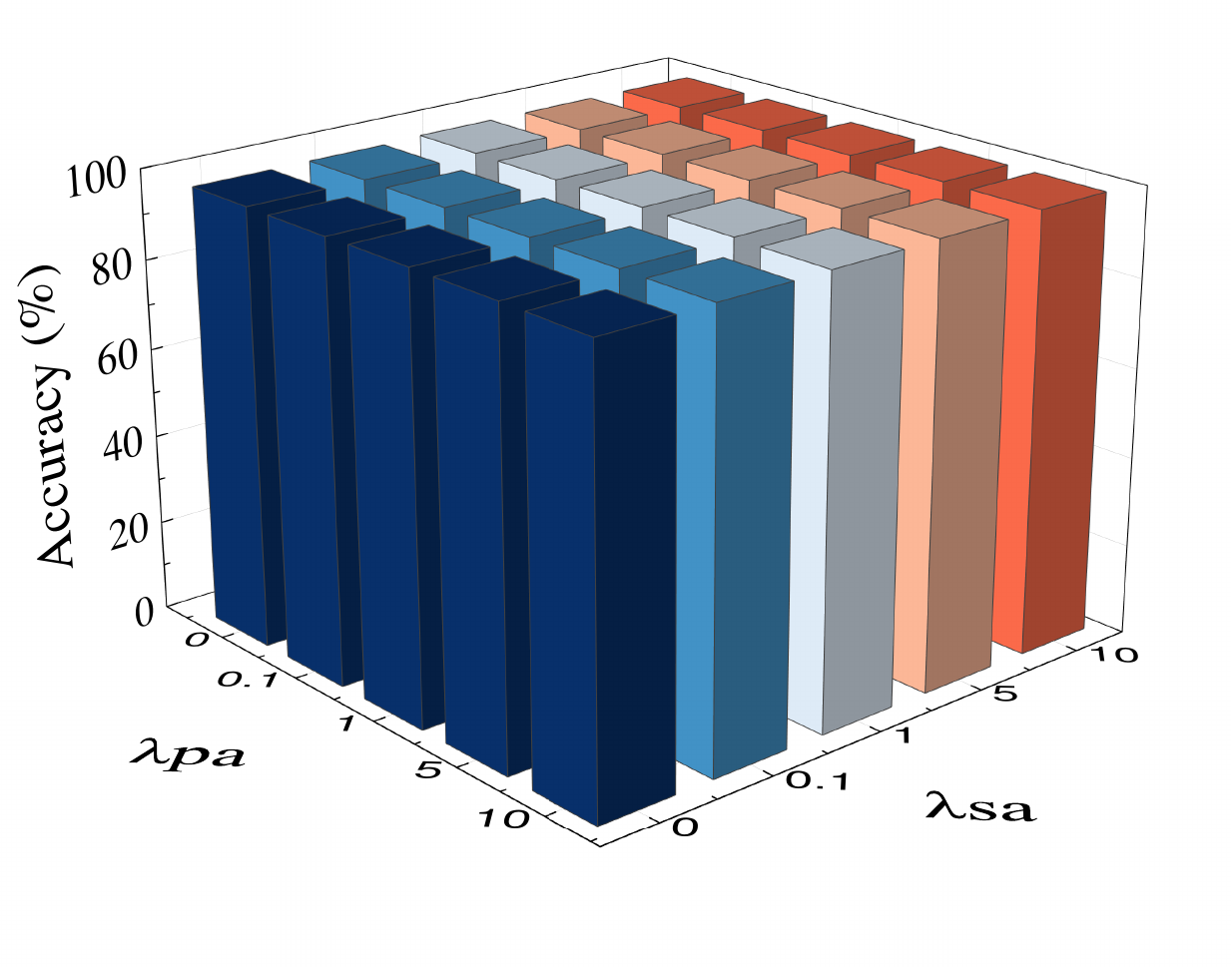}}\hspace{-3mm}
    \subfloat[\footnotesize ($ \lambda_{pa}$,$\lambda_{sa}$) on ImageNet-Dogs.]{
            \centering 
            \label{fig:imagenet_dog_sp_ai}
            \includegraphics[width = .24\linewidth]{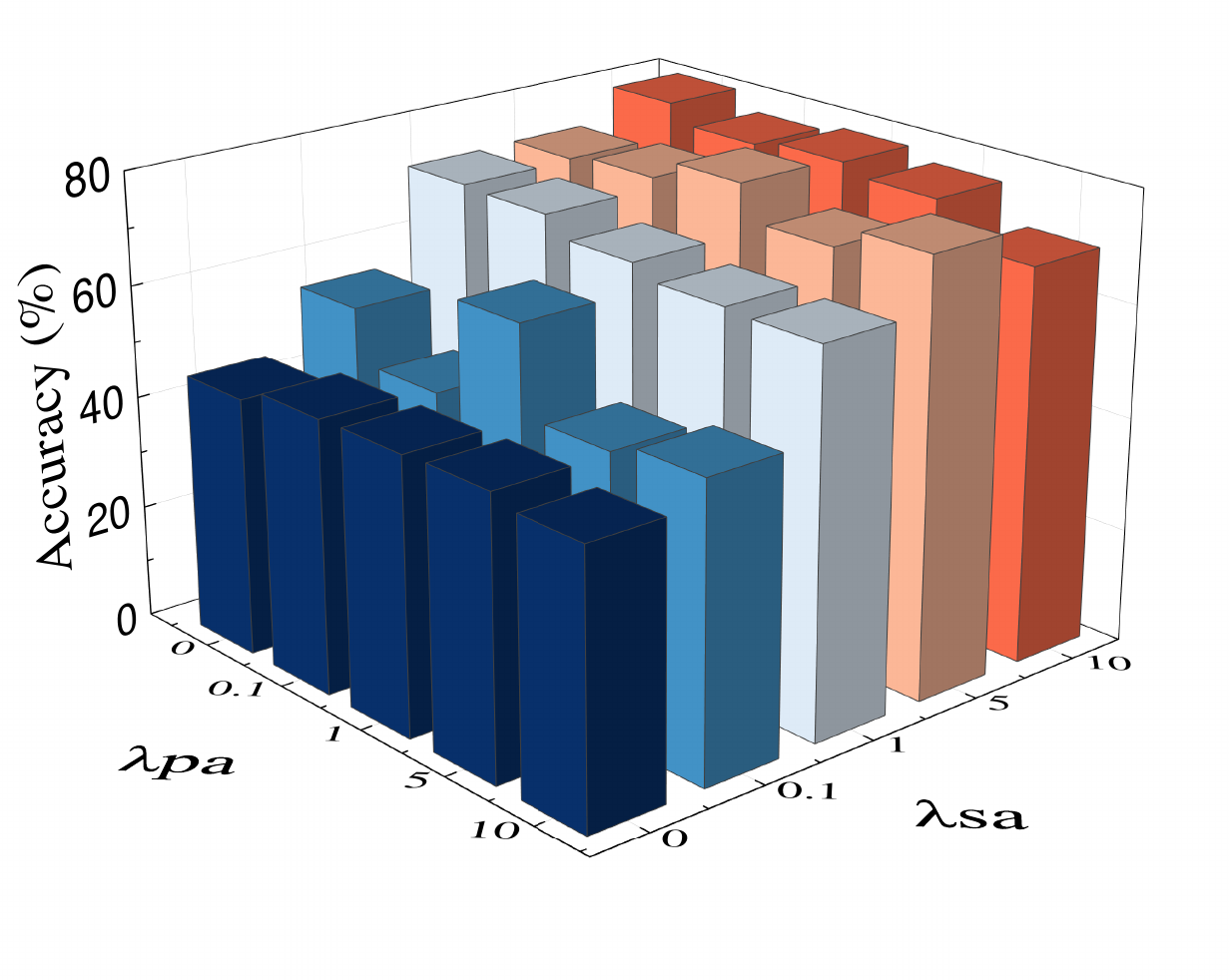}}            
            
\caption{Sensitivity analysis of trade-off parameters $\eta$, $\lambda_a$, $\lambda_{pa}$ and $\lambda_{sa}$. 
}
\label{fig:sens_trade_off}
\end{figure*}

\noindent \textbf{Sensitivity on trade-off parameters $\eta$, $\lambda_a$, $\lambda_{pa}$ and $\lambda_{sa}$.}
Figure~\ref{fig:sens_trade_off} shows the sensitivity analysis of trade-off parameters $\eta$, $\lambda_a$, $\lambda_{pa}$ and $\lambda_{sa}$. From these figures, we can observe that the performance of our method does not change too much with the change of four parameters on STL10. However, in ImageNet-Dogs, we can see that the performance of our method improves with increasing values of $\eta$, $\lambda_{pa}$, and $\lambda_{sa}$. In particular, our method appears to be more sensitive to changes in $\lambda_{sa}$ than $\lambda_{pa}$ on ImageNet-Dogs.

\subsection{Time Costs}
The training durations of SIC/MUST/MCA were approximately 0.5/2/1 hours on STL10, 1.5/6.5/3 hours on CIFAR10, 3/11.5/5 hours on CIFAR100-20, 0.5 / 1.7 / 1 hours on ImageNet-Dogs and 4/18/5 on Tiny-ImageNet. While it is true that our method generates a smaller semantic space, it does require the training of three networks, which can result in a longer training time compared to SIC. However, our method is much more efficient than MUST.

\section{Conclusion}
\label{sec:conclusion}
We have proposed a novel method to address the incorrect alignments in CLIP for image clustering. Our method includes the construction of a proper semantic space and a multi-level cross-modal alignment approach for aligning images and texts in downstream tasks at three levels. Theoretical results have shown interesting insights, and experimental results have demonstrated the superiority of our method. However, we acknowledge that our method may not be as cost-effective as SIC, as it involves three types of alignments. The proper setting of hierarchy levels also remains a challenge that needs further investigation. Additionally, our method exhibits lower performance compared to MUST, primarily due to the absence of class names. For future work, we will focus on enhancing the performance of our method and exploring new avenues for improvement. Leveraging the theoretical results to augment our method holds promise and will be a key research direction we pursue.

\section*{Acknowledgments}
This work is jointly supported by Major Project of the New Generation of Artificial Intelligence under Grant no.2018AAA0102900; in part by NSFC under Grant no. 92270122 and no.62206179; and in part by Guangdong Provincial Natural Science Foundation under Grant no. 2023A1515012584 and no.2022A1515010129; and in part by the Shenzhen Research Foundation for Basic Research, China, under Grant JCYJ20210324093000002; and in part by University stability Support program of Shenzhen under Grant no.20220811121315001.

%\bibliography{aaai24}
%\bibliographystyle{aaai24}

\nobibliography*

% Use \bibliography{yourbibfile} instead or the References section will not appear in your paper
\bibliography{aaai24}

\newpage

\appendix
\label{sec:appendix}
\onecolumn
\section{Appendix}
In this appendix, we provide detailed proof of the theoretical results.

\subsection{Proof of Theorem~\ref{thm:conv}}
\label{app:pr_thm_conv}

\begin{proof}

For simplicity, we denote $\mathcal{L}(g(\mathcal{S});\varphi)$ as $\mathcal{L}(\varphi)$. The update rule can be formulated as 
$$ \varphi^{(t+1)}=\varphi^{(t)}-\eta_{\varphi}(\nabla \mathcal{L}(\varphi^{(t)})+\zeta^{(t)}),$$
\noindent where $\zeta^{(t)}=\left.\nabla \mathcal{L}(\varphi^{(t)})\right|_{\mathcal{B}}-\nabla \mathcal{L}(\varphi^{(t)})$ and $\mathcal{B}$ consists of a mini-batch images sampled i.i.d from $\mathcal{X}$ and texts sampled i.i.d from the neighbors in $\mathcal{T}$ of these images. This indicates that $\mathbb{E}[\zeta^{(t)}]=0$ holds. Let $\|\zeta^{(t)}\|^2 \leq \sigma^2$. First, we have \\
$$
\begin{aligned}
& \mathcal{L}(\varphi^{(t+1)})-\mathcal{L}(\varphi^{(t)}) \\
\leq & \langle\nabla \mathcal{L}(\varphi^{(t)}), \varphi^{(t+1)}-\varphi^{(t)}\rangle+\frac{L}{2}\|\varphi^{(t+1)}-\varphi^{(t)}\|_{2}^{2}\\
=&\langle\nabla \mathcal{L}(\varphi^{(t)}),-\eta_{\varphi}[\nabla\mathcal{L}(\varphi^{(t)})+\zeta^{(t)}]\rangle+\frac{L \eta_{\varphi}^{2}}{2}\|\nabla \mathcal{L}(\varphi^{(t)})+\zeta^{(t)}\|_{2}^{2} \\
=&-(\eta_{\varphi}-\frac{L\eta_{\varphi}^{2}}{2})\|\nabla \mathcal{L}(\varphi^{(t)})\|_{2}^{2}+\frac{L \eta_{\varphi}^{2}}{2}\|\zeta^{(t)}\|_{2}^{2}-(\eta_{\varphi}-L\eta _{\varphi}^{2})\langle\nabla \mathcal{L}(\varphi^{(t)}), \zeta^{(t)}\rangle .
\end{aligned}
$$

\noindent Therefore, \\
$$
\begin{aligned}
&(\eta_{\varphi}-\frac{L\eta_{\varphi}^{2}}{2})\|\nabla \mathcal{L}(\varphi^{(t)})\|_{2}^{2} \\
\leq& \mathcal{L}(\varphi^{(t)})-\mathcal{L}(\varphi^{(t+1)}) + \frac{L \eta_{\varphi}^{2}}{2}\|\zeta^{(t)}\|_{2}^{2}-(\eta_{\varphi}-L\eta _{\varphi}^{2})\langle\nabla \mathcal{L}(\varphi^{(t)}), \zeta^{(t)}\rangle .
\end{aligned}
$$

\noindent Taking summation on both sides, we have \\
$$
\begin{aligned}
& \sum_{t=1}^{T}(\eta_{\varphi}-\frac{L\eta_{\varphi}^{2}}{2})\|\nabla \mathcal{L}(\varphi^{(t)})\|_{2}^{2}  \\
\leq & \mathcal{L}(\varphi^{(1)})-\mathcal{L}(\varphi^{(T+1)}) + \frac{L \eta_{\varphi}^{2}}{2}\sum_{t=1}^{T}\|\zeta^{(t)}\|_{2}^{2}-\sum_{t=1}^{T}(\eta_{\varphi}-L\eta _{\varphi}^{2})\langle\nabla \mathcal{L}(\varphi^{(t)}), \zeta^{(t)}\rangle .
\end{aligned}
$$

\noindent Further, taking the expectation with respect to $\zeta$ on both sides, one can find that \\

$$
\begin{aligned}
\sum_{t=1}^{T}(\eta_{\varphi}-\frac{L \eta_{\varphi}^{2}}{2}) \mathbb{E}_{\zeta}\|\nabla \mathcal{L}(\varphi^{(t)})\|^{2} & \leq \mathcal{L}(\varphi^{(1)})+\frac{L \eta_{\varphi}^{2}}{2} \sum_{t=1}^{T}\|\zeta^{(t)}\|^{2} \\
& \leq \mathcal{L}(\varphi^{(1)})+\frac{L \eta_{\varphi}^{2} T \sigma^{2}}{2}
\end{aligned}
$$

\noindent Thus we have \\
$$
\begin{aligned}
\min _{t} \mathbb{E}_{\zeta}\|\nabla \mathcal{L}(\varphi^{(t)})\|^{2} 
\leq & \frac{\sum_{t=1}^{T}(\eta_{\varphi}-\frac{L \eta_{\varphi}^{2}}{2}) \mathbb{E}_{\zeta}\|\nabla \mathcal{L}(\varphi^{(t)})\|_{2}^{2}}{\sum_{t=1}^{T}(\eta_{\varphi}-\frac{L \eta_{\varphi}^{2}}{2})}\\
\leq & \frac{2 \mathcal{L}(\varphi^{(1)})+L \eta_{\varphi}^{2} T \sigma^{2}}{\sum_{t=1}^{T}(2 \eta_{\varphi}-L \eta_{\varphi}^{2})}\\
\leq & \frac{2 \mathcal{L}(\varphi^{(1)})+L \eta_{\varphi}^{2} T \sigma^{2}}{T \eta_{\varphi}}\\
 =   & \frac{2 \mathcal{L}(\varphi^{(1)})}{T \eta_{\varphi}}+L \sigma^{2} \eta_{\varphi}\\
\leq & \frac{2 \mathcal{L}(\varphi^{(1)})}{T} \max \{L, \frac{\sqrt{T}}{C}\}+L \sigma^{2} \min \{\frac{1}{L}, \frac{C}{\sqrt{T}}\} \\
= & \mathcal{O}(\frac{1}{\sqrt{T}}) .
\end{aligned}
$$

\end{proof}

\subsection{Proof of Theorem~\ref{thm:risk}}
\label{app:pr_thm_loss}
To prove Theorem~\ref{thm:risk}, we first introduce the following lemmas.
% \begin{lemma} \label{lemma1}
% Assume that $\| \mathbf{q}_i\|_{\infty} \leq E$ holds for all $x_i \in \mathcal{X}$, where $\mathbf{q}_i = f(g(x);\phi)$ is cluster assignment probability from the MLP classification model. We define the empirical risk and its expectation as
% To simplify the proof, we consider the case where the neighbors contain all training samples. Then the empirical and expected risk are defined as
\begin{lma} \label{lma_cl}
Given a set of $n$ samples $\mathcal{S}$, we define the following empirical risk
$$
\begin{aligned}
\widehat{\mathcal{L}}_{n}(g) &=-\frac{1}{n} \sum_{i=1}^{n}\log \mathbf{q}_{i}^{T} \mathbf{q}_{rn(i)},
\end{aligned}
$$
and corresponding expected risk
$$
\mathcal{L}(g) = -\mathbb{E} \left[  \log  \mathbf{q}^{T} \mathbf{q}^{r} \right].
$$
where $rn(i)=rn\left(\mathcal{N}^{I}_{k_{I}}(x_i)\right)$ and $rn\left(\mathcal{N}^{I}_{k_{I}}(x_i)\right)$ randomly select a sample from $\mathcal{N}^{I}_{k_{I}}(x_i)$. $\mathbf{q_i}=g((x_i,\emptyset);\varphi)$ and $\mathbf{q}^{r}$ is the prediction of a randomly selected sample from the neighborhood of $\mathbf{q}$. 
With probability at least $1-\delta$, the following inequality holds\\
$$ 
\mathcal{L}(g) \le \widehat{\mathcal{L}}_{n}(g) +  \frac{2\mu_{I}^{-1}}{\sqrt{n}}+2\log \mu_{I}^{-1} \sqrt{\frac{\log \delta^{-1}}{2 n}}.
$$
\end{lma} 
\begin{proof}
% $$
% \begin{aligned}
% \widehat{\mathcal{L}}_{n}(g) &=-\frac{1}{n} \sum_{i=1}^{n} \log \mathbf{q}_{i} \mathbf{q}_{i'}
% \end{aligned}
% $$ 
Let $\mathcal{S}'=(\mathcal{S}-{s_j})\cup{s_{j'}} $. The empirical risks on $\mathcal{S}'$ is denoted as $\widehat{\mathcal{L}}'_{n}(g)$. We have
\begin{equation*}
    \begin{aligned}
\left|\sup _{g \in \mathcal{G}}\right| \mathcal{L}(g)-\widehat{\mathcal{L}}_{n}(g)\left|-\sup _{g \in \mathcal{G}}\right| \mathcal{L}(g)-\widehat{\mathcal{L}}'_{n}(g)\Bigg| \Bigg| 
\leq & \sup _{g \in \mathcal{G}}\left|\widehat{\mathcal{L}}_{n}(g)-\widehat{\mathcal{L}}'_{n}(g )\right|, \\
% \leq & \sup _{g \in \mathcal{G}}\frac{2}{n(n-1)} \sum_{i=1}^n \left| -(\log f^{\top}_{\mathcal{Q}}(\mathbf{x}_i)\log f^{\top}_{\mathcal{Q}}(\mathbf{x}_i) \right| 
\end{aligned}
\end{equation*}
According to Assumption~\ref{as:num_neighbor}, let $\{x_{j1}, x_{j2}, \dots, x_{jk_{S}'}\}$ be the set with $x_j$ as neighbor. Then we have  
\begin{equation*}
    \begin{aligned}
\sup _{g \in \mathcal{G}}\left|\widehat{\mathcal{L}}_{S}(g)-\widehat{\mathcal{L}}_{\mathcal{S}'}(g)\right| \leq & \sup _{g \in \mathcal{G}}\frac{1}{n} (\left| -(\log \mathbf{q}_{j}^{T} \mathbf{q}_{rn(j)}- \log \mathbf{q}_{j'}^{T}\mathbf{q}_{rn(j')}) \right| + |-\sum_{l=r1}^{rk_{S}'}(\log \mathbf{q}_{l}^T \mathbf{q}_{j}-\log \bar{\mathbf{q}}_{l}^T \bar{\mathbf{q}}_{rn(l)})|)\\
% \leq & \sup_{f \in \mathcal{F}} \frac{1}{n} \left| (\mathbf{q}_r \mathbf{q}_{r'}+\bar{\mathbf{q}}_r\bar{\mathbf{q}}_{r'}-2)\right| \\
\leq & \frac{(2+2k_{I}')\log \mu_{I}^{-1}}{n},
\end{aligned}
\end{equation*}
where the last inequality is according to Assumption~\ref{as:knn_I}. \\
Combine the above two equations and we get:
\begin{equation*}
    \begin{aligned}
&\left|\sup _{f \in \mathcal{F}}\right| \mathcal{L}(f)-\widehat{\mathcal{L}}_S(f)\left|-\sup _{f \in \mathcal{F}}\right| \mathcal{L}(f)-\widehat{\mathcal{L}}_{S'}(f)\Bigg| \Bigg| \leq & \frac{(2+2k_{I}')\log \mu_{I}^{-1}}{n} .
\end{aligned}
\end{equation*}

Let $\{\sigma_1,\sigma_2,\dots,\sigma_n\}$ be i.i.d. independent random variables taking values in $\{-1,1\}$ and $\bar{S} := \{\bar{s}_1,\dots,\bar{s}_n\}$ be the independent copy of $S$, Then we have
\begin{equation*}\small
\begin{aligned}
     \mathbb{E}_{\mathcal{S}} \left[ \mathop{\rm sup}_{g \in \mathcal{G}}|\mathcal{L}(g)-\widehat{\mathcal{L}}_{n}(g)| \right]
    % = & \mathbb{E}_{n} \left[\mathop{\rm sup}_{g\in\mathcal{G}} \frac{1}{n}\Bigg| \sum_{i=1}^n \log q_i q_{i'} - \mathbb{E}_{q,q'}[\log q q'] \Bigg|\right] \\
    \leq & \mathbb{E}_{\mathcal{S},\bar{\mathcal{S}},\sigma} \left[\sup _{g \in \mathcal{G}} \frac{1}{n}\Bigg| \sum_{i=1}^n \sigma_{i}(\log \mathbf{q}_i^{T} \mathbf{q}_{rn(i)} - \log \bar{\mathbf{q}}_i^{T} \bar{\mathbf{q}}_{rn(i)}) \Bigg|\right] \\
    \leq & 2 \mathbb{E}_{\mathcal{S},\sigma} \left[\sup _{g \in \mathcal{G}} \frac{1}{n}\Bigg| \sum_{i=1}^n \sigma_{i}\log \mathbf{q}_i^{T} \mathbf{q}_{rn(i)} \Bigg|\right] \\
    \leq & 2 \mathbb{E}_{\mathcal{S},\sigma} \left[\sup _{g \in \mathcal{G}} \frac{1}{n}\Bigg| \sum_{i=1}^n \sigma_{i}(\frac{1}{\mathbf{q}_i^{T} \mathbf{q}_{rn(i)}}-1)  \Bigg|\right] \\    
    \leq & 2 \mathbb{E}_{\mathcal{S},\sigma} \left[\sup _{g \in \mathcal{G}} \frac{1}{n}\Bigg| \sum_{i=1}^n \sigma_{i}\frac{1}{\mathbf{q}_i^{T} \mathbf{q}_{rn(i)} } \Bigg|\right] \\     
    \leq & 2 \mathbb{E}_{\mathcal{S}} \left[\sup _{g\in \mathcal{G}} \frac{1}{n}\Bigg( \sum_{i=1}^n\left(\frac{1}{\mathbf{q}_i^{T} \mathbf{q}_{rn(i)}}\right)^{2} \Bigg)^{\frac{1}{2}}\right] \\
    \leq&\frac{2\mu_{I}^{-1}}{\sqrt{n}}.   
\end{aligned} 
\end{equation*}
where the second to last inequality is obtained by Khintchine-Kahane inequality~\cite{latala1994best} and the last inequality is obtained by Assumption~\ref{as:knn_I}.

Thus according to the McDiarmid inequality~\cite{mohri2018foundations}, with probability at least $1-\delta$ for any $g\in \mathcal{G}$, we have
$$ 
\mathcal{L}(g) \le \widehat{\mathcal{L}}_{n}(g) +  \frac{2\mu_{I}^{-1}}{\sqrt{n}} + (2+2k_{I}')\log \mu^{-1}_{I}\sqrt{\frac{\log \delta^{-1}}{2 n}}.
$$
\end{proof}

\begin{lma}\label{lm_balance}
Given a set of $n$ samples $\mathcal{S}$, we define the following empirical risk
\begin{equation*}\small
    \widehat{\mathcal{L}}_{n}(g) = -\sum_{l=1}^c \Bigg( \frac{1}{n}\sum_{i=1}^n q_{il} \Bigg) \log \Bigg( \frac{1}{n}\sum_{i=1}^n q_{il} \Bigg),
\end{equation*}
and corresponding expected risk
\begin{equation*}\small
    \mathcal{L}(g) =  -\sum_{l=1}^c\mathbb{E}(q^{l})  \log \mathbb{E}(q^{l}).
\end{equation*}
where $\mathbf{q_i}=g((x_i,\emptyset);\varphi)$. 
With probability at least $1-\delta$, the following inequality holds
\begin{equation*}\small
\mathcal{L}(g) \leq \widehat{\mathcal{L}}_{n}(g) + \frac{2C}{\sqrt{n}} + C\sqrt{\frac{\log \delta^{-1}}{2n}}.
\end{equation*}
where $C=|\log\xi+1|$ is a bounded constant and $\xi$ is a constant according to the Lagrange Mean Theorem of the function $b(x) = x\log x$.
\end{lma} 

\begin{proof}
Let $\mathcal{S}'=(\mathcal{S}-{s_j})\cup{s_{j'}} $. The empirical risk on $\mathcal{S}'$ is denoted as $\widehat{\mathcal{L}}'_{n}$. Define $b(x)=x\log x$, according to the Lagrange Mean Theorem, there exists constant $\xi$ such that $|b(x)-b(y)| \leq |\log \xi + 1||x-y|$. We have
\begin{equation*}
    \begin{aligned}
&\left|\sup _{g \in \mathcal{G}}\right| \mathcal{L}(g)-\widehat{\mathcal{L}}_{n}(g)\left|-\sup _{g\in \mathcal{G}}\right| \mathcal{L}(g)-\widehat{\mathcal{L}}'_{n}(g)\Bigg| \Bigg| \\
\leq & \sup _{g\in \mathcal{G}}\left|\widehat{\mathcal{L}}_{n}(g)-\widehat{\mathcal{L}}'_{n}(g)\right| \\
\leq & \sup _{g\in \mathcal{G}}\frac{\left|\log\xi + 1\right|}{n}  \sum_{l}^c  \left|q_{jl}-q'_{jl}\right|\\
\leq & \frac{\left| \log\xi + 1 \right|}{n} \\
= & \frac{C}{n}. 
\end{aligned}
\end{equation*}
where $C = \left| \log\xi + 1 \right|$. \\
Next, we analyze the upper bound of the term $\mathbb{E}_{S} \left[\mathop{\rm sup}_{g\in \mathcal{G}}|\mathcal{L}(g)-\widehat{\mathcal{L}}_{n}(g)|\right ]$. Let $\sigma_1,\sigma_2,\dots,\sigma_n$ be i.i.d. independent random variables taking values in $\{-1,1\}$ and $\bar{S} := \{\bar{s}_1,\dots,\bar{s}_n\}$ be the independent copy of $S := \{s_1, \dots, s_n\}$. Then we have
\begin{equation*}\small
\begin{aligned}
    & \mathbb{E}_{S} \left[ \mathop{\rm sup}_{g\in\mathcal{G}}|\mathcal{L}(g)-\widehat{\mathcal{L}}_S(g)| \right]\\
    \leq & \mathbb{E}_{S,\bar{S},\sigma} \left[\sup _{g\in\mathcal{G}} \frac{1}{n}\sum_{i=1}^n\sigma_i \left(\sum_{l=1}^{c} |\log\xi + 1| | q_{il}-\bar{q}_{il} |\right)  \right] \\  
    \leq & 2|\log\xi + 1|\mathbb{E}_{S,\sigma} \left[\sup _{g\in\mathcal{G}} \frac{1}{n}\sum_{i=1}^n\sigma_i \sum_{l=1}^{c} q_{il}  \right] \\ 
    \leq & 2|\log\xi + 1|\mathbb{E}_{S} \left[\sup _{g\in\mathcal{G}} \frac{1}{n}\left(\sum_{i=1}^n \left[\sum_{l=1}^{c} q_{il}\right]^2 \right)^{\frac{1}{2}} \right] \\
    \leq & \frac{2|\log\xi + 1|}{\sqrt{n}} \\
    = & \frac{2C}{\sqrt{n}}.  \\
\end{aligned}
\end{equation*}
where the second to last inequality is obtained by Khintchine-Kahane inequality~\cite{latala1994best}.

Thus according to the McDiarmid inequality~\cite{mohri2018foundations}, with probability at least $1-\delta$ for any $g\in \mathcal{G}$, we have
$$ 
\mathcal{L}(g) \leq \widehat{\mathcal{L}}_{n}(g) + \frac{2C}{\sqrt{n}} + C\sqrt{\frac{\log \delta^{-1}}{2n}}.
$$
\end{proof}

\begin{lma} \label{lm_si}
Given a set of $n$ samples $\mathcal{S}$, we define the following empirical risk
$$
\begin{aligned}
  \widehat{\mathcal{L}}_{n}(g) &=-\frac{1}{n}  \sum_{i=1}^{n}\log\frac{exp(\mathbf{q}_{i}^T\mathbf{p}_{rn(i)}/\tau_{ia})}{\sum_{l=1,l\neq rn(i)}^{m}  exp(\mathbf{q}_{i}^{T}\mathbf{p}_{l}/\tau_{ia})}\\
  &=-\frac{1}{n}  \sum_{i=1}^{n}\left[\frac{\mathbf{q}_{i}^T\mathbf{p}_{rn(i)}}{\tau_{ia}}-\log\left(\sum_{l=1,l\neq rn(i)}^{m}exp(\mathbf{q}_{i}^{T}\mathbf{p}_{l}/\tau_{ia})\right)\right],
\end{aligned}
$$
and corresponding expected risk
$$
\mathcal{L}(g) =-\mathbb{E} \left[\frac{\mathbf{q}^T\mathbf{p}'}{\tau_{ia}}-\log(exp((\mathbb{E}(\mathbf{q}^{T}\mathbf{p})-\mathbf{q}^T\mathbf{p}')/\tau_{ia}))\right].
$$
where $rn(i)=rn\left(\mathcal{N}^{S}_{k_{S}}(x_i)\right)$ randomly select a sample from $\mathcal{N}^{S}_{k_{S}}(x_i)$ and $\mathbf{p}'$ is the prediction of a randomly selected sample from the neighborhood of $\mathbf{q}$. With probability at least $1-\delta$, the following inequality holds\\
$$ 
\mathcal{L}(g) \le  \widehat{\mathcal{L}}_{n}(g) +\frac{2m}{\sqrt{n}\tau_{ia}}+ \frac{2(1-\mu_{C})}{\tau_{ia}}\sqrt{\frac{\log \delta^{-1}}{2n}}.
$$
\end{lma} 
\begin{proof}
    Let $\mathcal{S}'=(\mathcal{S}-{s_j})\cup{s_{j'}} $. The empirical risk on $\mathcal{S}'$ is denoted as $\widehat{\mathcal{L}}'_{n}$. We have
\begin{equation*}
    \begin{aligned}
&\left|\sup _{g \in \mathcal{G}}\right| \mathcal{L}(g)-\widehat{\mathcal{L}}_{n}(g)\left|-\sup _{g\in \mathcal{G}}\right| \mathcal{L}(g)-\widehat{\mathcal{L}}'_{n}(g)\Bigg| \Bigg| \\
\leq & \sup _{g\in \mathcal{G}}\left|\widehat{\mathcal{L}}_{n}(g)-\widehat{\mathcal{L}}'_{n}(g)\right| \\
\leq & \sup _{g\in \mathcal{G}}\frac{1}{n}\left|
\frac{\mathbf{q}_{j}^{T}\mathbf{p}_{rn(j)}-\mathbf{q}_{j'}^{T}\mathbf{p}_{rn(j')}}{\tau_{ia}}
-\log\frac{\sum_{l=1,l\neq rn(j)}^{m}exp(\mathbf{q}_{j}^{T}\mathbf{p}_{l}/\tau_{ia})}{\sum_{l=1,l\neq rn(j')}^{m}exp(\mathbf{q}_{j'}^{T}\mathbf{p}_{l}/\tau_{ia})}\right|  \\
\leq & \frac{2(1-\mu_{C})}{n\tau_{ia}}
\end{aligned}
\end{equation*}
which is based on Assumption~\ref{as:cross_knn}.

Next, we analyze the upper bound of the term $\mathbb{E}_{S} \left[\mathop{\rm sup}_{g\in \mathcal{G}}|\mathcal{L}(g)-\widehat{\mathcal{L}}_{n}(g)|\right ]$. Let $\sigma_1,\sigma_2,\dots,\sigma_n$ be i.i.d. independent random variables taking values in $\{-1,1\}$ and $\bar{S} := \{\bar{s}_1,\dots,\bar{s}_n\}$ be the independent copy of $S := \{s_1, \dots, s_n\}$. Then we have
\begin{equation*}\small
\begin{aligned}
    \mathbb{E}_{\mathcal{S}} \left[ \mathop{\rm sup}_{g\in\mathcal{G}}|\mathcal{L}(g)-\widehat{\mathcal{L}}_{\mathcal{S}}(g)| \right]
    \leq & \mathbb{E}_{{\mathcal{S}},\bar{{\mathcal{S}}},\sigma} \left[\sup_{g\in\mathcal{G}} \frac{1}{n}\sum_{i=1}^n\sigma_i \left|
\frac{\mathbf{q}_{i}^{T}\mathbf{p}_{rn(i)}-(\bar{\mathbf{q}}_{i}^{I})^{T}\bar{\mathbf{p}}_{rn(i)}}{\tau_{ia}}
-\log\frac{\sum_{l=1,l\neq rn(i)}^{m}exp(\mathbf{q}_{i}^{T}\mathbf{p}_{l}/\tau_{ia})}{\sum_{l=1,l\neq rn(i)}^{m}exp((\bar{\mathbf{q}}_{i}^{I})^{T}\bar{\mathbf{p}}_{l}/\tau_{ia})}\right| \right] \\ 
    \leq & 2\mathbb{E}_{{\mathcal{S}},\sigma} \left[\sup_{g\in\mathcal{G}} \frac{1}{n}\sum_{i=1}^n\sigma_i \left|
\frac{\mathbf{q}_{i}^{T}\mathbf{p}_{rn(i)}}{\tau_{ia}}
-\log\left(\sum_{l=1,l\neq rn(i)}^{m}exp(\mathbf{q}_{i}^{T}\mathbf{p}_{l}/\tau_{ia})\right)\right| \right] \\
\leq & 2\mathbb{E}_{{\mathcal{S}},\sigma} \left[\sup_{g\in\mathcal{G}} \frac{1}{n}\sum_{i=1}^n\sigma_i \left(\left|
\frac{\mathbf{q}_{i}^{T}\mathbf{p}_{rn(i)}}{\tau_{ia}}\right|+\left|
\sum_{l=1,l\neq rn(i)}^{m}\frac{\mathbf{q}_{i}^{T}\mathbf{p}_{l}}{\tau_{ia}}\right|\right) \right] \\
\leq & 2\mathbb{E}_{\mathcal{S}} \left[\sup_{g\in\mathcal{G}} \frac{1}{n\tau_{ia}} \left(
\Bigg(\sum_{i=1}^n\left(\mathbf{q}_i^{T} \mathbf{p}_{rn(i)}\right)^{2} \Bigg)^{\frac{1}{2}}+\Bigg(\sum_{i=1}^n\left(\sum_{l=1,l\neq rn(i)}^{m}\mathbf{q}_{i}^{T}\mathbf{p}_{l}\right)^{2} \Bigg)^{\frac{1}{2}} \right)\right] \\
\leq &\frac{2m}{\sqrt{n}\tau_{ia}}
\end{aligned}
\end{equation*}
where the second to last inequality is obtained by Khintchine-Kahane inequality~\cite{latala1994best} and the last inequality is obtained by Assumption~\ref{as:cross_knn}.

Thus according to the McDiarmid inequality~\cite{mohri2018foundations}, with probability at least $1-\delta$ for any $g\in \mathcal{G}$, we have
$$ 
\mathcal{L}(g) \leq \widehat{\mathcal{L}}_{n}(g) +\frac{2m}{\sqrt{n}\tau_{ia}}+ \frac{2(1-\mu_{C})}{\tau_{ia}}\sqrt{\frac{\log \delta^{-1}}{2n}}.
$$
\end{proof}

\begin{lma} \label{lm_sp}
Given a set of $n$ samples $\mathcal{S}$, we define the following empirical risk
$$
\begin{aligned}
  \widehat{\mathcal{L}}_{n}(g) &=-\frac{1}{c}\sum_{j=1}^{c}\log\frac{exp((\rho_{j}^{I})^T \rho_{j}^{S}/\tau_{pa})}{\sum_{l=1,l\neq j}^{c}exp((\rho_{j}^{I})^T \rho_{j}^{S}/\tau_{pa})}=-\frac{1}{c}\sum_{j=1}^{c}\left[\frac{(\rho_{j}^{I})^T \rho_{j}^{S}}{\tau_{pa}}-\log\left(\sum_{l=1,l\neq j}^{c}exp((\rho_{j}^{I})^T \rho_{j}^{S}/\tau_{pa})\right)
\right]
\end{aligned}
$$
where $\rho^{I}=f_{I}(\mathbf{h}^{I},\phi)$ and $\rho^{S}=f_{S}(\mathbf{h}^{S},\theta)$. $h_{l}^{I}=\frac{\sum_{i=1}^{n}q_{il}\mathbf{u}_{i}}{\|\mathbf{q}_{l}\|}$ and $\mathbf{h}^{S}=close(\mathbf{h}^{I})$ computed according to Section~\ref{sec:align} is close to $\mathbf{h}^{I}$.

The corresponding expected risk is defined as
$$
\begin{aligned}
  \mathcal{L}(g) &=-\frac{1}{c}\sum_{j=1}^{c}\left[\frac{\mathbb{E}(\rho_{j}^{I})^T \mathbb{E}(\rho_{j}^{S},\theta)}{\tau_{pa}}-\log\left(\sum_{l=1,l\neq j}^{c}exp(\mathbb{E}(\rho_{j}^{I})^T \mathbb{E}(\rho_{l}^{S})/\tau_{pa})\right)\right]
\end{aligned}
$$
With probability at least $1-\delta$, the following inequality holds\\
$$ 
\mathcal{L}(g) \leq \widehat{\mathcal{L}}_{n}(g) +\frac{2dL_{IS}M_{u}}{n\tau_{pa}}+ \frac{dcL_{I}M_{u}^2}{\tau_{pa}}\sqrt{\frac{\log \delta^{-1}}{2n}}.
$$
\end{lma} 
\begin{proof}
    Let $\mathcal{S}'=(\mathcal{S}-{s_r})\cup{s_{r'}} $. The empirical risk on $\mathcal{S}'$ is denoted as $\widehat{\mathcal{L}}'_{n}$. We have
\begin{equation*}
    \begin{aligned}
&\left|\sup _{g \in \mathcal{G}}\right| \mathcal{L}(g)-\widehat{\mathcal{L}}_{n}(g)\left|-\sup _{g\in \mathcal{G}}\right| \mathcal{L}(g)-\widehat{\mathcal{L}}'_{n}(g)\Bigg| \Bigg| \\
\leq & \sup _{g\in \mathcal{G}}\left|\widehat{\mathcal{L}}_{n}(g)-\widehat{\mathcal{L}}'_{n}(g)\right| \\
\leq & \sup _{g\in \mathcal{G}}\frac{1}{c}\sum_{j=1}^{c}\left|
\frac{(\rho_{j}^{I})^{T}\rho_{j}^{S}-((\rho')_{j}^{I})^{T}(\rho')_{j}^{S}}{\tau_{pa}}
-\log\frac{\sum_{l=1,l\neq j}^{c}exp((\rho_{j}^{I})^{T}\rho_{l}^{S}/\tau_{pa})}{\sum_{l=1,l\neq j}^{c}exp(((\rho')_{j}^{I})^{T}(\rho')_{l}^{S}/\tau_{pa})}\right|  \\
\leq & \sup _{g\in \mathcal{G}}\frac{2}{\tau_{pa}}\max_{j\in[c]}\left|
(\rho_{j}^{I})^{T}\rho_{j}^{S}-((\rho')_{j}^{I})^{T}(\rho')_{j}^{S}
\right|  \\
\leq & \sup _{g\in \mathcal{G}}\frac{2L_{IS}}{\tau_{pa}}\max_{j\in[c]}\left\|
\mathbf{h}_{j}^{S}-(\mathbf{h}'_{j})^{S}
\right\|  \\
\leq & \sup _{g\in \mathcal{G}}\frac{2L_{IS}}{\tau_{pa}}\max_{j\in[c]}\left\|\frac{1}{n}(q_{rl}\mathbf{u}_{r}-q_{r'l}\mathbf{u}_{r'})
\right\|  \\
\leq & \frac{2dL_{IS}M_{u}}{n\tau_{pa}}
\end{aligned}
\end{equation*}

Next, we analyze the upper bound of the term $\mathbb{E}_{S} \left[\mathop{\rm sup}_{g\in \mathcal{G}}|\mathcal{L}(g)-\widehat{\mathcal{L}}_{n}(g)|\right ]$. Let $\sigma_1,\sigma_2,\dots,\sigma_n$ be i.i.d. independent random variables taking values in $\{-1,1\}$ and $\bar{S} := \{\bar{s}_1,\dots,\bar{s}_n\}$ be the independent copy of $S := \{s_1, \dots, s_n\}$. Then we have
\begin{equation*}\small
\begin{aligned}
    \mathbb{E}_{\mathcal{S}} \left[ \mathop{\rm sup}_{g\in\mathcal{G}}|\mathcal{L}(g)-\widehat{\mathcal{L}}_{\mathcal{S}}(g)| \right]
    \leq & \mathbb{E}_{{\mathcal{S}},\bar{{\mathcal{S}}},\sigma} \left[\sup_{g\in\mathcal{G}} \frac{1}{c}\sum_{j=1}^c \left|
\frac{(\rho_{j}^{I})^{T}\rho_{j}^{S}-(\bar{\rho}_{j}^{I})^{T}\bar{\rho}_{j}^{S}}{\tau_{pa}}-
\log\frac{\sum_{l=1,l\neq rn(i)}^{m}exp((\rho_{j}^{I})^{T}\rho_{l}^{S}/\tau_{pa})}{\sum_{l=1,l\neq rn(i)}^{m}exp((\bar{\rho}_{j}^{I})^{T}\bar{\rho}_{l}^{S}/\tau_{ia})}\right| \right] \\ 
    \leq & 2\mathbb{E}_{{\mathcal{S}},\sigma} \left[\sup_{g\in\mathcal{G}} \frac{1}{c}\sum_{j=1}^{c}\left|
\frac{(\rho_{j}^{I})^{T}\rho_{j}^{S}}{\tau_{pa}}
-\log\left(\sum_{l=1,l\neq j}^{c}exp((\rho_{j}^{I})^{T}\rho_{l}^{S}/\tau_{pa})\right)\right| \right] \\
\leq & 2\mathbb{E}_{{\mathcal{S}},\sigma} \left[\sup_{g\in\mathcal{G}} \frac{1}{c\tau_{pa}}\sum_{j=1}^{c}\left(
(\rho_{j}^{I})^{T}\rho_{j}^{S}+
 \sum_{l=1,l\neq j}^{c}(\rho_{j}^{I})^{T}\rho_{l}^{S}\right) \right] \\
\leq & 2\mathbb{E}_{{\mathcal{S}},\sigma} \left[\sup_{g\in\mathcal{G}} \frac{c}{\tau_{pa}}\max_{j\in[c]}\left(
(\rho_{j}^{I})^{T}\rho_{j}^{S}\right) \right] \\
\leq & 2\mathbb{E}_{{\mathcal{S}},\sigma} \left[\sup_{g\in\mathcal{G}} \frac{cL_{I}}{\tau_{pa}}\max_{j\in[c]}\|f_{I}\left(\frac{\sum_{i=1}^{n}\sigma_{i}q_{ij }\mathbf{u}_{i}}{\|\mathbf{q}_{l}\|},\phi\right)_{j} \|^{2}\right] \\
\leq & 2\mathbb{E}_{{\mathcal{S}},\sigma} \left[\sup_{g\in\mathcal{G}} \frac{cL_{I}}{\tau_{pa}}\max_{j\in[c]}\|\left(\frac{\sum_{i=1}^{n}\sigma_{i}q_{ij}\mathbf{u}_{i}}{n}\right)_{j} \|^{2}\right] \\
\leq & 2\mathbb{E}_{{\mathcal{S}},\sigma} \left[\sup_{g\in\mathcal{G}} \frac{cL_{I}}{\tau_{pa}}\max_{j\in[c]}\frac{\sum_{i=1}^{n}\sum_{l=1}^{d}(\sigma_{i}q_{ij}u_{il})^{2}}{n^2} \right] \\
\leq &\frac{dcL_{I}M_{u}^2}{\tau_{pa}n}
\end{aligned}
\end{equation*}
where the last inequality is obtained by assumption in Theorem~\ref{thm:conv}.

Thus according to the McDiarmid inequality~\cite{mohri2018foundations}, with probability at least $1-\delta$ for any $g\in \mathcal{G}$, we have
$$ 
\mathcal{L}(g) \leq \widehat{\mathcal{L}}_{n}(g) +\frac{2dL_{IS}M_{u}}{n\tau_{pa}}+ \frac{dcL_{I}M_{u}^2}{\tau_{pa}}\sqrt{\frac{\log \delta^{-1}}{2n}}.
$$
\end{proof}

\begin{lma}\label{lma_sa}
Given a set of n samples $\mathcal{S}$, we define the following empirical risk
\begin{equation*}\small
    \widehat{\mathcal{L}}_{n}(g) = -\frac{1}{n}\sum_{i=1}^n\sum_{l=1}^c q_{il}\log q'_{il},
\end{equation*}
and
\begin{equation*}\small
    \mathcal{L}(g) = -\sum_{l=1}^c \mathbb{E}(\mathbf{q})^{l}\log \mathbb{E}(\mathbf{q}')^{l}) .
\end{equation*}
where $\mathbf{q}'$ contains one-hot pseudo-label. With probability at least $1-\delta$, the following inequality holds
$$ 
\mathcal{L}(f) \le \widehat{\mathcal{L}}_{n}(f) + \frac{2\log\mu_p^{-1}}{\sqrt{n}} + 2\log\mu_p^{-1}\sqrt{\frac{\log \delta^{-1}}{2 n}}.
$$
\end{lma} 
\emph{Proof.} Let $\mathcal{S}'=(\mathcal{S}-{s_j})\cup s_{j'}$. We have
\begin{equation*}
    \begin{aligned}
& \sup _{f \in \mathcal{F}}\left|\widehat{\mathcal{L}}_{n}(f)-\widehat{\mathcal{L}}_{n}'(f)\right| \\
\leq & \sup _{f \in \mathcal{F}}\frac{1}{n} \left| \sum_{l}^c (q_{jl}\log q'_{jl}- q_{j'l}\log q'_{j'l}) \right| \\
% \leq & \sup _{f \in \mathcal{F}}\frac{1}{n} \left| \sum_{l}^c (p_{rl}(\frac{1}{q_{rl}}-1) + \bar{p}_{rl}(\frac{1}{\bar{q}_{rl}}-1)) \right| \\
% \leq & \sup _{f \in \mathcal{F}}\frac{1}{n} \left| \sum_{l}^c (p_{rl}(q_{rl}-1) + \bar{p}_{rl}(\bar{q}_{rl}-1)) \right| \\
\leq & \frac{2c\log\mu_{p}^{-1}}{n}.
\end{aligned}
\end{equation*}
Next we analyze the upper bound of the expectation term, \emph{i.e.}, $\mathbb{E}_{S} \left[\mathop{\rm sup}_{f \in \mathcal{F}}|\mathcal{L}(f)-\widehat{\mathcal{L}}_{S}(f)|\right ]$. Let $\sigma_1,\sigma_2,\dots,\sigma_n$ be i.i.d. independent random variables taking values in $\{-1,1\}$ and $\bar{S} := \{\bar{x}_1,\dots,\bar{x}_n\}$ be the independent copy of $S := \{x_1, \dots, x_n\}$. Then we have
\begin{equation*}\small
\begin{aligned}
    & \mathbb{E}_{S} \left[ \mathop{\rm sup}_{f_{\mathcal{Q}} \in \mathcal{F}}|\mathcal{L}(f)-\widehat{\mathcal{L}}_{n}(f)| \right]\\
    = & \mathbb{E}_{S,\bar{S},\sigma} \left[\sup _{f_{\mathcal{Q}} \in \mathcal{F}} \frac{1}{n}\Bigg| \sum_{i=1}^n \sigma_{i}(\sum_{l=1}^{c}(p_{il}\log q_{il} - \bar{p}_{il}\log \bar{q}_{il}) ) \Bigg|\right] \\
    \leq & 2 \mathbb{E}_{S,\sigma} \left[\sup _{f_{\mathcal{Q}} \in \mathcal{F}} \frac{1}{n}\Bigg| \sum_{i=1}^n \sigma_{i} \sum_{l=1}^{c} p_{il}\log q_{il}  \Bigg|\right] \\
    % \leq & 2 \mathbb{E}_{S,\sigma} \left[\sup _{f_{\mathcal{Q}} \in \mathcal{F}} \frac{1}{n}\Bigg| \sum_{i=1}^n \sigma_{i}\sum_{l=1}^{c} p_{il} (\frac{1}{q_{il}}-1)  \Bigg|\right] \\ 
    \leq & 2 \mathbb{E}_{S,\sigma} \left[\sup _{f_{\mathcal{Q}} \in \mathcal{F}} \frac{1}{n}\left( \sum_{i=1}^n \left[\sum_{l=1}^{c}p_{il} \log q_{il}\right]^2 \right)^{\frac{1}{2}}\right] \\
    \leq & \frac{2c\log\mu_{p}^{-1}}{\sqrt{n}}.
\end{aligned}
\end{equation*}
Thus according to the McDiarmid inequality~\cite{mohri2018foundations}, with probability at least $1-\delta$ for any $f\in \mathcal{F}$, we have
$$ 
\mathcal{L}(f) \le \widehat{\mathcal{L}}_{n}(f) + \frac{2c\log\mu_p^{-1}}{\sqrt{n}} + 2c\log\mu_p^{-1}\sqrt{\frac{\log \delta^{-1}}{2 n}}.
$$

Now we give the proof of Theorem~\ref{thm:risk}. \\
\begin{proof}

According to Lemma \ref{lma_cl}, \ref{lm_balance}, \ref{lm_si}, \ref{lm_sp} and \ref{lma_sa}, with probability at least $1-\delta$ for any $g\in \mathcal{G}$, the empirical risk and the expected risk of the overall loss function in Eq. (\ref{loss_overall}) has the following relationship
\begin{equation*}
    \mathcal{L}(g) \leq \widehat{\mathcal{L}}_{n}(g) + \frac{\tilde{c}_1}{\sqrt{n}} + \tilde{c}_2\sqrt{\frac{1}{2n}\log \delta^{-1}}+\frac{2dL_{IS}M_{u}}{n\tau_{pa}}.
\end{equation*}
where $\tilde{c}_1=2\mu_{I}^{-1}+2\eta C+2\lambda_{a}m/\tau_{ia}+2\lambda_{a}\lambda_{ca}c\log\mu_p^{-1}$ and $\tilde{c}_2= (2+2k_{I}')\log\mu_{I}^{-1} + \eta C +\frac{2\lambda_{a}(1-\mu_{C})}{\tau_{ia}}+\lambda_{a}\lambda_{sp}\frac{dcL_{I}M_{u}^2}{\tau_{pa}}+2\lambda_{a}\lambda_{ca}c\log\mu_p^{-1}$. This finishes the proof.
\end{proof}

\subsection{Implementation Details}
We used ViT-32~\cite{dosovitskiy2020image} and Transformer~\cite{vaswani2017attention} in CLIP as the image and text encoders, respectively. The cluster heads $f_{I}$ and $f_{S}$ were implemented by two $d \times c$ fully connected layers, where $d=512$ is the embedding dimension and $c$ is the number of clusters, respectively. Before training, all datasets were preprocessed with the augmentation method used in CLIP~\cite{zhou2021learning}, i.e., random square crop from resized images. We used cosine similarity to measure the similarity between image and word embeddings and thus construct the top $k_I$ nearest image neighbors $\mathcal{N}_{k_I}^{I}(x_i)$ and the top $k_S$ nearest texts $\mathcal{N}_{k_S}^{S}(x_i)$ for each image $x_i$. Epoch numbers and batch sizes of all datasets were set to 100 and 128. The nearest neighbors used in image consistency learning were searched through Faiss Library~\cite{jeff2021billion}. All experiments were carried out on four NVIDIA GeForce RTX 2080 Tis.

\subsection{Benchmark Datasets}
The details of five benchmark datasets are shown in Figure \ref{tab:data}.
\begin{table}[htb]
  \begin{center}
    \normalsize
  \caption{Statistics of five benchmark datasets.}
  \label{tab:data}
  \vskip 0.15in
  \setlength{\tabcolsep}{0.5mm}{
  \begin{tabular}{c|cccc}\toprule
    Dataset & Image size & \#Classes & \#Training & \#Testing  \\
  \hline
  \textbf{STL10} &$96\times96$ & $10$ &  $5,000$ & $8,000$ \\
  \textbf{Cifar10} & $32\times32$ & $10$ & $50,000$ & $10,000$  \\
  \textbf{Cifar100-20} &$32\times32$ &  $20$ & $50,000$ & $10,000$\\
  \textbf{ImageNet-Dogs} & $224\times224$ & $15$ & $19,500$ & $750$  \\
  \textbf{Tiny-ImageNet} & $64\times64$ &  $200$ & $100,000$ & $10,000$   \\
  \bottomrule
  \end{tabular}}
  \end{center}
  \end{table}

\subsection{Semantic Space Construction}

\begin{figure}[!htb]
\vspace{-25px}
\centering
    \subfloat[\footnotesize Result on STL10.]{
            \centering 
            \label{fig:stl10_gamma_h}
            \includegraphics[width = .49\linewidth]{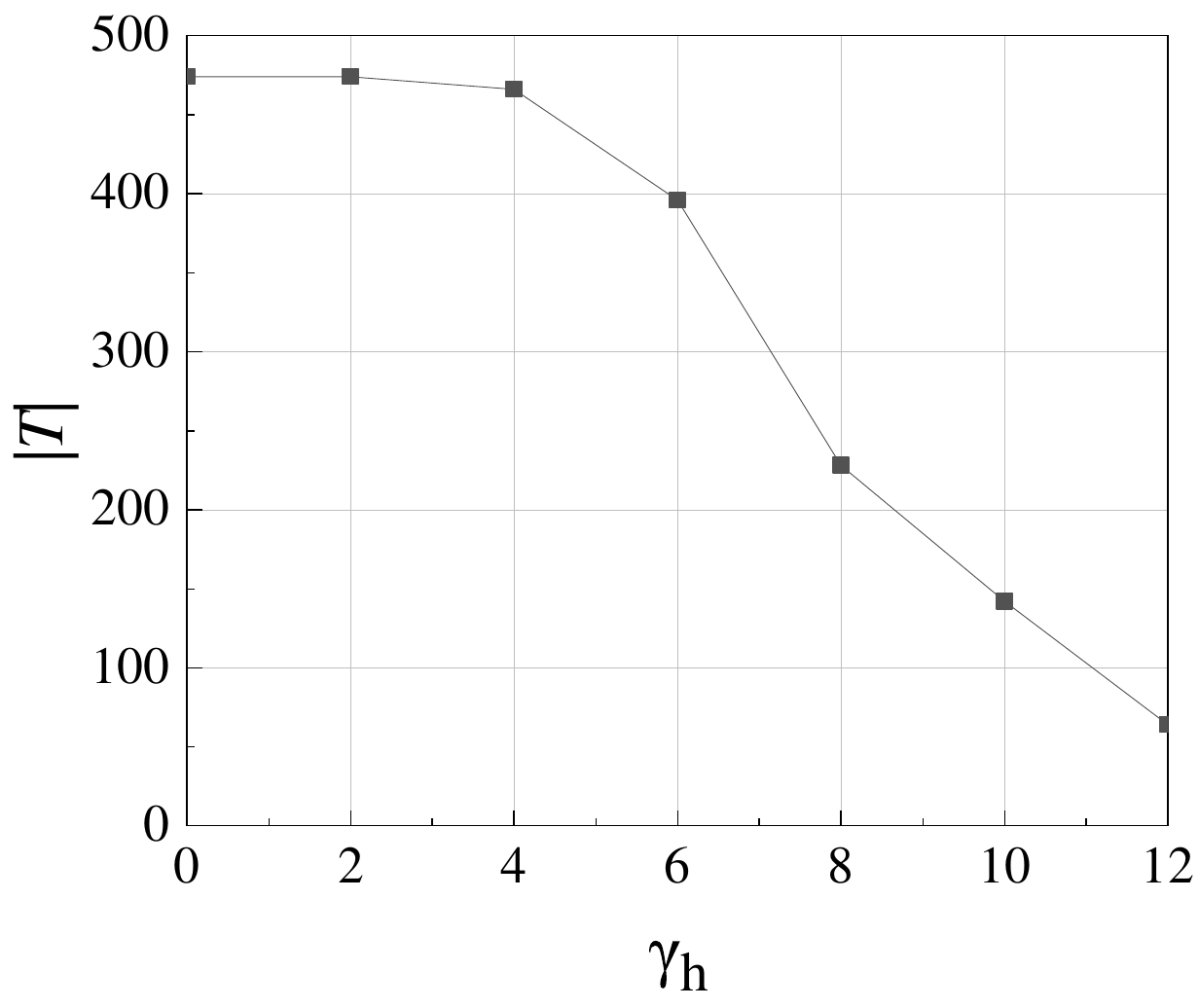}}
    \subfloat[\footnotesize Result on ImageNet-Dogs.]{
           \centering 
           \label{fig:imagenet_dog_gamma_h}  
           \includegraphics[width = .49\linewidth]{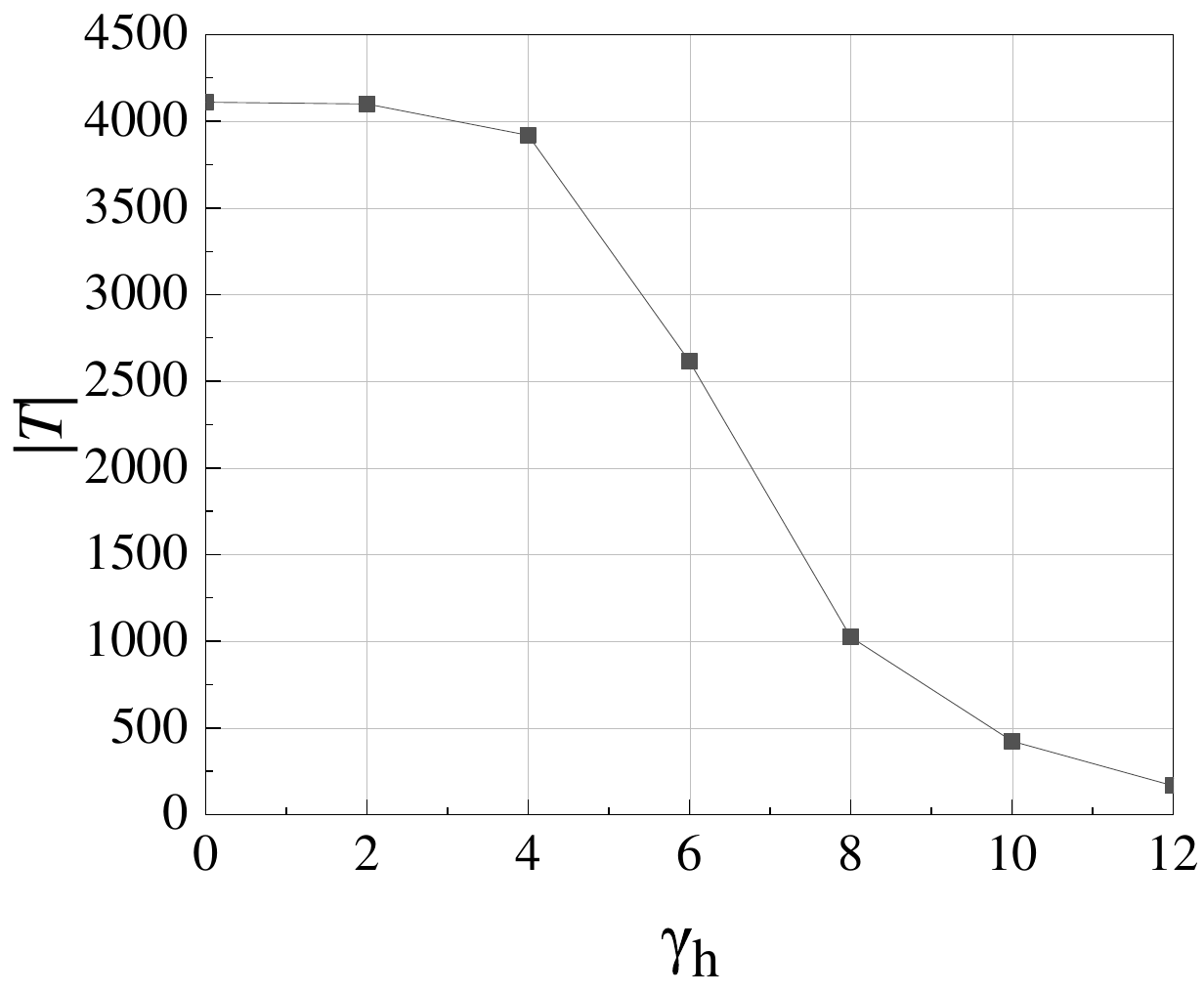}} 
\caption{$|\mathcal{T}|$ versus $\gamma_h$.}
\label{fig:vocab_gamma}
\end{figure}

\begin{figure}[!htb]
\centering
    
\subfloat[\footnotesize $\gamma_h$=0.]{
           \centering 
           \label{fig:stl10_gamma_h_0}  
           \includegraphics[width = .24\linewidth]{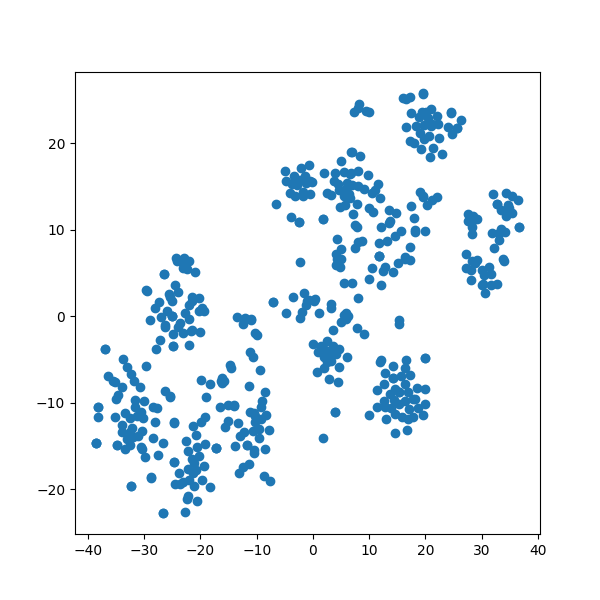}}\hspace{-3mm}
    \subfloat[\footnotesize $\gamma_h$=2.]{
            \centering 
            \label{fig:stl10_gamma_h_2}
            \includegraphics[width = .24\linewidth]{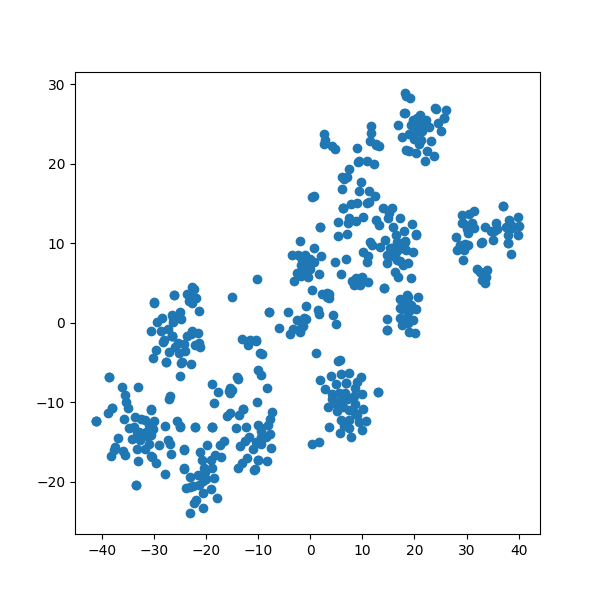}}\hspace{-3mm}
    \subfloat[\footnotesize $\gamma_h$=4.]{
           \centering 
           \label{fig:stl10_gamma_h_4}  
           \includegraphics[width = .24\linewidth]{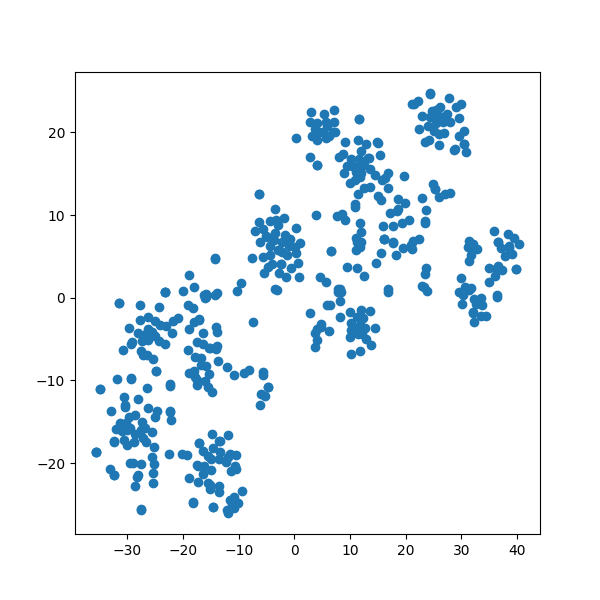}}\hspace{-3mm}
    \subfloat[\footnotesize $\gamma_h$=6.]{
           \centering 
           \label{fig:stl10_gamma_h_6}  
           \includegraphics[width = .24\linewidth]{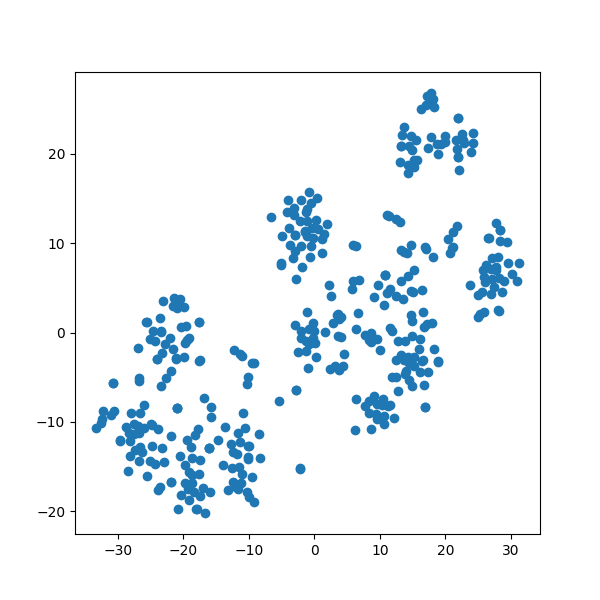}}\hspace{-3mm}
    \subfloat[\footnotesize $\gamma_h$=8.]{
            \centering 
            \label{fig:stl10_gamma_h_8}
            \includegraphics[width = .24\linewidth]{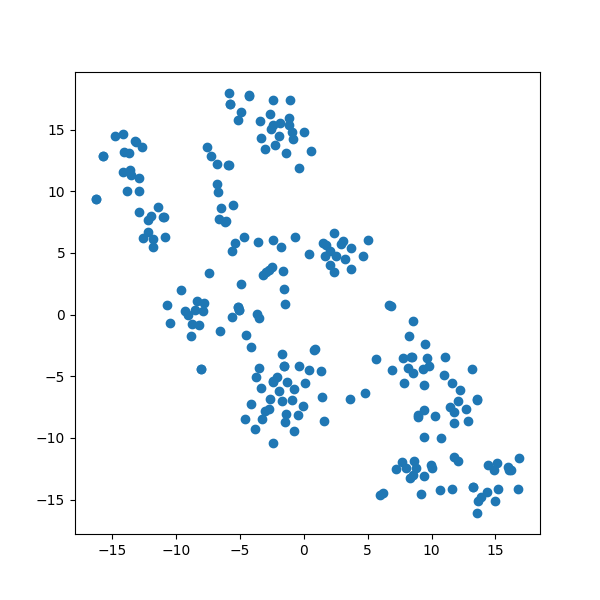}}\hspace{-3mm}
    \subfloat[\footnotesize $\gamma_h$=10.]{
           \centering 
           \label{fig:stl10_gamma_h_10}  
           \includegraphics[width = .24\linewidth]{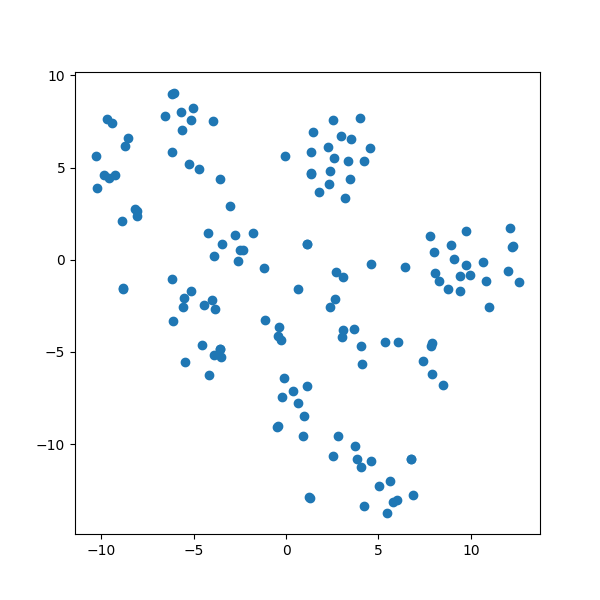}}\hspace{-3mm}
    \subfloat[\footnotesize  $\gamma_h$=12.]{
            \centering 
            \label{fig:stl10_gamma_h_12}
            \includegraphics[width = .24\linewidth]{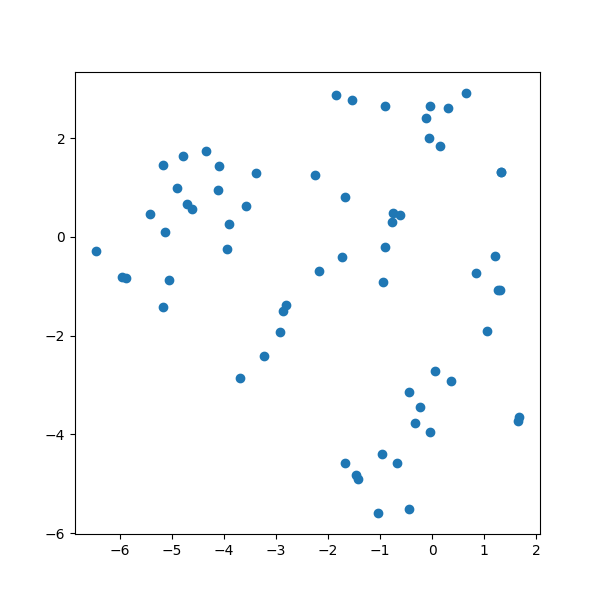}}\hspace{-3mm}            
            
\caption{ $t$-SNE visualization of the words in $\mathcal{T}$ on STL10, with the increasing of $\gamma_h$. 
}
\label{fig:t_visualization_stl10}
\end{figure}

\begin{figure}[!htb]
\vspace{-30px}
\centering
    
\subfloat[\footnotesize $\gamma_h$=0.]{
           \centering 
           \label{fig:imagenet_dog_gamma_h_0}  
           \includegraphics[width = .24\linewidth]{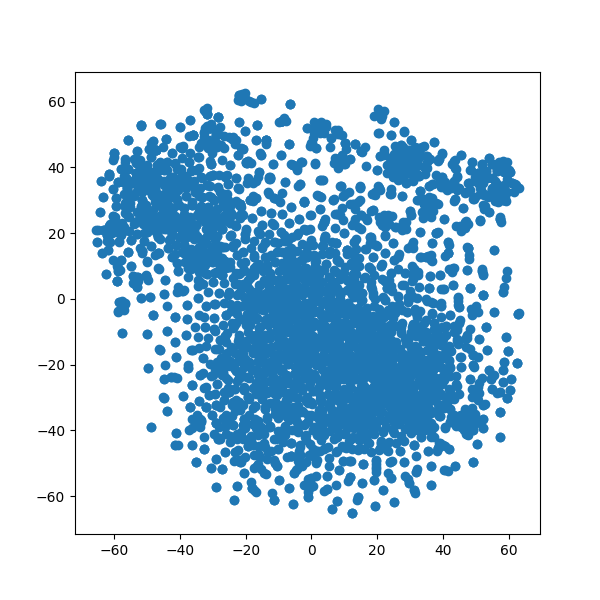}}\hspace{-3mm}
    \subfloat[\footnotesize $\gamma_h$=2.]{
            \centering 
            \label{fig:imagenet_dog_gamma_h_2}
            \includegraphics[width = .24\linewidth]{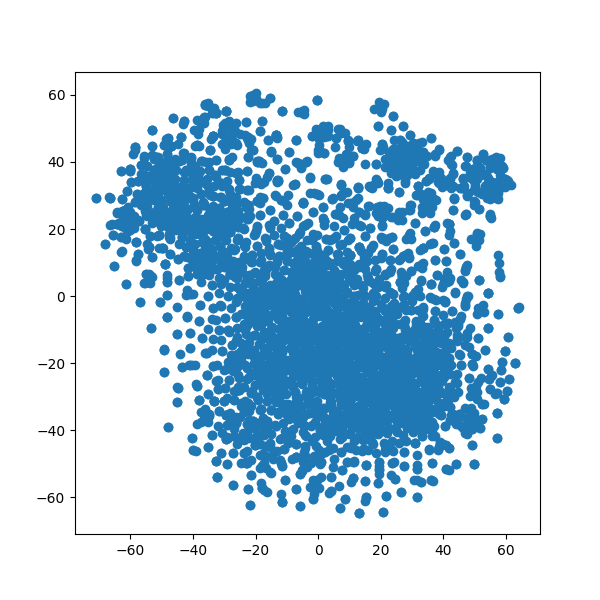}}\hspace{-3mm}
    \subfloat[\footnotesize $\gamma_h$=4.]{
           \centering 
           \label{fig:imagenet_dog_gamma_h_4}  
           \includegraphics[width = .24\linewidth]{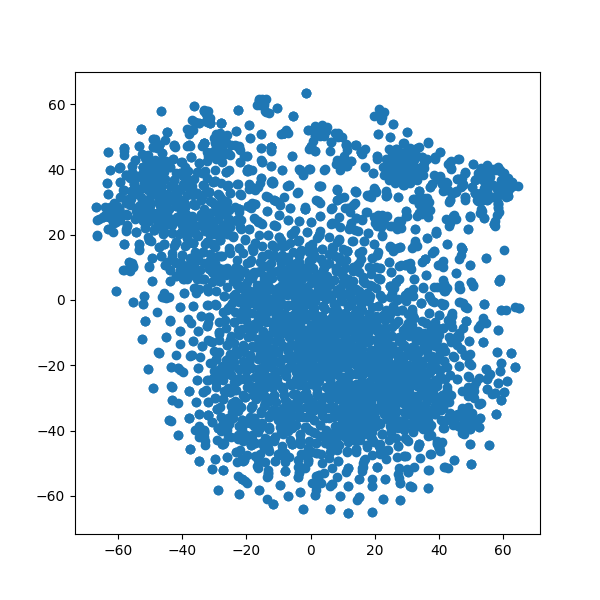}}\hspace{-3mm}
    \subfloat[\footnotesize $\gamma_h$=6.]{
           \centering 
           \label{fig:imagenet_dog_gamma_h_6}  
           \includegraphics[width = .24\linewidth]{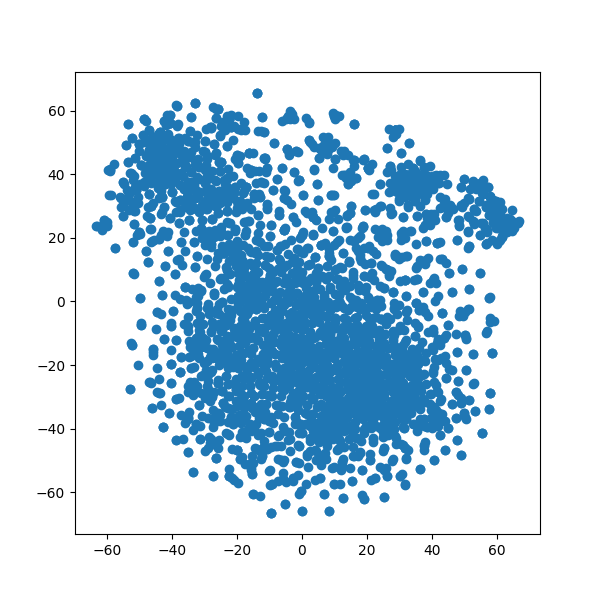}}\hspace{-3mm}
    \subfloat[\footnotesize $\gamma_h$=8.]{
            \centering 
            \label{fig:imagenet_dog_gamma_h_8}
            \includegraphics[width = .24\linewidth]{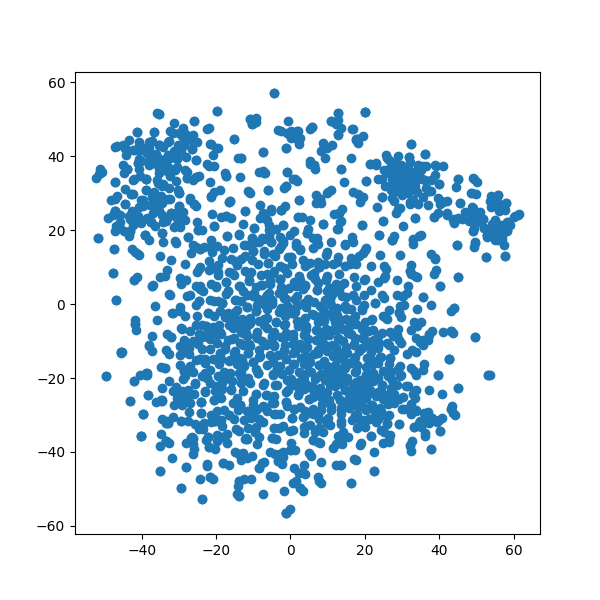}}\hspace{-3mm}
    \subfloat[\footnotesize $\gamma_h$=10.]{
           \centering 
           \label{fig:imagenet_dog_gamma_h_10}  
           \includegraphics[width = .24\linewidth]{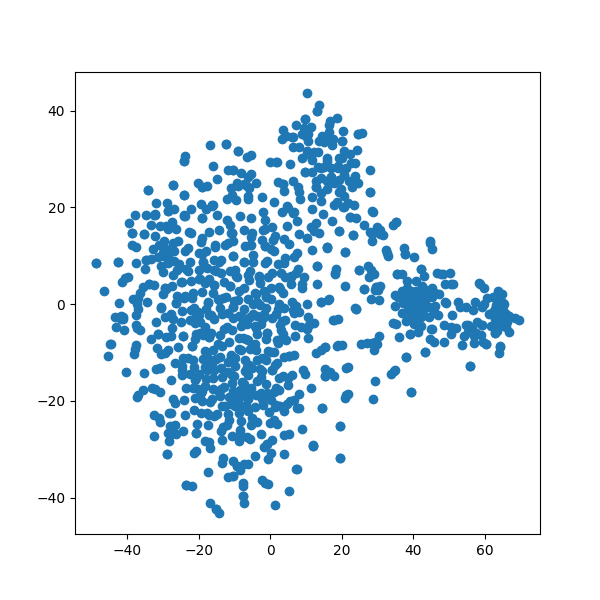}}\hspace{-3mm}
    \subfloat[\footnotesize  $\gamma_h$=12.]{
            \centering 
            \label{fig:imagenet_dog_gamma_h_12}
            \includegraphics[width = .24\linewidth]{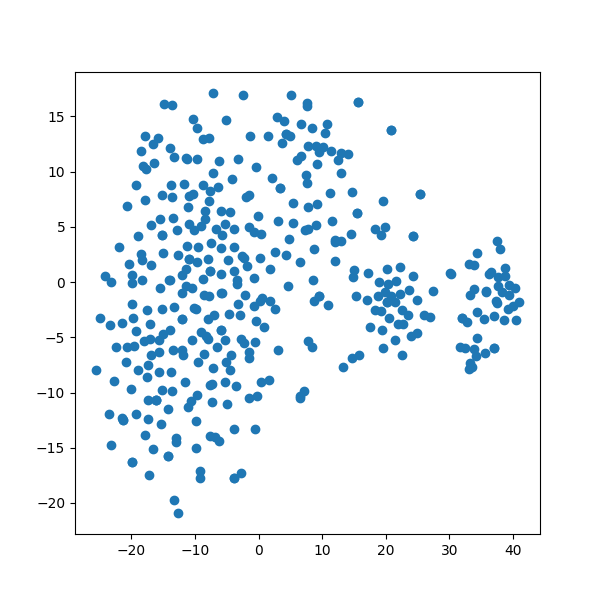}}\hspace{-3mm}            
            
\caption{$t$-SNE visualization of the words in $\mathcal{T}$ on ImageNet-Dogs, with the increase of $\gamma_h$. 
}
\label{fig:t_visualization_imagedogs}
\end{figure}

Figures \ref{fig:hierarchy_structure_stl10} and \ref{fig:hierarchy_structure_imagenet_dog} show two examples of the hierarchy-based filtering results on STL-10 and ImageNet-Dogs, which are parts of the hierarchical semantic tree built from WordNet~\cite{miller1995wordnet}. From these figures, we can observe the two hierarchical semantic trees contain more than 10 levels of words. However, we also observe that the top-level words, such as ``group'', ``unit'' in Figures \ref{fig:hierarchy_structure_stl10}, are useless for distinguishing the 10 classes in STL-10. Therefore, we removed top-6 levels of words on STL-10 and top-10 levels of words on ImageNet-Dogs in our experiments. We also show the changes of $|\mathcal{T}|$ versus $\gamma$ in Figure~\ref{fig:vocab_gamma}, indicating that the number of words in $\mathcal{T}$ decreases quickly with the increase of $\gamma_h$. Figures~\ref{fig:t_visualization_stl10} and~\ref{fig:t_visualization_imagedogs} display the $t$-SNE visualization of the words in $\mathcal{T}$ for two datasets. It is evident that, as the value of $\gamma_h$ increases, the remaining words tend to preserve the structural information. We assess the effectiveness of a hierarchy-based filtering technique.

\subsection{Hyper-parameters}
\label{app:hp}

The hyper-parameters used on five benchmark datasets are listed in Table~\ref{tab:best_hyper}.

\begin{table}[htb]
\caption{The hyper-parameters used on five benchmark datasets. $\eta_{\phi}$: learning rate, $\gamma_{r}$: the number of nearest nouns for each image center, $\gamma_{u}$: the number of most unique nouns, $\gamma_{h}$: the number of hierarchy levels in the hierarchical semantic tree (the root is level 0) to filter, $k_{I}$: the number of nearest images, $k_S$: the number of nearest texts, $k_{p}$: the number of nearest texts to an image center, $\tau_{ia}$: the temperature to control the softness of instance-level alignment, $\tau_{pa}$: the temperature to control the softness of prototype-level alignment, $\eta$, $\lambda_{a}$, $\lambda_{pa}$ and $\lambda_{ca}$: the trade-off parameters.}
\centering
\resizebox{\linewidth}{!}{
\begin{tabular}{lcccccccccccccc}
\toprule
\textbf{Dataset} & $\eta_{\phi}$ & $\gamma_{r}$ & $\gamma_{u}$ & $\gamma_{h}$ & $k_{I}$ & $k_S$ & $k_p$ & $\tau_{ia}$ & $\tau_{pa}$ & $\eta$ & $\lambda_{a}$ & $\lambda_{pa}$ & $\lambda_{ca}$  \\
\bottomrule
STL10& 1e-4& 50 & 0.05 & 6 &  20  & 5 & 20 & 1 & 1  & 10 & 10 & 0.1 & 5\\
Cifar10& 1e-3  & 500  & 0.03 & 12 & 5 & 5 & 1 & 0.2 & 0.05 & 1 & 1 & 1 & 1   \\
 Cifar100-20& 1e-2 & 100 & 0.05 &5 & 5 &20 & 20 & 1 & 1 & 1 & 1 & 1 & 1   \\
ImageNet-Dogs & 1e-5 & 1000 & 0.05 &10 & 5 & 20 & 30 & 0.05 & 0.6 & 10 & 1 & 1 & 5    \\
Tiny-ImageNet & 1e-3 & 100  &0.03 &5 & 10 & 5& 20 & 1 & 1 & 1 & 1 & 1 & 1   \\
\bottomrule
\end{tabular}
}

\label{tab:best_hyper}
\end{table}

\subsection{Sensitivity Analysis}

\noindent \textbf{Sensitivity on hyperparameters $\gamma_r$ and $\gamma_h$.} 

In our method, $\gamma_r$ and $\gamma_h$ are used to construct the semantic space, where $\gamma_r$ controls the number of central nearest neighbors selected in WordNet, and $\gamma_h$ allows further filtering of low-level semantic words. Figure~\ref{fig:sens_gamma_r} shows the sensitivity results of our method with respect to $\gamma_r$, indicating that a too-small $\gamma_r$ leads to a loss of semantic information while a too-large $\gamma_r$ leads to noise in the semantic space, resulting in a risk of cluster degradation. Therefore, we set $\gamma_r=50$ for STL10 and $\gamma_r=1000$ for ImageNet-Dogs.

As for $\gamma_h$, a larger $\gamma_h$ removes useless words in higher levels such that the remaining words in $\mathcal{T}$ can better distinguish the images, while a small $\gamma_h$ retains enough words for clustering. Figure~\ref{fig:sens_gamma_h} shows the sensitivity results of our method with respect to $\gamma_h$, which indicates that too small or too large $\gamma_h$ causes performance degeneration. Relatively speaking, the performance changes caused by the variation of $\gamma_h$ on ImageNet-Dogs are larger than those on STL10. Finally, we set $\gamma_h=6$ for STL10 and $\gamma_h=10$ for ImageNet-Dogs.

Compared to Figure~\ref{fig:vocab_gamma}, we can observe that our method produces the best results when retaining about $[400,500]$ words in $\mathcal{T}$ on both datasets.

\begin{figure}[!htb]
\vspace{-10px}
\centering
 \subfloat[\footnotesize $\gamma_r$ on STL10.]{
            \centering 
            \label{fig:stl10_gamma_r}
            \includegraphics[width = .49\linewidth]{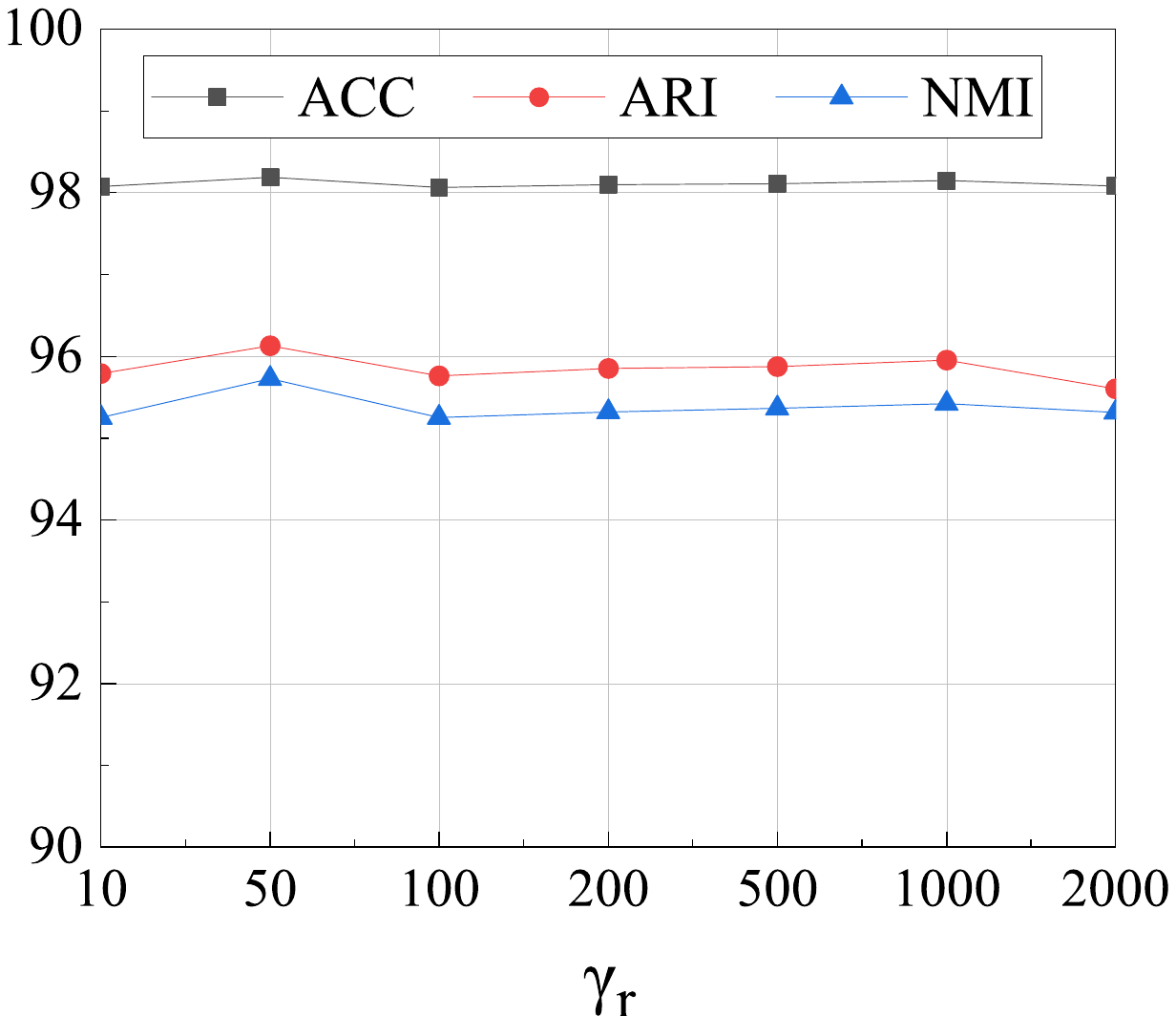}}
    \subfloat[\footnotesize $\gamma_r$ on ImageNet-Dogs.]{
            \centering 
            \label{fig:dog_gamma_r}
            \includegraphics[width = .48\linewidth]{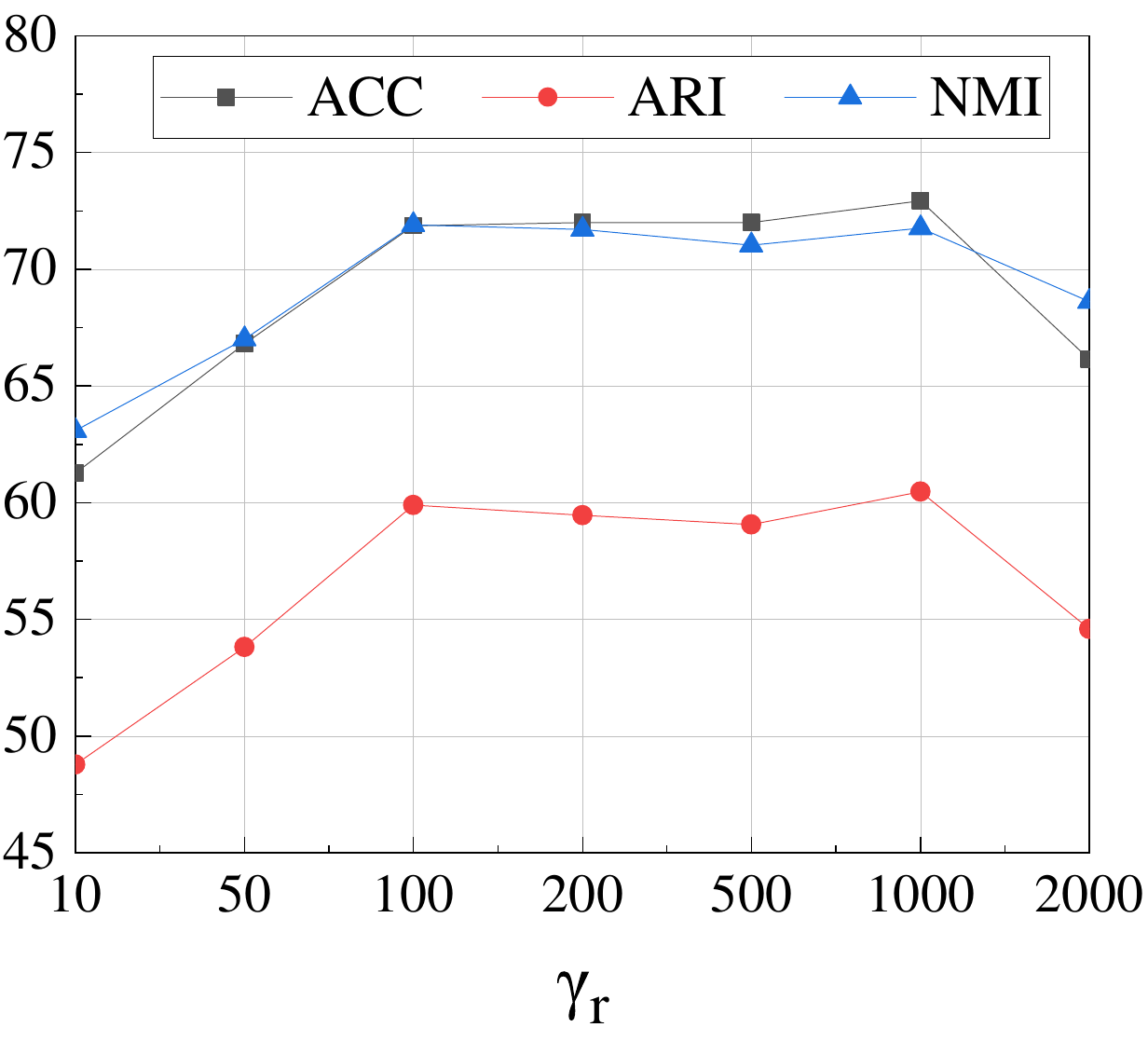}}

\caption{Sensitivity analysis of $\gamma_r$ .}
\label{fig:sens_gamma_r}
\end{figure}

\begin{figure}[!htb]
\vspace{-20px}
\centering
   
    \subfloat[\footnotesize $\gamma_h$ on STL10.]{
           \centering 
           \label{fig:stl10_gamma_h}  
           \includegraphics[width = .49\linewidth]{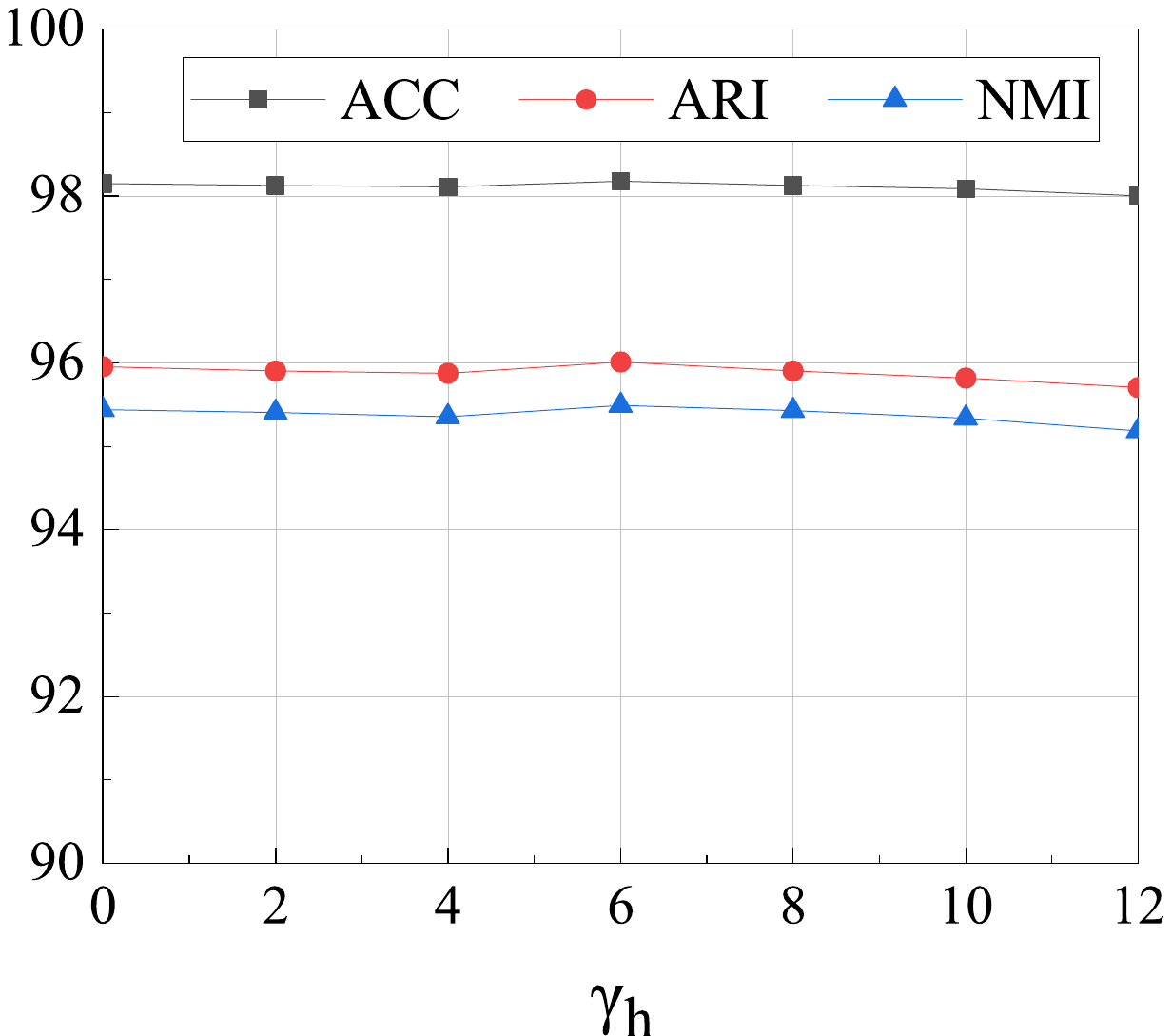}} 
    \subfloat[\footnotesize $\gamma_h$ on Imaganet-Dog.]{
           \centering 
           \label{fig:dog_gamma_h}  
           \includegraphics[width = .48\linewidth]{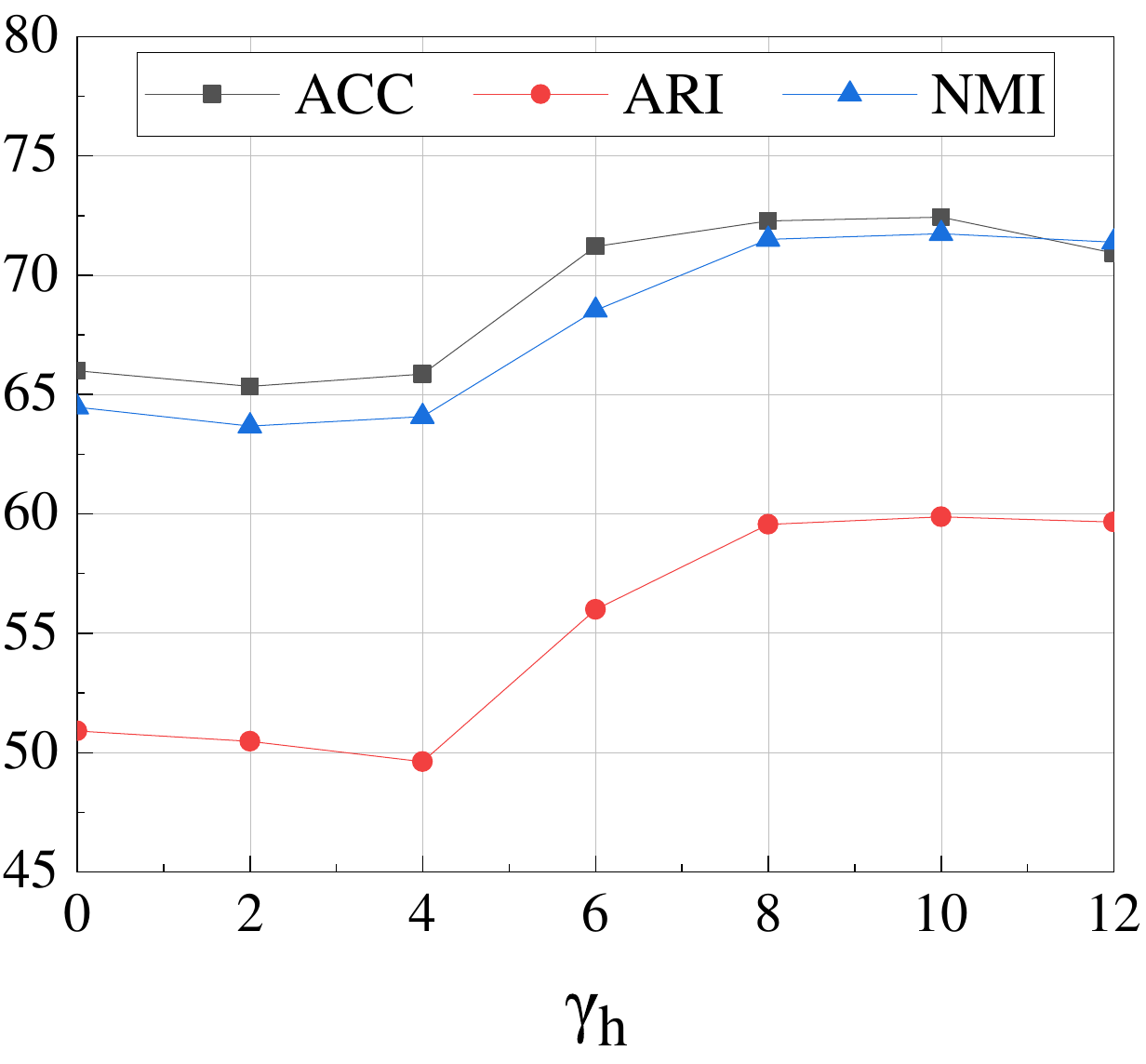}}

\caption{Sensitivity analysis of $\gamma_h$ .}
\label{fig:sens_gamma_h}
\end{figure}

\noindent \textbf{Sensitivity on hyperparameters $\tau_{ia}$ and $\tau_{pa}$.} 
$\tau_{ia}$ and $\tau_{pa}$ are the alignment temperature parameters to control the softness of the instance-level alignment and prototype-level alignment, separately. As shown in Figures~\ref{fig:sens_tau_si} and~\ref{fig:sens_tau_pa}, we can observe that the performances of our method remain relatively stable with changes to $\tau_{ia}$, but exhibit significant changes with variations in $\tau_{pa}$.

\begin{figure}[!htb]
\vspace{-20px}
\centering
    \subfloat[\footnotesize $\tau_{ia}$ on Cifar10.]{
           \centering 
           \label{fig:cifar10_alignment_si}  
           \includegraphics[width = .49\linewidth]{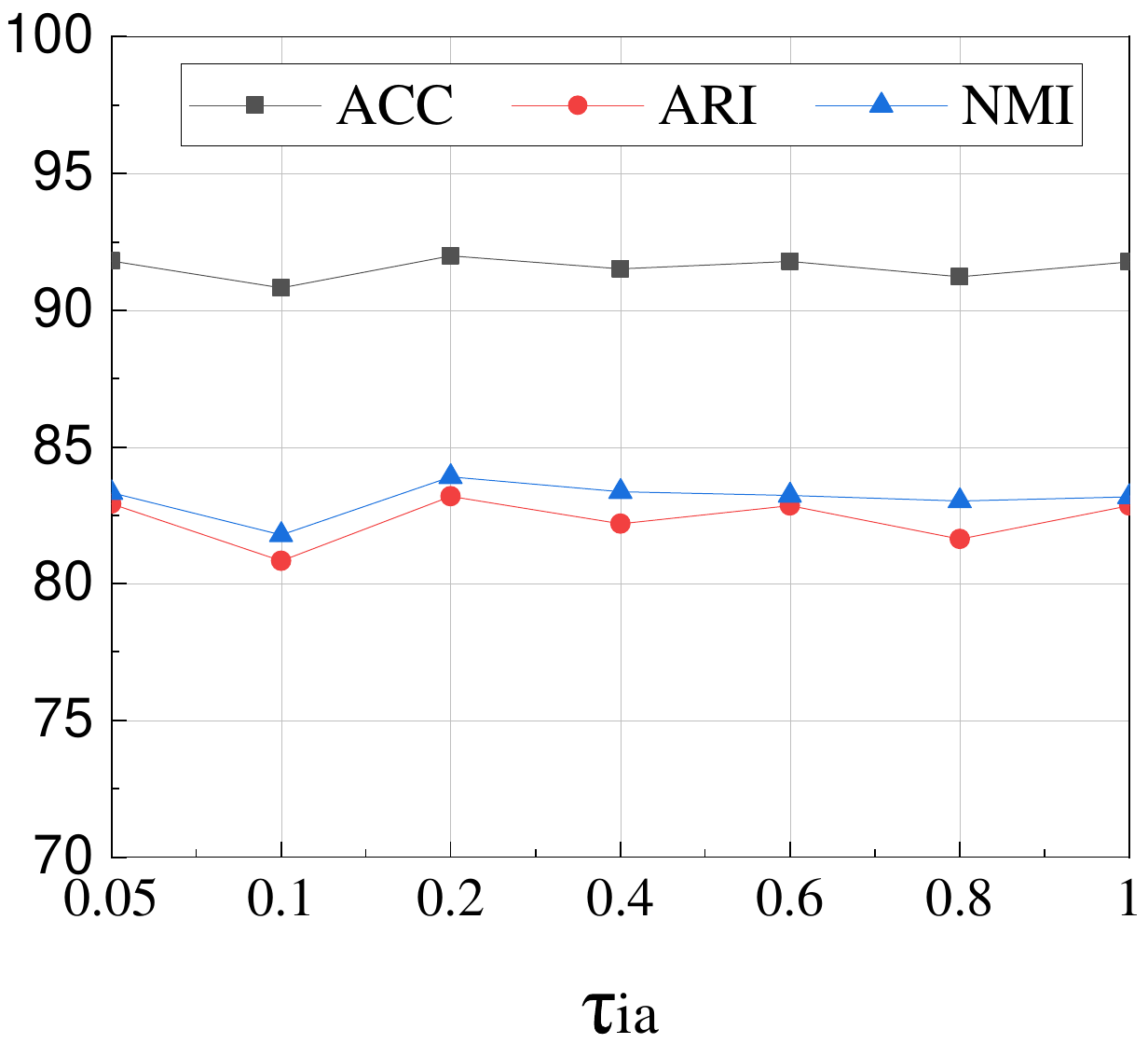}}
    \subfloat[\footnotesize $\tau_{ia}$ on ImageNet-Dogs.]{
            \centering 
            \label{fig:imagenet_dog_alignment_si}
            \includegraphics[width = .48\linewidth]{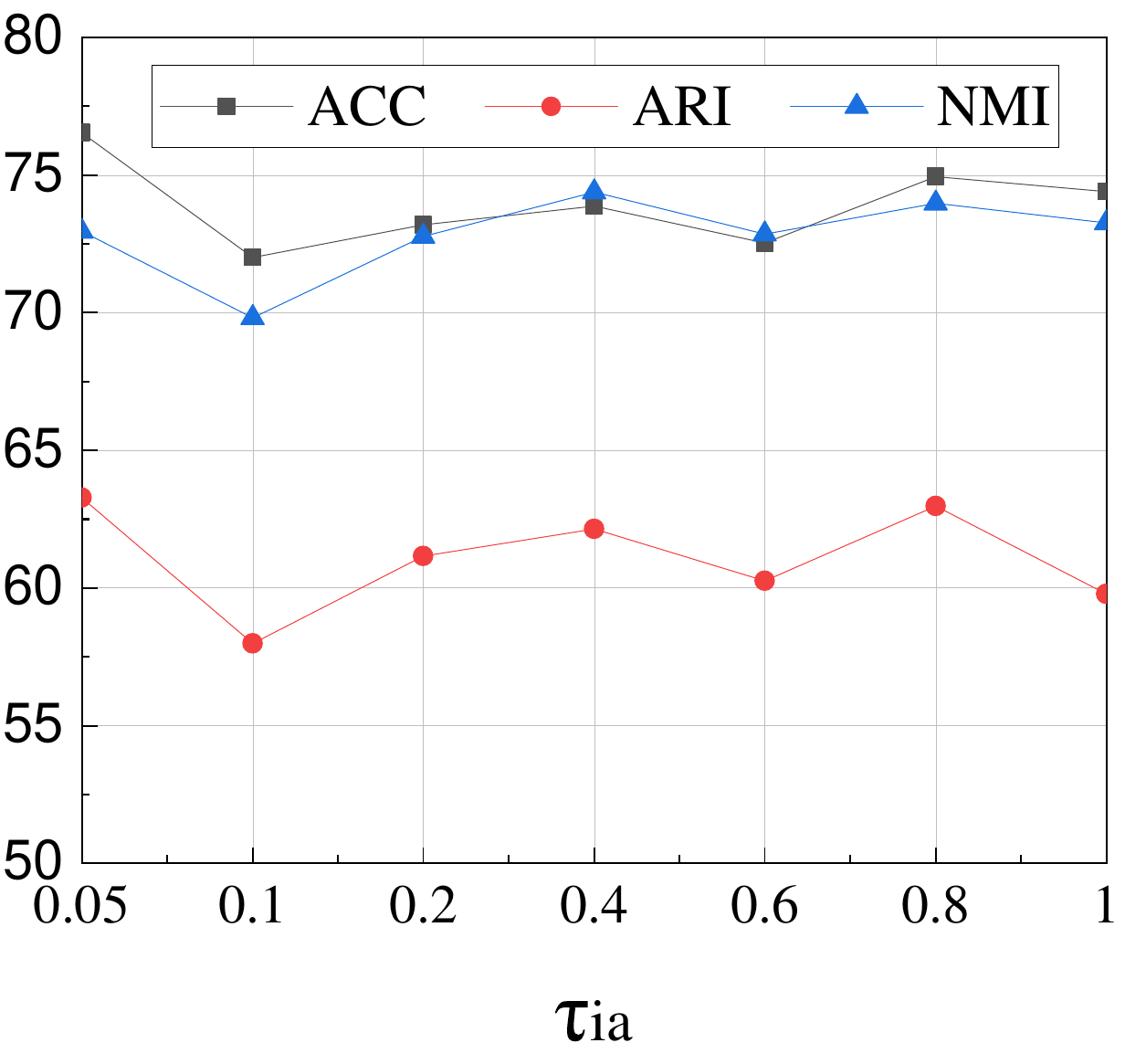}}
            
\caption{Sensitivity analysis of $\tau_{ia}$.}
\label{fig:sens_tau_si}
\end{figure}

\begin{figure}[!htb]
\vspace{-20px}
\centering
    \subfloat[\footnotesize $\tau_{pa}$ on Cifar10.]{
           \centering 
           \label{fig:cifar10_alignment_pa}  
           \includegraphics[width = .49\linewidth]{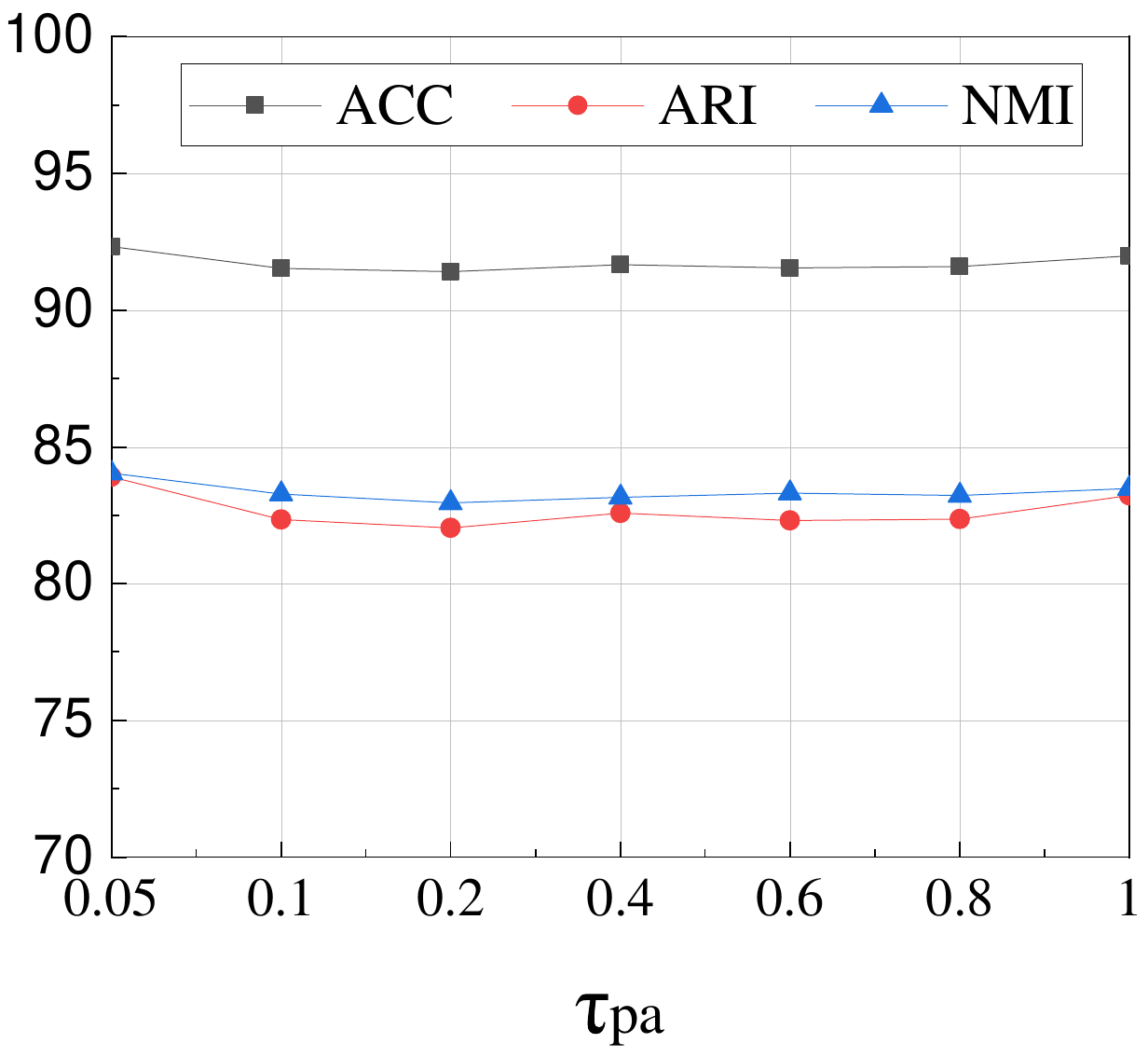}}
    \subfloat[\footnotesize $\tau_{pa}$ on ImageNet-Dogs.]{
            \centering 
            \label{fig:imagenet_dog_alignment_pa}
            \includegraphics[width = .48\linewidth]{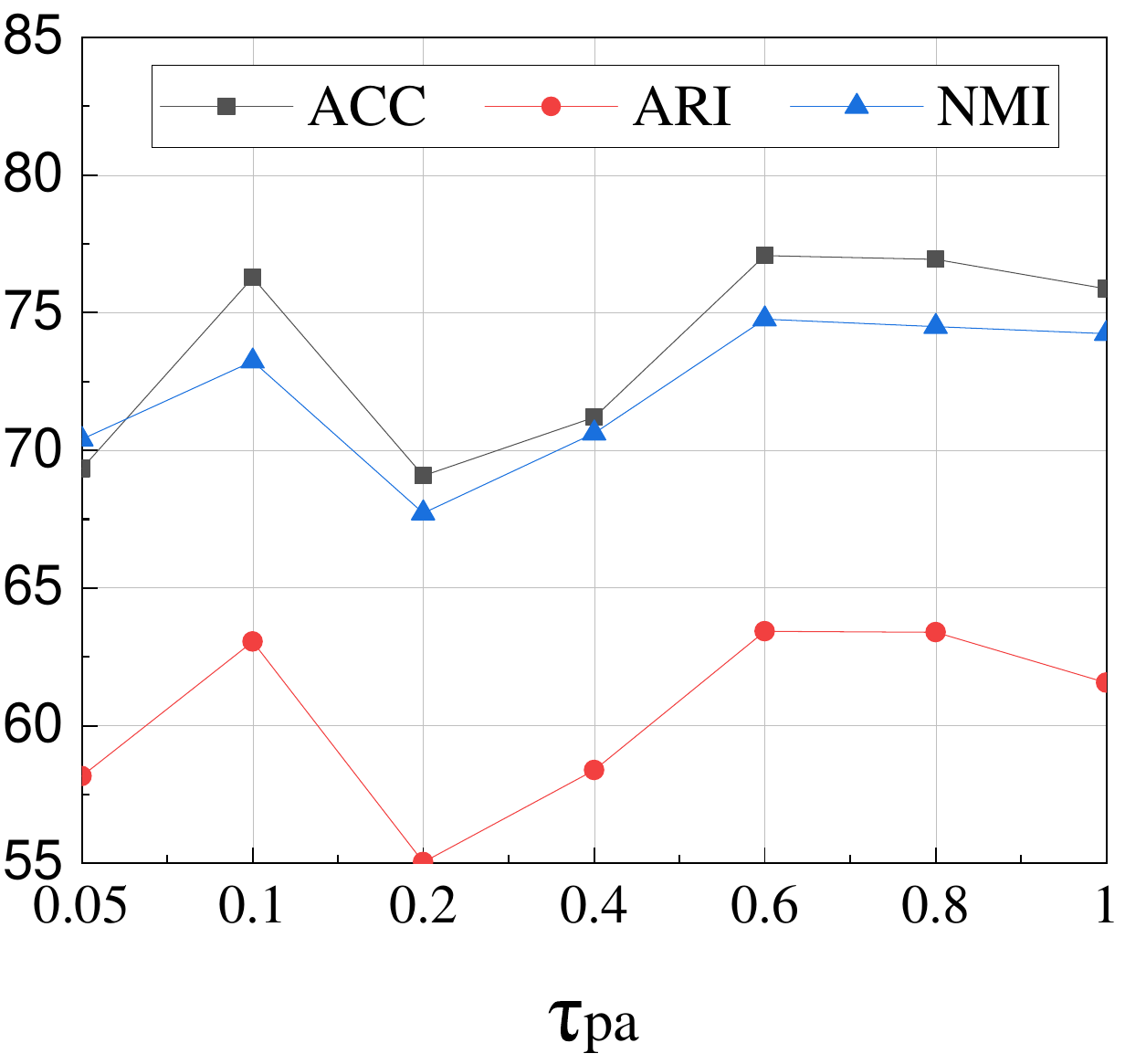}}
            
\caption{Sensitivity analysis of $\tau_{pa}$.}
\label{fig:sens_tau_pa}
\end{figure}

\subsection{Compare with Zero-shot Learning}
We compared MCA with two zero-shot learning methods, i.e., CLIP and MUST. Compared to CLIP, MUST improve CLIP by 3.3\% on Caltech101 and $18\%$ on UCF101. When compared to MUST, our method deteriorates MUST by -8.2\% on Caltech101 and -8.5\% on UCF101. Although the results show that our method performs worse than CLIP and MUST, our method has the advantage of not requiring class names as input, which is necessary for zero-shot learning. This makes our approach more versatile and applicable to a broader range of real-world scenarios.

\begin{figure}[!htb]
    \vspace{-30px}
     \centering
     \includegraphics[width=0.6\linewidth]{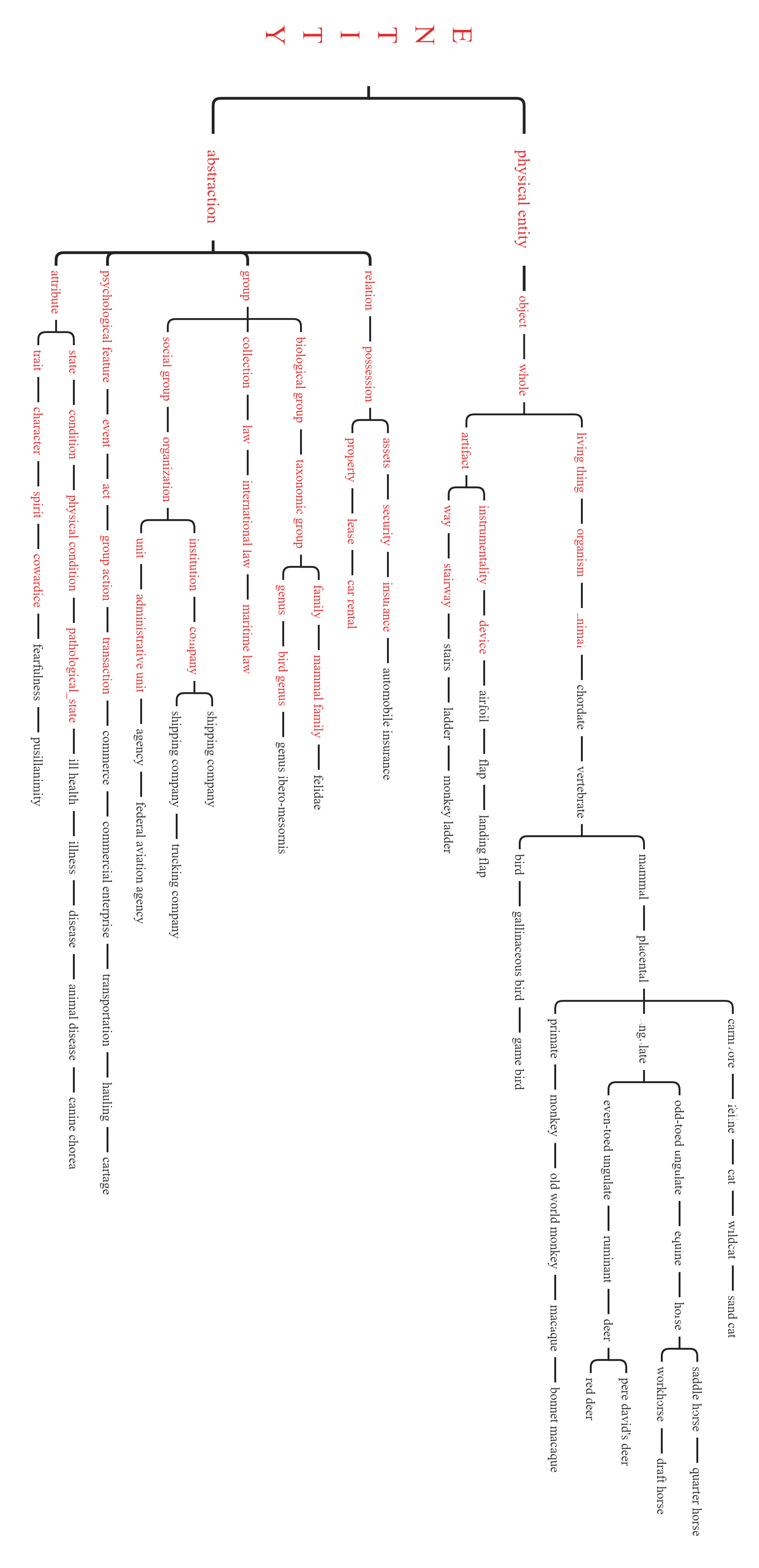}
     \caption{Example of hierarchy-based filtering result on STL10. The hierarchical semantic tree was built from WordNet~\cite{miller1995wordnet} and the top-6 levels of words in the tree (highlighted in red) were filtered.}    
     \label{fig:hierarchy_structure_stl10}
 \end{figure}

\begin{figure}[!htb]
    \vspace{-30px}
     \centering
     \includegraphics[width=0.4\textwidth]{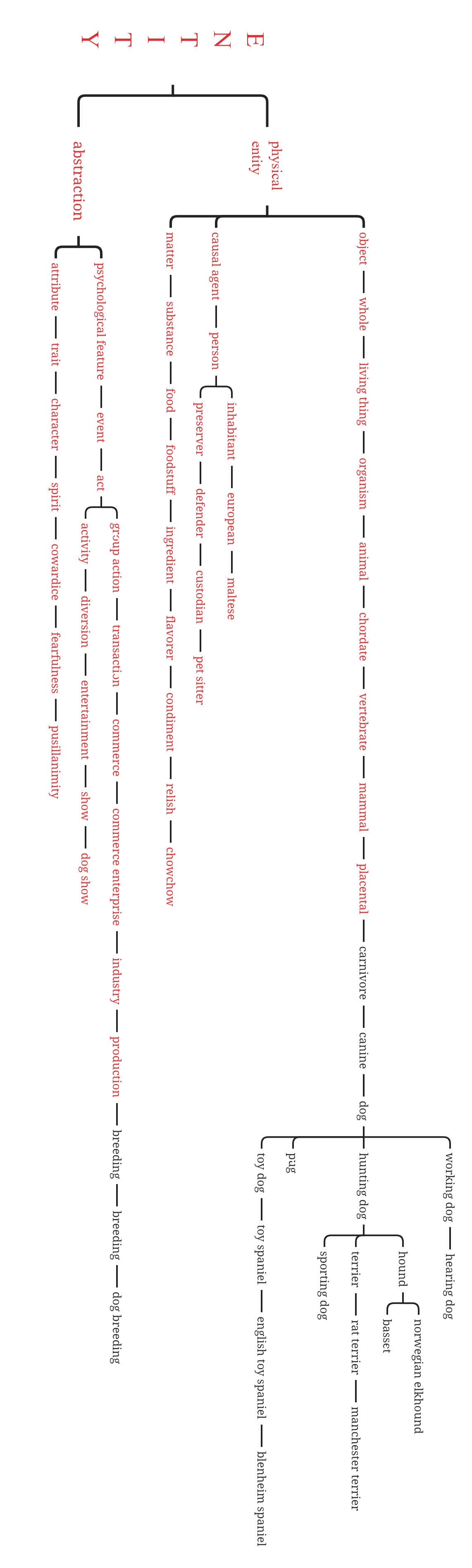}
     \caption{Example of hierarchy-based filtering result on ImageNet-Dogs. The hierarchical semantic tree was built from WordNet~\cite{miller1995wordnet} and the top-10 levels of words in the tree (highlighted in red) were filtered.}    
     \label{fig:hierarchy_structure_imagenet_dog}
\end{figure}

\end{document}